\pdfoutput=1
\documentclass[11pt]{article}
\usepackage[textwidth=400pt,centering]{geometry}
\usepackage{fullpage}
\usepackage{microtype}
\usepackage{graphicx}
\usepackage{subfigure}
\usepackage{booktabs} 
\usepackage[sort]{natbib}
\usepackage[nottoc,notlot,notlof]{tocbibind}
\usepackage{algorithm}
\usepackage{algorithmic}
\usepackage[hidelinks]{hyperref}
\hypersetup{
    colorlinks,
    linkcolor={red!50!black},
    citecolor={blue!50!black},
    urlcolor={blue!80!black}
}

\usepackage{amsmath,amssymb,amsthm}
\usepackage{thmtools,thm-restate}
\usepackage[capitalize,noabbrev]{cleveref}

\theoremstyle{plain}
\newtheorem{theorem}{Theorem}

\newtheorem{lemma}[theorem]{Lemma}

\theoremstyle{definition}

\newtheorem{fact}[theorem]{Fact}
\theoremstyle{remark}

\usepackage[utf8]{inputenc} 
\usepackage[T1]{fontenc}    
\usepackage{url}            
\usepackage{amsfonts}       
\usepackage{nicefrac}       
\usepackage{xcolor}    

\usepackage{pifont}
\usepackage{bbm}
\usepackage{bm}
\usepackage{multirow}
\usepackage{braket}
\usepackage{physics}
\usepackage{enumitem}
\definecolor{darkblue}{rgb}{0,0.08,0.8}
\DeclareMathOperator*{\argmax}{arg\,max}
\DeclareMathOperator*{\argmin}{arg\,min}
\newcommand{\Lmid}{\,\middle\vert\, }
\newcommand{\ceil}[1]{\left\lceil {#1} \right\rceil}
\newcommand\numberthis{\addtocounter{equation}{1}\tag{\theequation}}

\newcommand{\I}{\mathbbm{1}}
\newcommand{\dx}{\mathrm{d}}

\newcommand{\bth}{\boldsymbol{\theta}}

\newcommand{\ka}{{\kappa, \alpha}}
\newcommand{\bn}{\bar{n}}
\newcommand{\ta}{\tilde{\alpha}}
\newcommand{\tk}{\tilde{\kappa}}
\newcommand{\ha}{\hat{\alpha}}
\newcommand{\hk}{\hat{\kappa}}
\newcommand{\ba}{\bar{\alpha}}

\newcommand{\hmu}{\hat{\mu}}

\newcommand{\tmu}{\tilde{\mu}}
\newcommand{\Ft}{\mathcal{F}_t}

\newcommand{\KLinf}{\mathrm{KL}_{\mathrm{inf}}}
\newcommand{\KL}{\mathrm{KL}}

\newcommand{\TS}{\mathrm{STS}}
\newcommand{\TST}{\mathrm{STS}\text{-}\mathrm{T}}
\DeclareMathSymbol{\mh}{\mathord}{operators}{`\-}
\newcommand{\Par}{\mathrm{Pa}}
\newcommand{\InvG}{\mathrm{IG}}
\newcommand{\Er}{\mathrm{Erlang}}

\newcommand{\Exp}{\mathrm{Exp}}
\newcommand{\reg}{\mathrm{Reg}}
\newcommand{\eps}{\epsilon}
\newcommand{\veps}{\varepsilon}

\newcommand{\eA}{\mathcal{A}}
\newcommand{\eK}{\mathcal{K}}
\newcommand{\eE}{\mathcal{E}}
\newcommand{\eM}{\mathcal{M}}

\usepackage{times}  
\usepackage{helvet} 
\usepackage{courier}
\usepackage{graphicx}
\usepackage{datatool}

\newcommand{\br}{\bm{r}}

\newcommand{\bmu}{\boldsymbol{\mu}}

\title{Optimality of Thompson Sampling with Noninformative Priors \\ for Pareto Bandits}
\author{Jongyeong Lee$^{1,2}$ \and Junya Honda$^{3,2}$  \and Chao-Kai Chiang$^{1}$ \and Masashi Sugiyama$^{2,1}$}
\date{
$^1$ The University of Tokyo 
$^2$ RIKEN AIP
$^3$ Kyoto University
}

\begin{document}

\maketitle

\begin{abstract}
In the stochastic multi-armed bandit problem, a randomized probability matching policy called Thompson sampling (TS) has shown excellent performance in various reward models.
In addition to the empirical performance, TS has been shown to achieve asymptotic problem-dependent lower bounds in several models.
However, its optimality has been mainly addressed under light-tailed or one-parameter models that belong to exponential families.
In this paper, we consider the optimality of TS for the Pareto model that has a heavy tail and is parameterized by two unknown parameters.
Specifically, we discuss the optimality of TS with probability matching priors that include the Jeffreys prior and the reference priors.
We first prove that TS with certain probability matching priors can achieve the optimal regret bound.
Then, we show the suboptimality of TS with other priors, including the Jeffreys and the reference priors.
Nevertheless, we find that TS with the Jeffreys and reference priors can achieve the asymptotic lower bound if one uses a truncation procedure.
These results suggest carefully choosing noninformative priors to avoid suboptimality and show the effectiveness of truncation procedures in TS-based policies.
\end{abstract}

\section{Introduction}
In the multi-armed bandit (MAB) problem, an agent plays an arm and observes a reward only from the played arm, which is partial feedback~\citep{thompson1933likelihood, robbins1952some}.
The rewards are further assumed to be generated from the distribution of the corresponding arm in the stochastic MAB problem~\citep{bubeck2012regret}.
Since only partial observations are available, the agent has to estimate unknown distributions to guess which arm is optimal while avoiding playing suboptimal arms that induce loss of resources.
Thus, the agent has to cope with the dilemma of exploration and exploitation.

In this problem, Thompson sampling (TS), a randomized Bayesian policy that plays an arm according to the posterior probability of being optimal, has been widely adopted because of its outstanding empirical performance~\citep{chapelle2011empirical, russo2018tutorial}.
Following its empirical success, theoretical analysis of TS has been conducted for several reward models such as Bernoulli models~\citep{agrawal2012analysis, kaufmann2012thompson}, one-dimensional exponential families~\citep{KordaTS}, Gaussian models~\citep{honda2014optimality}, and bounded support models~\citep{riou2020bandit,pmlr-v139-baudry21a} where asymptotic optimality of TS was established.
Here, an algorithm is said to be asymptotically optimal if it can achieve the theoretical problem-dependent lower bound derived by \citet{lai1985asymptotically} for one-parameter models and \citet{burnetas1996optimal} for multiparameter or nonparametric models.
Note that the performance of any reasonable algorithms cannot be better than these lower bounds.

Apart from the problem-dependent regret analysis, several works studied the problem-independent or prior-independent bounds of TS~\citep{bubeck2013prior,russo2016information, agrawal2017near}.
In this paper, we study how the choice of noninformative priors affects the performance of TS for any given problem instance.
In other words, we focus on the asymptotic optimality of TS depending on the choice of noninformative priors.

The asymptotic optimality of TS has been mainly considered in the one-parameter model, while its optimality under the multiparameter model has not been well studied.
To the best of our knowledge, the asymptotic optimality of TS in the noncompact multiparameter model is only known for the Gaussian bandits~\citep{honda2014optimality} where both the mean and variance are unknown.
They showed that TS with the uniform prior is optimal while TS with the Jeffreys prior and reference prior cannot achieve the lower bound.
The success of the uniform prior is due to its frequent playing of seemingly suboptimal arms.
Its conservativeness comes from a moderate overestimation of the posterior probability that current suboptimal arms might be optimal.

In this paper, we consider the two-parameter Pareto models where the tail function is heavy-tailed.
We first derive the closed form of the problem-dependent constant that appears in the theoretical lower bound in Pareto models, which is not trivial, unlike those for exponential families.
Based on this result, we show that TS with some probability-matching priors achieves the optimal bound, which is the first result for two-parameter Pareto bandit models, to our knowledge.

We further show that TS with different choices of probability matching priors, called optimistic priors, suffers a polynomial regret in expectation.
Therefore, being conservative would be better when one chooses noninformative priors to avoid suboptimality in view of expectation.
Nevertheless, we show that TS with the Jeffreys prior or the reference prior can achieve the optimal regret bound if we add a truncation procedure on the shape parameter. 
Our contributions are summarized as follows:
\begin{itemize}
    \item We prove the asymptotic optimality/suboptimality of TS under different choices of priors, which shows the importance of the choice of noninformative priors in cases of two-parameter Pareto models.
    \item We provide another option to achieve optimality: adding a truncation procedure to the parameter space of the posterior distribution instead of finding an optimal prior.
\end{itemize}

This paper is organized as follows.
In Section~\ref{sec: preliminaries}, we formulate the stochastic MAB problems under the Pareto distribution and derive its regret lower bound.
Based on the choice of noninformative priors and their corresponding posteriors, we formulate TS for the Pareto models and propose another TS-based algorithm to solve the suboptimality problem of the Jeffreys prior and the reference prior in Section~\ref{sec: TS}.
In Section~\ref{sec: main}, we provide the main results on the optimality of TS and TS with a truncation procedure, whose proof outline is given in Section~\ref{sec: optimality}.
Numerical results that support our theoretical analysis are provided in Section~\ref{sec: experiments}.

\section{Preliminaries}\label{sec: preliminaries}
In this section, we formulate the stochastic MAB problem. 
We derive the exact form of the problem-dependent constant that appears in the lower bound of the expected regret in Pareto bandits. 

\subsection{Notations}
We consider the stochastic $K$-armed bandit problem where the rewards are generated from Pareto distributions with fixed parameters.
An agent chooses an arm $a$ in $[K]:=\{1, \ldots, K \}$ at each round $t \in \mathbb{N}$ and observes an independent and identically distributed reward from $\Par(\kappa_a, \alpha_a)$, where $\Par(\ka)$ denotes the Pareto distribution parameterized by scale $\kappa >0$ and shape $\alpha>0$.
This has the density function of form
\begin{equation}\label{eq: pdfPar}
  f_{\ka}^{\mathrm{Pa}}(x) = \frac{\alpha \kappa^\alpha}{x^{\alpha+1}}\I[x \geq \kappa],
\end{equation}
where $\I[\cdot]$ denotes the indicator function.
We consider a bandit model where parameters $\theta_a = (\kappa_a, \alpha_a) \in \mathbb{R}_{+} \times (1, \infty)$ are unknown to the agent.
We denote the mean of a random variable following $\Par(\theta_a)$ by $\mu_a = \mu(\theta_a) := \frac{\kappa_a \alpha_a}{\alpha_a -1}$.
Note that $\alpha >1$ is a necessary condition of an arm to have a finite mean, which is required to define the sub-optimality gap $\Delta_a := \max_{i\in [K]}\mu_i - \mu_a$.
We assume without loss of generality that arm $1$ has the maximum mean for simplicity, i.e., $\mu_1 = \max_{i \in [K]} \mu_i$.
Let $j(t)$ be the arm played at round $t\in\mathbb{N}$ and $N_a(t)=\sum_{s=1}^{t-1}\I[j(s)=a]$ denote the number of rounds the arm $a$ is played until round~$t$.
Then, the regret at round $T$ is given as
\begin{equation*}
    \reg(T) = \sum_{t=1}^T \Delta_{j(t)} = \sum_{a=2}^K \Delta_a N_a(T+1).
\end{equation*}
Let $r_{a,n}$ be the $n$-th reward generated from the arm $a$.
In the Pareto distribution, the maximum likelihood estimators (MLEs) of $\ka$ for arm $a$ given $n$ rewards and their distributions are given as follows~\citep{malik1970estimation}:
\begin{align*}
    \hk_a(n) = \min_{s \in [n]} r_{a,s} \sim \Par(\kappa_a, n\alpha_a), ~~~~
    \ha_a(n) = \frac{n}{ \sum_{s=1}^{n} \log(r_{a,s}) - n\log \hk_a(n)} \sim \InvG(n-1, n\alpha_a), \numberthis{\label{eq: MLEdist}}
\end{align*}
where $\InvG(n, \alpha)$ denotes the inverse-gamma distribution with shape $n>0$ and scale $\alpha>0$.
Note that \citet{malik1970estimation} further showed the stochastic independence of $\ha(n)$ and $\hk(n)$.

\subsection{Asymptotic lower bound}
\citet{burnetas1996optimal} provided a problem-dependent lower bound of the expected regret such that any uniformly fast convergent policy, which is a policy satisfying $\reg(T) = o(t^\alpha)$ for all $\alpha \in (0,1)$, must satisfy
\begin{equation}\label{eq: lowerbound}
     \liminf_{T \to \infty} \frac{\mathbb{E}[\reg(T)]}{\log T}
    \geq \sum_{a=2}^K \frac{\Delta_a}{\inf_{\theta: \mu(\theta) > \mu_1} \KL(\Par(\kappa_a, \alpha_a), \Par(\theta))},
\end{equation}
where $\KL(\cdot, \cdot)$ denotes the Kullback-Leibler (KL) divergence.
Notice that the bandit model $(\theta_a)_{a\in[K]}$ is considered as a fixed constant in the problem-dependent analysis.

The KL divergence between Pareto distributions is given as
\begin{align*}
    \KL(\Par(\kappa_1, \alpha_1),\Par(\kappa_2, \alpha_2)) = \begin{cases}
    \log\left(\frac{\alpha_1}{\alpha_2}\right) + \alpha_2 \log\left(\frac{\kappa_1}{\kappa_2}\right) + \frac{\alpha_2}{\alpha_1} - 1 &\text{if } \kappa_2 \leq \kappa_1, \\
    \infty &\text{otherwise}.
    \end{cases}
\end{align*}
Here the divergence sometimes becomes infinity since the scale parameter $\kappa$ determines the support of the Pareto distribution.
We denote the numerator in (\ref{eq: lowerbound}) for $a \ne 1$ by
\begin{align*}
    \KLinf(a) &:= \inf_{\theta: \mu(\theta) > \mu_1} \KL(\Par(\kappa_a, \alpha_a), \Par(\theta))
    \\  &= \inf_{\theta\in \Theta_a} \log \frac{\alpha_a}{\alpha} +  \alpha \log \frac{\kappa_a}{\kappa} + \frac{\alpha}{\alpha_a} -1,
\end{align*}
where
\begin{equation}\label{eq: thetaset}
    \Theta_a  = \left\{ (\kappa, \alpha) \in (0, \kappa_a] \times (0,\infty) : \mu(\ka) > \mu_1 \right\}.
\end{equation}
Notice that $\Theta_a$ allows parameters whose expected rewards are infinite ($\alpha \in (0,1]$), although we consider a bandit model with $\alpha_a >1$ for all $a\in[K]$ so that the sub-optimality gap $\Delta_a$ becomes finite.
This implies that $\KLinf(a)$ does not depend on whether the agent considers the possibility that an arm has the infinite expected reward or not.
Then, we can simply rewrite the lower bound in (\ref{eq: lowerbound}) as
\begin{equation*}
     \liminf_{T \to \infty} \frac{\mathbb{E}[\reg(T)]}{\log T} \geq \sum_{a=2}^K \frac{\Delta_a}{\KLinf(a)}.
\end{equation*}
The following lemma shows the closed form of this infimum, whose proof is given in Appendix~\ref{sec: L1proof}.
\begin{restatable}{lemma}{lemKLinf}\label{lem: KLinf}
For any arm $a \ne 1$, it holds that
\begin{equation*}
        \KLinf(a) = \log\left(\alpha_a\frac{\mu_1-\kappa_a}{\mu_1}\right) + \frac{1}{\alpha_a}\frac{\mu_1}{\mu_1-\kappa_a} -1.
\end{equation*}
\end{restatable}

\subsection{Relation with bounded moment models}
In MAB literature, several algorithms based on the upper confidence bound (UCB) were proposed to tackle heavy-tailed models with infinite variance under additional assumptions on moments~\citep{bubeck2013bandits}.
One major assumption is that the moment of any arm $a$ satisfies $\mathbb{E}[|r_{a,n}|^{\gamma}] \leq v$ for some fixed $\gamma \in [1,2)$ and known $v<\infty$~\citep{bubeck2013bandits}.
Note that the $\gamma$-th raw moment of the density function of $X$ following $\Par(\ka)$ is given as 
\begin{equation*}
    \mathbb{E}\left[X^\gamma \right] = 
    \begin{cases}
        \infty & \alpha \leq \gamma, \\
        \frac{\alpha \kappa^\gamma}{\alpha - \gamma} & \alpha > \gamma,
    \end{cases}
\end{equation*}
which implies that the Pareto models and the bounded moment models are not a subset of each other.

Recently, \citet{agrawal2021regret} proposed an asymptotically optimal KL-UCB based algorithm that requires solving the optimization problem at every round.
Since the bounded moment model only covers certain Pareto distributions in general, the known optimality result of KL-UCB does not necessarily imply the optimality in the sense of (\ref{eq: lowerbound}).

\section{Thompson sampling and probability matching priors}\label{sec: TS}
TS is a policy from the Bayesian viewpoint, where the choices of priors are important.
Although one can utilize prior knowledge on parameters when choosing the prior, such information would not always be available in practice.
To deal with such scenarios, we consider noninformative priors based on the Fisher information (FI) matrix, which does not assume any information on unknown parameters.

For a random variable $X$ with density $f(\cdot| \bth)$, FI is defined as the variance of the score, a partial derivative of $\log f$ with respect to $\bth$, which is given as follows~\citep{cover2006elements}:
\begin{align*}
    I_{ij} = [I(\bth)]_{ij} := \mathbb{E}_{X}\left[ \left( \frac{\partial}{\partial \theta_i}\log f(X|\bth) \right) \left( \frac{\partial}{\partial \theta_j}\log f(X|\bth) \right)\Lmid \bth \right].\numberthis{\label{eq: infmetric}}
\end{align*}
It is known that the FI matrix in (\ref{eq: infmetric}) coincides with the negative expected value of the Hessian matrix of $\log f(X| \bth)$ if the model satisfies the FI regular condition~\citep{schervish2012theory}.
However, $\Par(\ka)$ does not satisfy this condition since it is a parametric-support family.
Therefore, for $X$ with density function in (\ref{eq: pdfPar}), one can obtain the FI matrix of $\Par(\ka)$ based on (\ref{eq: infmetric}) as follows~\citep{Pareto_prior}:
\begin{equation}\label{eq: fisher}
    I(\ka) = 
    \begin{bmatrix}
        \frac{\alpha^2}{\kappa^2} & 0 \\
        0 & \frac{1}{\alpha^2}
    \end{bmatrix}
    =
    \begin{bmatrix}
        I_{11}(\kappa)I_{11}(\alpha) & 0 \\
        0 & I_{22}(\alpha)
    \end{bmatrix},
\end{equation}
where $I_{11}(\kappa) = \frac{1}{\kappa^2}$, $I_{11}(\alpha) = \alpha^2$, and $I_{22}(\alpha) = \frac{1}{\alpha^2}$.
Note that $I_{11}$ differs from $-\mathbb{E}\left[ \frac{\partial^2}{\partial \kappa^2} \log f_{\ka}^{\Par}(X;\theta)\Lmid \theta \right]=\frac{\alpha}{\kappa^2}$.

Based on (\ref{eq: fisher}), the Jeffreys prior and the reference prior are given as $\pi_{\mathrm{J}}(\ka) \propto \sqrt{\det(I)}=  \frac{1}{\kappa}$ and $\pi_{\mathrm{R}}(\ka) \propto \sqrt{I_{11}(\kappa) I_{22}(\alpha)} = \frac{1}{\kappa \alpha}$, respectively.
Here, the reverse reference prior is the same as the reference prior from the orthogonality of parameters~\citep{datta1995some, datta1996priors}.

From the orthogonality of parameters, the probability matching prior when $\kappa$ is of interest and $\alpha$ is the nuisance parameter is given as
\begin{equation*}
    \pi_{\mathrm{P}}(\ka) \propto \sqrt{I_{11}}g_1(\alpha) = \frac{\alpha}{\kappa} g_1(\alpha)
\end{equation*}
for arbitrary $g_1(\alpha) > 0$~\citep{tibshirani1989noninformative}.
In this paper, we consider the prior $\pi(\ka) \propto \frac{\alpha^{-k}}{\kappa}$ for $k \in \mathbb{Z}$ since the cases $k=0, 1$ correspond to the Jeffreys prior and the (reverse) reference prior, respectively.
\paragraph{Remark 1.}
The Pareto distribution discussed in this paper is sometimes called the Pareto type $1$ distribution~\citep{arnold2008pareto}.
On the other hand, \citet{kim2009noninformative} derived several noninformative priors for a special case of the Pareto type $2$ distribution called the Lomax distribution~\citep{lomax1954business}, where the FI matrix can be written using the negative Hessian.

In the multiparameter cases, the Jeffreys prior is known to suffer from many problems~\citep{datta1996invariance,ghosh2011objective}.
For example, it is known that the Jeffreys prior leads to inconsistent estimators for the error variance in the Neyman-Scott problem~\citep[see][Example 3]{berger1992development}.
This might be a possible reason why TS with Jeffreys prior suffers a polynomial expected regret in a multiparameter distribution setting.
More details on the probability matching prior and the Jeffreys prior can be found in Appendix~\ref{sec: bayesian}.

\begin{algorithm}[!t]
    \caption{\textcolor{red}{$\TS$} / \textcolor{blue}{$\TST$}}
    \label{alg: STS}
    \begin{algorithmic}[1] 
        \STATE \textbf{Parameter:} $k \in \mathbb{Z}$, $\bn = \max\{ 2, k+1 \}$.
        \STATE \textbf{Initialization:} Select each arm $\bn$ times.
        \STATE \textbf{Loop:}
        \STATE \textcolor{red}{Sample $\ta_{a}(t) \sim \Er \left(N_a(t) -k , \frac{N_a(t)}{\ha_a(N_a(t))}\right)$.} 
        \STATE \textcolor{blue}{$\ba_a(N_a(t)) \gets \min(N_a(t),\ha_a(N_a(t)))$ .}
        \STATE \textcolor{blue}{Sample $\ta_{a}(t) \sim \Er \left(N_a(t) -k , \frac{N_a(t)}{\ba_a(N_a(t))}\right)$.} 
        \IF{$\{a \in [K]: \ta_a(t) \leq 1\} \ne \emptyset$}  
            \STATE Select $j(t) = \argmin_{a\in[K]} \ta_a(t)$.
        \ELSE 
            \STATE Sample $u_a \sim U(0,1)$ for every $a \in [K]$. 
            \STATE $\tk_{a}(t) = \hk_{a}(N_a(t)) u_a^{1/(N_a(t) \ta_{a}(t))}$.
            \STATE Play $j(t) = \argmax_{a\in [K]} \frac{\tk_{a}(t) \ta_{a}(t) }{\ta_{a}(t) -1}$
            \STATE \hspace{4.3em} $=\argmax_{a\in [K]} \tmu_a(t)$.
        \ENDIF
    \end{algorithmic}
\end{algorithm}

\subsection{Sampling procedure}
Let $\Ft := ( j(s), r_{j(s), N_{j(s)}(s)})_{s=1}^{t-1}$ be the history until round $t$.
Under the prior $\frac{\alpha^{-k}}{\kappa}$ with $k\in\mathbb{Z}$, the marginalized posterior distribution of the shape parameter of arm $a$ is given as
\begin{equation}\label{eq: pdfalpha}
    \alpha_a \mid \Ft \sim \Er\left(N_a(t)-k, \frac{N_a(t)}{\ha_a(N_a(t))}\right),
\end{equation}
where $\Er(s, \beta)$ denotes the Erlang distribution with shape $s$ and rate $\beta$.
Note that we require $\bn \geq \max\{ 2, k+1\}$ initial plays to avoid improper posteriors and MLE of $\alpha$.
When the shape parameter $\alpha_a$ is given as $\beta$, the cumulative distribution function (CDF) of the conditional posterior of $\kappa_a$ is given as~
\begin{equation}\label{eq: cdfkappa}
    \mathbb{P}\left[\kappa_a \leq x \mid \Ft, \alpha_a=\beta \right] = \left( \frac{x}{\hk_a(N_a(t))}\right)^{\beta N_a(t)},
\end{equation}
if $0 < x \leq \hk_a(N_a(t))$.
Since one can derive the posteriors following the same steps as \citet{Pareto_inference}, the detailed derivation is postponed to Appendix~\ref{sec: derivation_posterior}.
At round $t$, we denote the sampled scale and shape parameters of arm $a$ by $\tk_a(t)$ and $\ta_a(t)$, respectively, and the corresponding mean reward by $\tmu_a(t) := \mu(\tk_a(t), \ta_a(t))$.
We first sample the shape parameter from the marginalized posterior in (\ref{eq: pdfalpha}).
Then, we sample the scale parameter given the sampled shape parameter from the CDF of the conditional posterior in (\ref{eq: cdfkappa}) by using inverse transform sampling.
TS based on this sequential procedure, which we call Sequential Thompson Sampling ($\TS$), can be formulated as Algorithm~\ref{alg: STS}.

In Theorem~\ref{thm: RegSubOopt} given in the next section, $\TS$ with the Jeffreys prior and the reference prior turns out to be suboptimal in view of the lower bound in (\ref{eq: lowerbound}).
Their suboptimality is mainly due to the behavior of the posterior in (\ref{eq: pdfalpha}) when $\ha_1(n)$ is overestimated for small $N_1(t)=n$.
To overcome such issues, we propose the $\TST$ policy, a variant of $\TS$ with \underline{t}runcation, where we replace $\ha(n)$ with $\ba(n) := \min\left( n, \ha(n) \right)$ in (\ref{eq: pdfalpha}).
Note that such a truncation procedure is especially considered in the posterior sampling by (\ref{eq: pdfalpha}) and (\ref{eq: cdfkappa}).
We show that $\TST$ with the Jeffreys prior and the reference prior can achieve the optimal regret bound in Theorem~\ref{thm: RegOptTruncation}.

\subsection{Interpretation of the prior parameter $k$}\label{sec: interpretation}
The Erlang distribution is a special case of the Gamma distribution, where the shape parameter is a positive integer.
If a random variable $X$ follows $\Er(s,\beta)$, then it has the density of form
\begin{equation}\label{eq: pdf_erlang}
    f_{s, \beta}^{\mathrm{Er}}(x) = \frac{\beta^s}{\Gamma(s)}x^{s-1}e^{-\beta x} \I[x\in\mathbb{R}_{+}],
\end{equation}
where $s \in \mathbb{N}$ and $\beta>0$ denote the shape and rate parameter, respectively.
Then, the CDF evaluated at $x>0$ is given as
\begin{equation} \label{eq: def_inc_gamma}
    F_{s,\beta}^{\mathrm{Er}}(x) = \frac{\int_0^{\beta x} t^{s-1 } e^{-t}\dx t}{\Gamma(s)} = \frac{\gamma(s, \beta x)}{\Gamma(s)},
\end{equation}
where $\gamma(\cdot, \cdot)$ denotes the lower incomplete gamma function.
Since $\gamma(s+1, x) = s\gamma(s,x) -x^s e^{-x}$ holds, one can observe that for any $x >0$
\begin{equation}\label{eq: CDF_dec}
    F_{s,\beta}^{\mathrm{Er}}(x) \geq F_{s+1,\beta}^{\mathrm{Er}}(x).
\end{equation}
From the sampling procedure of $\TS$ and $\TST$, $\tmu$ depends on $\tk$ only when $\ta > 1$ holds since $\ta \leq 1$ results in $\tmu(\cdot, \ta) = \infty$. 
Therefore, for any $\beta > 1$ in (\ref{eq: cdfkappa}), $\tk$ will concentrate on $\hk$ for sufficiently large $N_a(t)=n$.
Thus, $\tmu$ will be mainly determined by $\ta$ and $\hk$, where the choice of $k$ affects the sampling of $\ta$ by (\ref{eq: pdfalpha}).
From (\ref{eq: CDF_dec}), one could see that the probability of sampling small $\ta$ increases as shape $n-k$ decreases.
Therefore, $\tmu$ of suboptimal arms would increase as $k$ increases for the same $n$.
In other words, the probability of sampling large $\tmu$ becomes large as $k$ increases.
Therefore, TS with large $k$ becomes a conservative policy that could frequently play currently suboptimal arms. 
In contrast, priors with small $k$ yield an optimistic policy that focuses on playing the current best arm.

\section{Main results}\label{sec: main}
In this section, we provide regret bounds of $\TS$ and $\TST$ with different choices of $k \in \mathbb{Z}$.
At first, we show the asymptotic optimality of $\TS$ for priors $\pi(\ka) \propto \frac{\alpha^{-k}}{\kappa}$ with $k \in \mathbb{Z}_{\geq 2}$.
\begin{restatable}{theorem}{thmRegOpt}\label{thm: RegOpt}
Assume that arm $1$ is the unique optimal arm with a finite mean.
For every $a \in [K]$, let $\veps_a = \min \left \{ \frac{\kappa_a}{\alpha_a(\kappa_a+1)} \frac{\kappa_a \delta_a}{\mu_a(\mu_a + \delta_a - \kappa_a)+ \kappa_a \delta_a}, \frac{\kappa_a \delta_a}{\mu_a(1+\mu_a + \delta_a)} \right\}$ where $\delta_a = \frac{\Delta_a}{2}$ for $a \ne 1$ and $\delta_1 = \min_{a\ne 1} \delta_a$.
Given arbitrary $\eps \in (0, \min_{a\in [K]} \veps_a)$, the expected regret of $\TS$ with $k \in \mathbb{Z}_{\geq 2}$ is bounded as 
\begin{equation*}
        \mathbb{E}[\mathrm{Reg}(T)] \leq \sum_{a=2}^K \frac{\Delta_a \log T}{D_{a,k}(\eps)} + \mathcal{O}\left(\eps^{-2}\right),
\end{equation*}
where $D_{a,k}(\eps)$ is a function such that $ \lim_{\eps \to 0} D_{a,k}(\eps) = \KLinf(a)$ for any fixed $k \in \mathbb{Z}$.
\end{restatable}
By letting $\eps = o\left(1\right)$ in Theorem~\ref{thm: RegOpt}, we see that $\TS$ with $ k \in \mathbb{Z}_{\geq 2}$ satisfies
\begin{equation*}
    \liminf_{T \to \infty} \frac{\mathbb{E}[\mathrm{Reg}(T)]}{\log T} \leq \sum_{a=2}^K \frac{\Delta_a}{\KLinf(a)},
\end{equation*}
which shows the asymptotic optimality of $\TS$ in terms of the lower bound in (\ref{eq: lowerbound}).

Next, we show that $\TS$ with $k \in \mathbb{Z}_{\leq 1}$ cannot achieve the asymptotic bound in the theorem below.
Following the proofs for Gaussian bandits~\citep{honda2014optimality}, we consider two-armed bandit problems where the full information on the suboptimal arm is given to simplify the analysis.
We further assume that two arms have the same scale parameter $\kappa_1 = \kappa_2$.
\begin{restatable}{theorem}{thmRegSubOpt}\label{thm: RegSubOopt}
Consider a two-armed bandit problem where $\kappa_1 = \kappa_2$ and $1 < \alpha_1 < \alpha_2$.
When $\ta_1(t)$ and $\tk_1(t)$ are sampled from the posteriors in (\ref{eq: pdfalpha}) and (\ref{eq: cdfkappa}), respectively and $\tmu_2(t) = \mu_2$ holds, there exists a constant $C(\alpha_1, \alpha_2)>0$ satisfying 
\begin{equation*}
    \liminf_{T \to \infty} \frac{\mathbb{E}[\reg(T)]}{\log T} \geq C(\alpha_1, \alpha_2),
\end{equation*}
where $C(\alpha_1, \alpha_2)> \KLinf(2)$ holds for some instances.
In particular, for $k \in \mathbb{Z}_{\leq 0}$, there exists $C'(\alpha_1, \alpha_2)$ satisfying
\begin{equation*}
    \liminf_{T \to \infty} \frac{\mathbb{E}[\reg(T)]}{\sqrt{T}} \geq C'(\alpha_1, \alpha_2).
\end{equation*}
\end{restatable}
From Theorems~\ref{thm: RegOpt} and~\ref{thm: RegSubOopt}, we find that the prior should be conservative to some extent when one considers maximizing rewards in \emph{expectation}.

Although $\TS$ with the Jeffreys prior ($k=0$) and reference prior ($k=1$) were shown to be suboptimal, we show that a modified algorithm, $\TST$, can achieve the optimal regret bound with $k \in \mathbb{Z}_{\geq 0}$.
\begin{restatable}{theorem}{thmRegOptTruncation}\label{thm: RegOptTruncation}
With the same notation as Theorem~\ref{thm: RegOpt}, the expected regret of $\TST$ with $k \in \mathbb{Z}_{\geq 0}$ is bounded as 
\begin{equation*}
        \mathbb{E}[\mathrm{Reg}(T)] \leq \sum_{a=2}^K \frac{\Delta_a \log T}{D_{a,k}(\eps)} + \mathcal{O}(\eps^{-m}),
\end{equation*}
where $m=\max(2, 3-k)$.
\end{restatable}
From Theorems~\ref{thm: RegOpt} and~\ref{thm: RegOptTruncation}, we have two choices to achieve the lower bound in (\ref{eq: lowerbound}): use either the conservative priors with MLEs or moderately optimistic priors with truncated samples.
Since initialization steps require playing every arm $\max(2, k+1)$ times, if the number of arms $K$ is large, the Jeffreys priors or the reference prior with the truncated estimator would be a better choice.
On the other hand, if the model can contain arms with large $\alpha$, where the truncation might be problematic for small $n$, it would be better to use $\TS$ with conservative priors.

\section{Experiments}\label{sec: experiments}
In this section, we present numerical results to demonstrate the performance of $\TS$ and $\TST$, which supports our theoretical analysis.
We consider the $4$-armed bandit model $\bth_4$ with parameters given in Table~\ref{tab: 4armbandit} as an example where suboptimal arms have smaller, equal, and larger $\kappa$ compared with the optimal arm.
$\bth_4$ has $\bmu=(4.55, \,3.2,\, 2.74,\, 3)$ and infinite variance.
Further experimental results can be found in Appendix~\ref{app: Exp}.
\begin{table}[h]
\caption{$4$-armed bandit model $\bth_4$.}
\label{tab: 4armbandit}
    \begin{center}
    \begin{sc}
    \begin{tabular}{lcccc}
         & arm $1$ & arm $2$ & arm $3$  &arm $4$\\
         \toprule
        $\kappa$ & 1.3 & 1.2 & 1.3 & 1.5 \\
        $\alpha$ & 1.4 & 1.6 & 1.9 & 2.0
    \end{tabular}
    \end{sc}
    \end{center}
\end{table}

Figure~\ref{fig: overall} shows the cumulative regret for the proposed policies with various choices of parameters $k$ on the prior.
The solid lines denote the averaged cumulative regret over 100,000 independent runs of priors that can achieve the optimal lower bound in (\ref{eq: lowerbound}), whereas the dashed lines denote that of priors that cannot.
The green dotted line denotes the problem-dependent lower bound and shaded regions denote a quarter standard deviation.

In Figures~\ref{fig: all_k} and~\ref{fig: all_k_remain}, we investigate the difference between $\TS$ and $\TST$ with the same $k$.
The solid lines denote the averaged cumulative regret over 100,000 independent runs.
The shaded regions and dashed lines show the central $99\%$ interval and the upper $0.05\%$ of regret.
\paragraph{The Jeffreys prior ($k=0$)}
In Figure~\ref{fig: STS}, the Jeffreys prior seems to have a larger order of the regret compared with priors $k=1,3$, which performed the best in this setting.
As Theorem~\ref{thm: RegOptTruncation} states, its performance improves under $\TST$, which shows a similar performance to that of $k=1,3$.
\begin{figure*}[!t]
    \centering
    \subfigure[Cumulative regret of $\TS$ with various $k$]{\label{fig: STS}\includegraphics[width=0.4\textwidth, height=0.4\textwidth]{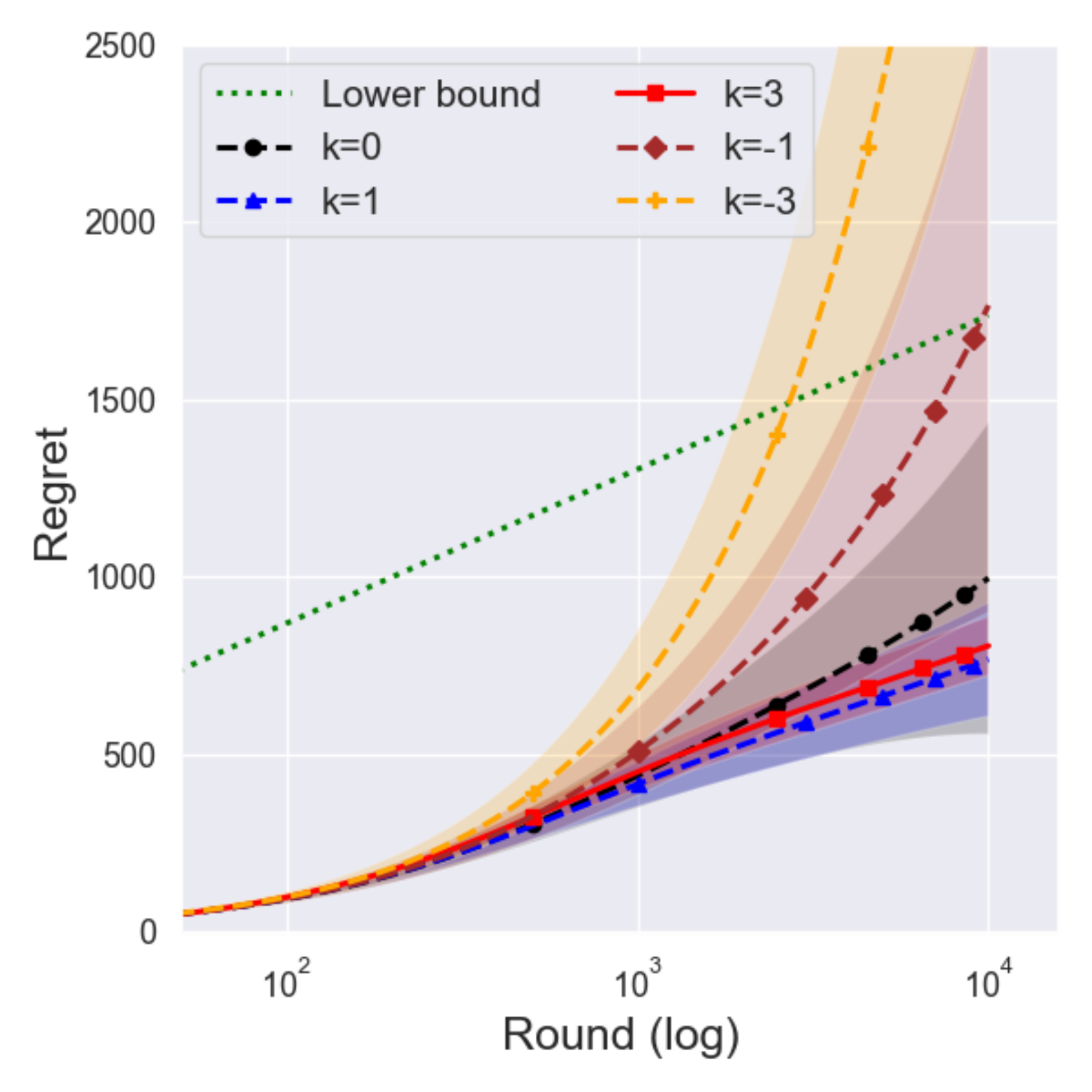}}
    \quad
    \subfigure[Cumulative regret of $\TST$ with various $k$]{\label{fig: STST}\includegraphics[width=0.4\textwidth, height=0.4\textwidth]{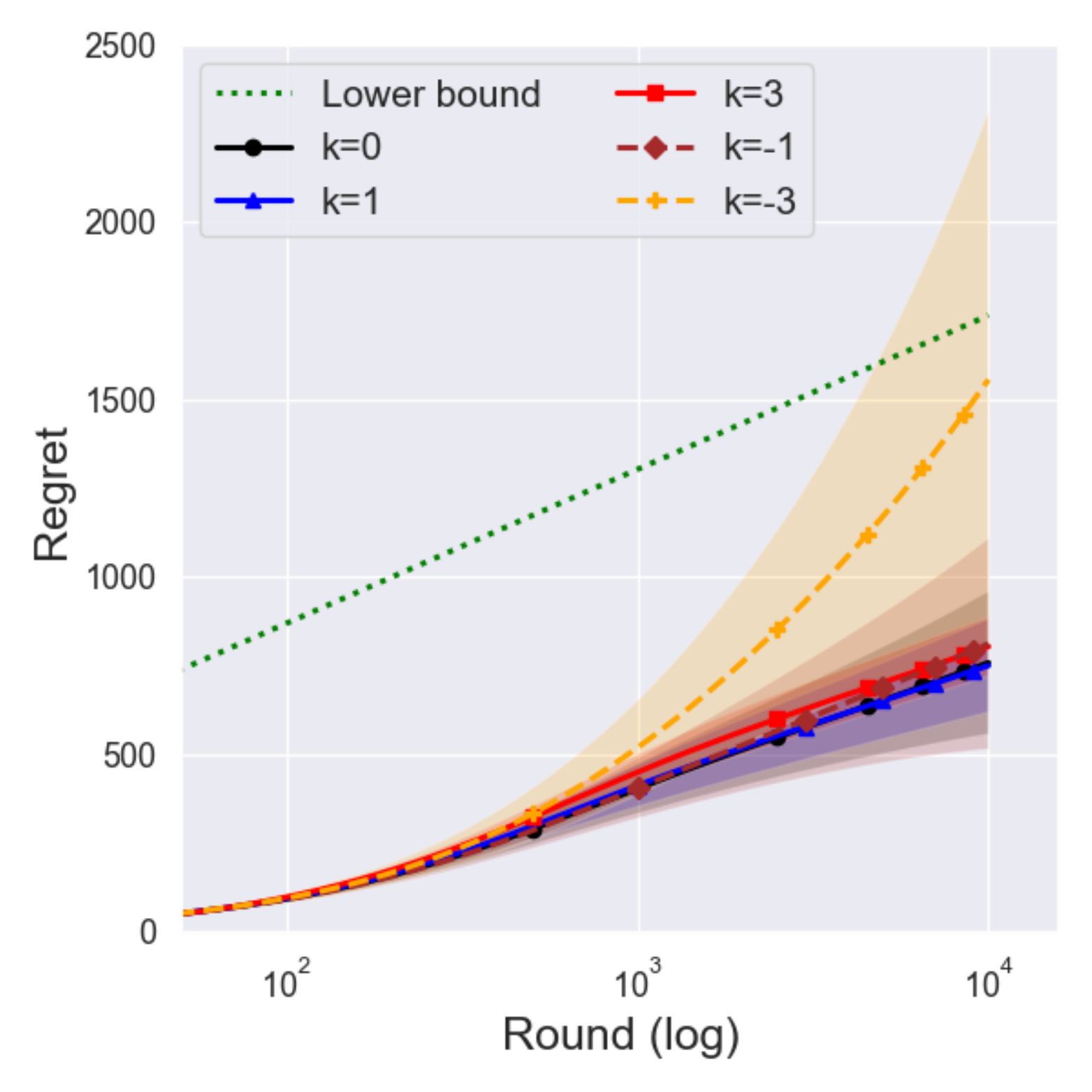}}
\caption{The solid lines denote the averaged cumulative regret over 100,000 independent runs of priors that can achieve the optimal lower bound in (\ref{eq: lowerbound}).
The dashed lines denote that of priors that cannot achieve the optimal lower bound in (\ref{eq: lowerbound}).
The shaded regions show a quarter standard deviation.
The green dotted line denotes the problem-dependent lower bound based on Lemma~\ref{lem: KLinf}.}
\label{fig: overall}
\end{figure*}
\begin{figure*}[!t]
\centering
    \subfigure[The Jeffreys prior $k=0$]{\label{fig: k0}\includegraphics[width=0.32\textwidth, height=0.3\textwidth ]{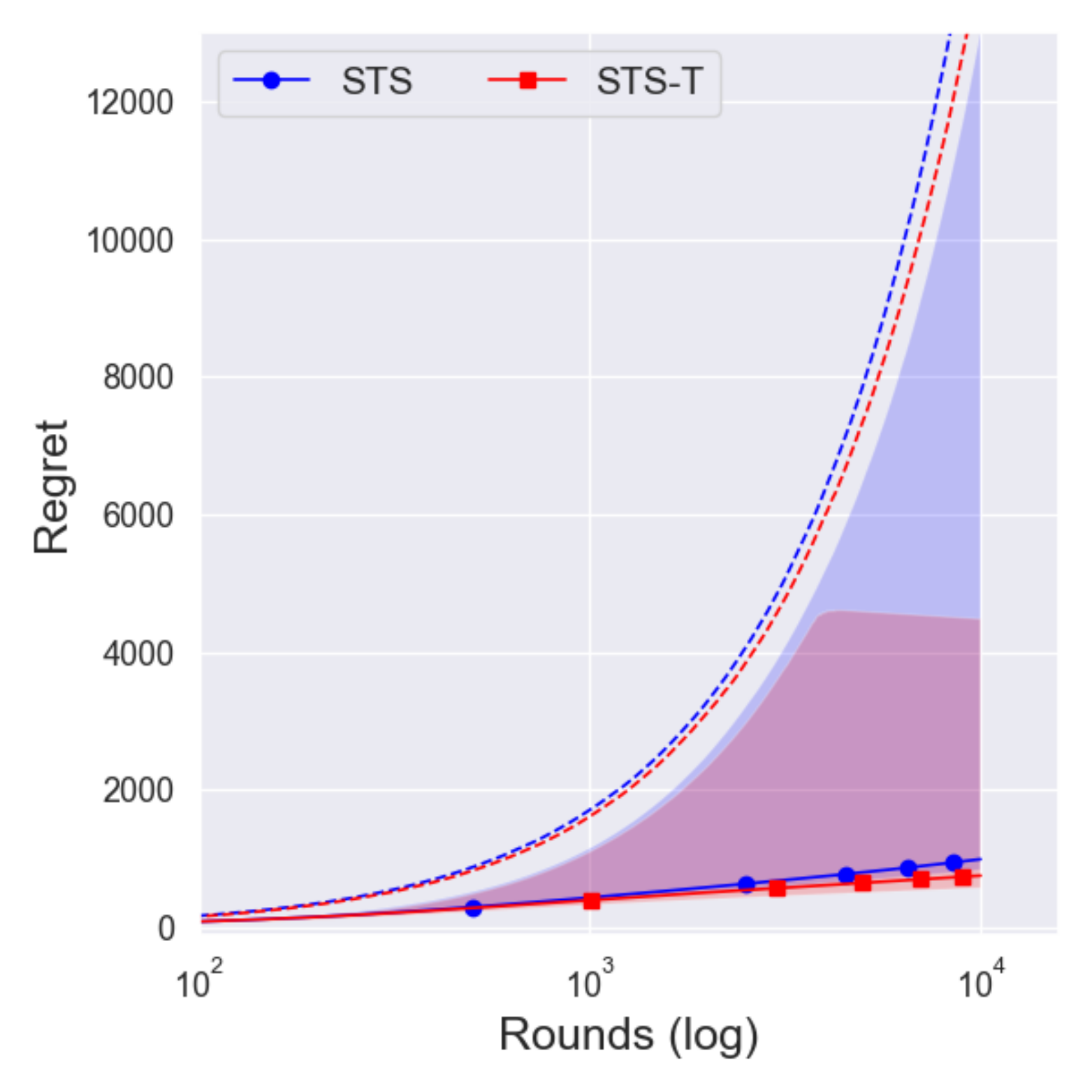}}
     \hfill
    \subfigure[The reference prior $k=1$]{\label{fig: k1}\includegraphics[width=0.32\textwidth, height=0.3\textwidth]{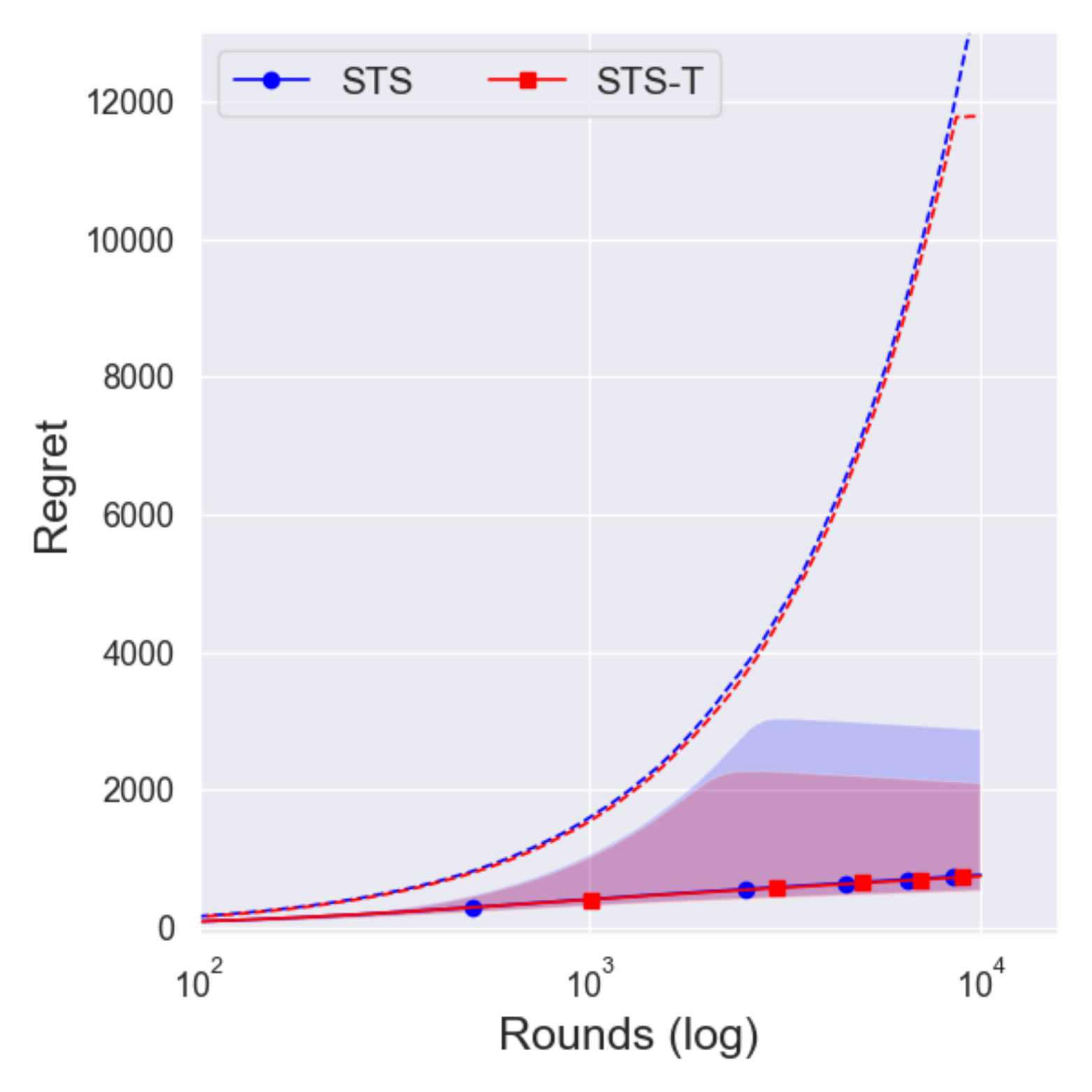}}
    \hfill
    \subfigure[Prior with $k=3$]{\label{fig: k3}\includegraphics[width=0.32\textwidth, height=0.3\textwidth]{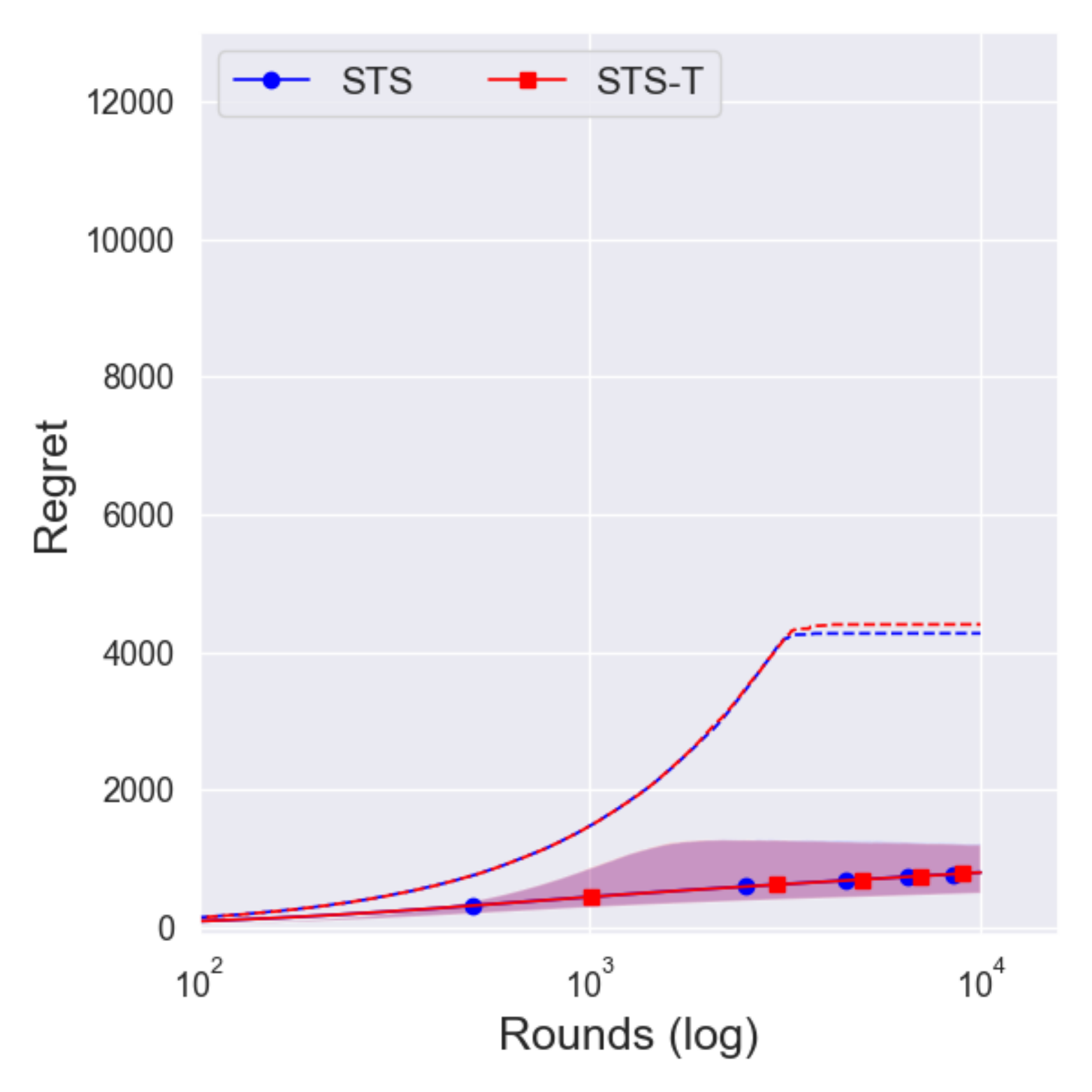}} 
\caption{The solid lines denote an averaged regret over independent 100,000 runs. The shaded regions and dashed lines show the central $99\%$ interval and the upper $0.05\%$ of the regret, respectively.}
\label{fig: all_k}
\end{figure*}
Figure~\ref{fig: k0} illustrates the possible reason for the improvements, where the central $99\%$ interval of the regret noticeably shrank under $\TST$.
Since the suboptimality of $\TS$ with the Jeffreys prior ($k=0$) is due to an extreme case that induces a polynomial regret with small probability, this kind of shrink contributes to decreasing the expected regret of $\TST$ with the Jeffreys prior. 

\paragraph{The reference prior ($k=1$)}
The reference prior showed a similar performance to the asymptotically optimal prior $k=3$, although it was shown to be suboptimal for some instances under $\TS$ in Theorem~\ref{thm: RegSubOopt}.
Similarly to the Jeffreys prior~$(k=0)$, the reference prior~($k=1$) under $\TST$ has a smaller central $99\%$ interval of the regret than that under $\TS$ as shown in Figure~\ref{fig: k1}, although its decrement is comparably smaller than that of the Jeffreys prior.
This would imply that the reference prior is more conservative than the Jeffreys prior.
\begin{figure*}[!ht]
\centering
    \subfigure[Prior with $k=-1$]{\label{fig: k-1}\includegraphics[width=0.4\textwidth, height=0.38\textwidth]{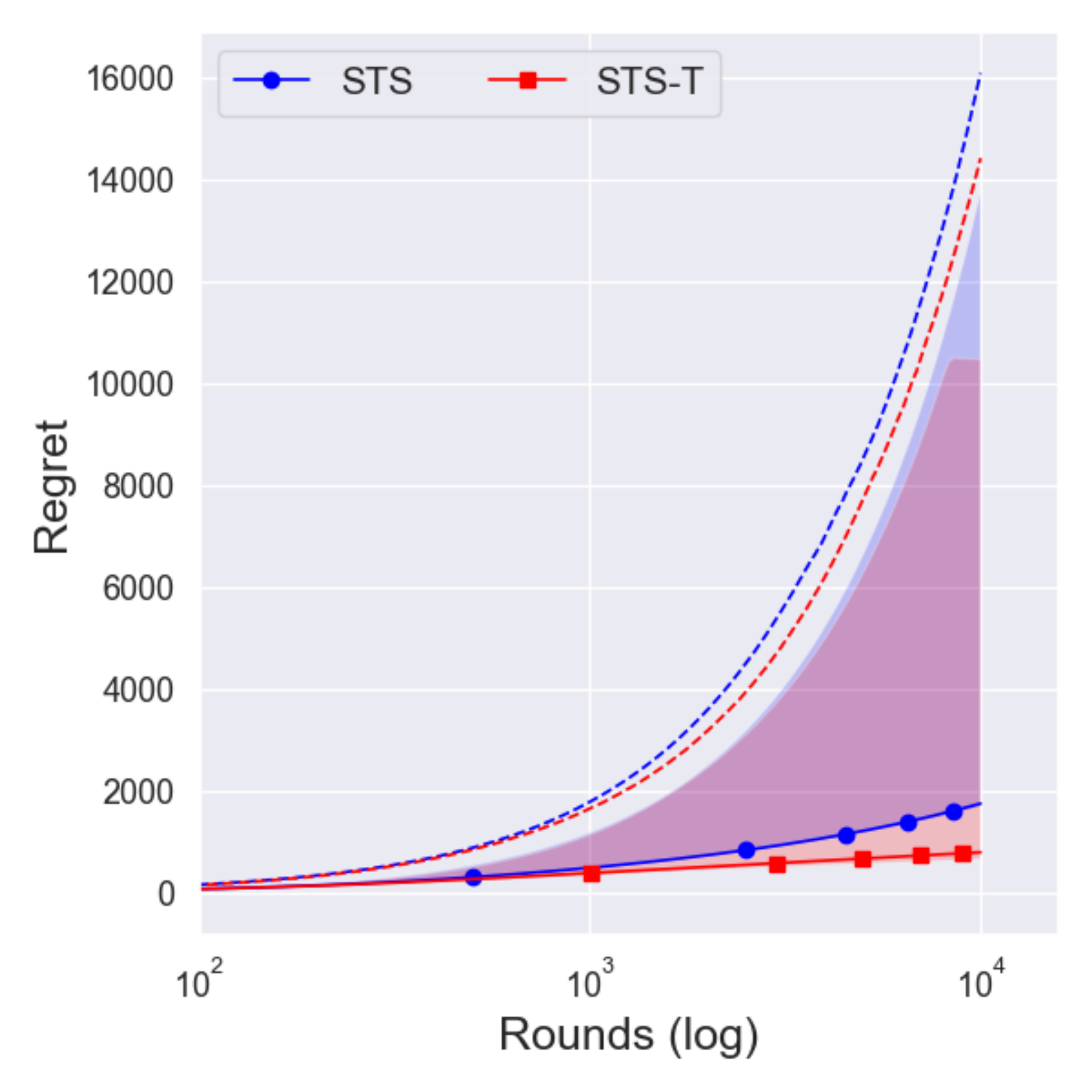}}
     \quad
    \subfigure[Prior with $k=-3$]{\label{fig: k-3}\includegraphics[width=0.4\textwidth, height=0.38\textwidth]{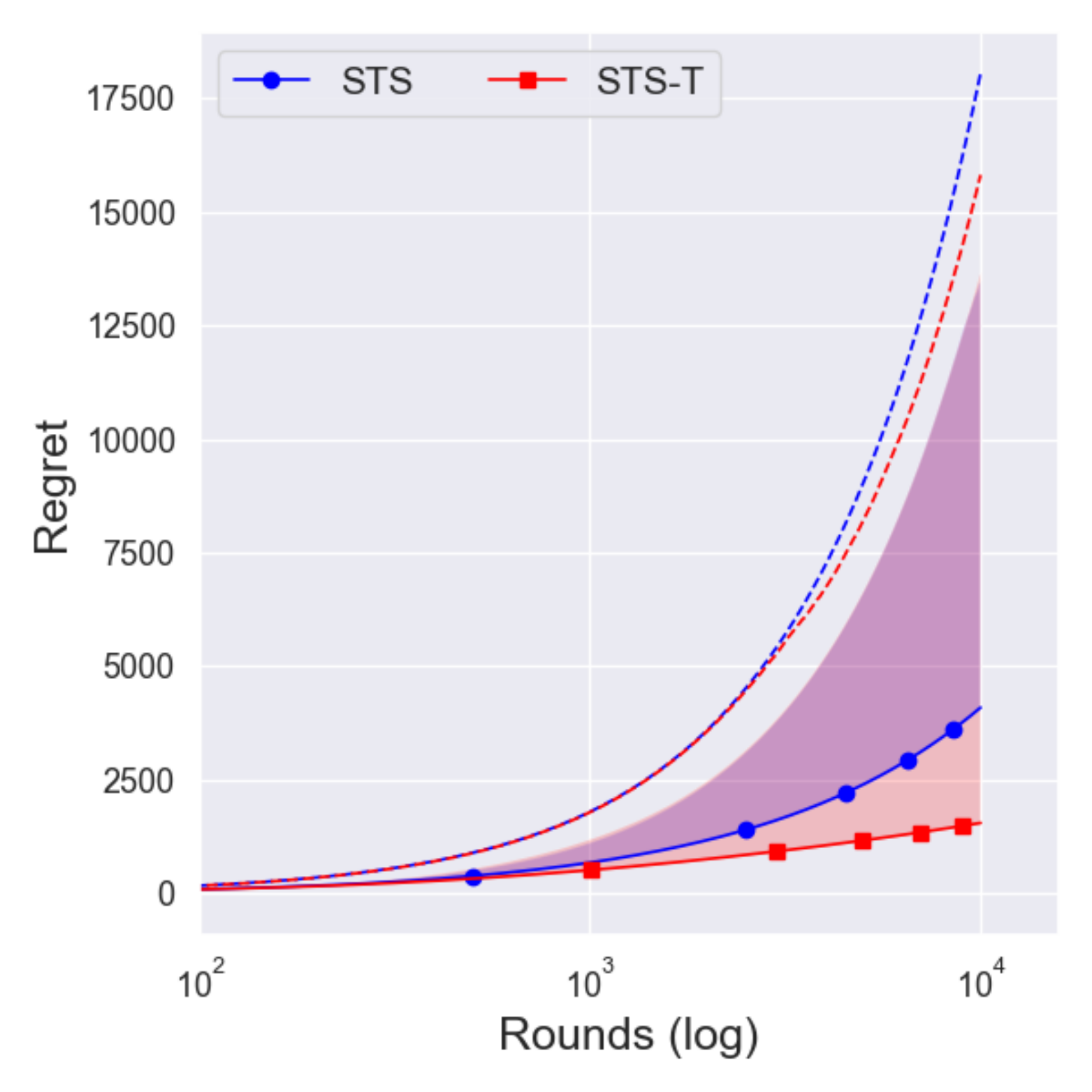}}
\caption{The solid lines denote an averaged regret over independent 100,000 runs. The shaded regions and dashed lines show the central $99\%$ interval and the upper $0.05\%$ of the regret, respectively.}
\label{fig: all_k_remain}
\end{figure*}
\paragraph{The conservative prior ($k=3$)}
Interestingly, Figure~\ref{fig: k3} showed that a truncated procedure does not affect the central $99\%$ interval of the regret and even degrade the performance in upper $0.05\%$.
Notice that the upper $0.05\%$ of the regret of $k=3$ is much lower than that of $k=0,1$, which shows the stability of the conservative prior in Figure~\ref{fig: all_k}.

Since a truncation procedure was adopted to prevent an extreme case that was a problem for $k\in\mathbb{Z}_{\leq 1}$, it is natural to see that there is no difference between $\TS$ and $\TST$ with $k=3$. 
This would imply that $k=3$ is sufficiently conservative, and so the truncated procedure does not affect the overall performance.

\paragraph{Optimistic priors ($k<0$)}
In Figure~\ref{fig: STS}, one can see that the averaged regret of $k=-1$ and $k=-3$ increases much faster than that of $k=0,1,3$ under the $\TS$ policy, which illustrates the suboptimality of $\TS$ with priors $k \in \mathbb{Z}_{<0}$.

As the optimistic priors ($k<0$) showed better performance under $\TST$ in Figure~\ref{fig: overall}, we can check the effectiveness of a truncation procedure in the posterior sampling with optimistic priors.
However, as Figures~\ref{fig: k-1} and~\ref{fig: k-3} illustrate, there is no big difference in the central $99\%$ interval of the regret between $\TS$ and $\TST$ with $k=-1,-3$, which might imply that a prior with $k \in \mathbb{Z}_{<0}$ is too optimistic.
Therefore, we might need to use a more conservative truncation procedure such as the one using $\bar{\alpha}_a(n) = \max( \sqrt{n}, \ha_{a}(n))$ or $\max( \log n, \ha_{a}(n))$, which would induce a larger regret in the finite time horizon.

\section{Proof outline of optimal results}\label{sec: optimality}
In this section, we provide the proof outline of Theorem~\ref{thm: RegOpt} and Theorem~\ref{thm: RegOptTruncation}, whose detailed proof is given in Appendix~\ref{sec: Opt}.
Note that the proof of Theorem~\ref{thm: RegSubOopt} is postponed to Appendix~\ref{sec: SubOpt}.

Let us first consider good events on MLEs defined by
\begin{align*}
    \eK_{a,n}(\eps)  &:= \{ \hk_{a}(n) \in [\kappa_a, \kappa_a + \eps] \} \\
    \eA_{a,n}(\eps)  &:= \{ \ha_{a}(n) \in [\alpha_a - \eps_{a,l}(\eps), \alpha_a + \eps_{a,u}(\eps)] \} \\
    \eE_{a,n}(\eps)  &:= \eK_{a,n}(\eps) \cap \eA_{a,n}(\eps), 
\end{align*}
where $n\in\mathbb{N}$ and
\begin{equation}\label{eq: epslu}
    \eps_{a,l}(\eps) = \frac{\eps \alpha_a^2}{1+\eps \alpha_a}, \hspace{1em} \eps_{a,u}(\eps) = \frac{\eps \alpha_a^2(\kappa_a+1)}{\kappa_a - \eps\alpha_a(\kappa_a +1)}.
\end{equation}
Note that $\ba_a(n) = \ha_a(n)$ holds on $\eA_{a,n}(\eps)$ for any $n \geq \alpha_a+1$.
Here, we set $\veps_a$ to satisfy $\hmu_a \in \left[ \mu_a -\delta_a, \mu_a + \delta_a  \right]$ on $\eE_a(\eps)$ for any $\eps \leq \veps_a$.
Define an event on the sample of the optimal arm
\begin{equation*}
    \eM_{\eps}(t) := \{ \tmu_1(t) \geq \mu_1 - \eps \}.
\end{equation*}
Then, the expected regret at round $T$ can be decomposed as follows:
\begin{align*}
    \mathbb{E}[\reg (T)] &= \mathbb{E}\left[\sum_{t=1}^T \Delta_{j(t)}\right] 
   \\&= \sum_{a=2}^K \Delta_a \left( \bn + \sum_{t=\bn K+1}^T \mathbb{E}[\I [j(t)=a]] \right) 
    \\
    &\leq \Delta_{2} \sum_{t=\bn K+1}^T \bigg(\mathbb{E}\left[\I [j(t)\ne 1, \eK_{1,N_1(t)}(\eps), \eM_{\eps}^c(t)]\right] + \mathbb{E}\left[\I [j(t)\ne 1, \eK_{1,N_1(t)}^c(\eps), \eM_{\eps}^c(t)]\right] \bigg) \\
    &\quad+  \sum_{a=2}^K \Delta_a \bigg\{ \bn +  \sum_{t=\bn K+1}^T \bigg(\mathbb{E}\left[\I [j(t)=a, \eM_{\eps}(t), \eE_{a,N_a(t)}(\eps) ]\right] \\
    &\hspace{15em}+ \mathbb{E}\left[ \I [j(t)=a, \eM_{\eps}(t), \eE_{a,N_a(t)}^c(\eps)]\right]  \bigg) \bigg\},
\end{align*}
where $\eE^c$ denotes the complementary set of $\eE$.
Lemmas~\ref{lem: difficult}--\ref{lem: minor} complete the proof of Theorems~\ref{thm: RegOpt} and~\ref{thm: RegOptTruncation}, whose proofs are given in Appendix~\ref{sec: Opt}.

\begin{restatable}{lemma}{lemdifficult}\label{lem: difficult}
Under $\TS$ with $k \in \mathbb{Z}_{\geq 2}$,
\begin{equation*}
        \sum_{t=\bn K+1}^T \mathbb{E}\left[\I [j(t) \ne 1, \eK_{1,N_1(t)}^c(\eps), \eM_{\eps}^c(t)]\right] \leq \mathcal{O}(\eps^{-2}).
\end{equation*}
and under $\TST$ with $k \in \mathbb{Z}_{\geq 0}$,
\begin{equation*}
        \sum_{t=\bn K+1}^T \mathbb{E}\left[\I [j(t) \ne 1, \eK_{1,N_1(t)}^c(\eps), \eM_{\eps}^c(t)]\right] \leq \mathcal{O}(\eps^{-m}),
\end{equation*}
where $m=\max(2, 3-k)$.
\end{restatable}
Although Lemma~\ref{lem: mainreg} contributes to the main term of the regret, the proof of Lemma~\ref{lem: difficult} is the main difficulty in the regret analysis.
We found that our analysis does not result in a finite upper bound for $\TS$ with $k\in \mathbb{Z}_{< 2}$ and designed $\TST$ to solve such problems.
\begin{restatable}{lemma}{lemmainreg}\label{lem: mainreg}
Under $\TS$ and $\TST$ with $k \in \mathbb{Z}$, it holds that for any $a \in [K]$
\begin{align*}
        \sum_{t=\bn K+1}^T \mathbb{E}&[\I [j(t)=a, \eM_{\eps}(t), \eE_{a,N_a(t)}(\eps) ]] \leq \max(0,k) + 1 + \frac{1}{\alpha_a \eps} \I[k >0] + \frac{\log T}{D_{a,k}(\eps)}.
\end{align*}
where $D_{a,k}(\eps) > 0$ is a finite problem-deterministic constant satisfying $\lim_{\eps \to 0} D_{a,k}(\eps) = \KLinf(a)$.
\end{restatable}
Since large $k$ yields a more conservative policy and it requires additional initial plays of every arm, large $k$ might induce larger regret for a finite time horizon $T$, which corresponds to the component of the regret discussed in Lemma~\ref{lem: mainreg}.
Thus, this lemma would imply that the policy has to be conservative to some extent and being overly conservative would induce larger regrets in a finite time.
\begin{restatable}{lemma}{lemkordaused}\label{lem: korda}
Under $\TS$ and $\TST$ with $k \in \mathbb{Z}_{\geq 0}$,
\begin{equation*}
    \sum_{t=\bn K+1}^T \mathbb{E}\left[\I [j(t)\ne 1, \eK_{1,N_1(t)}(\eps), \eM_{\eps}^c(t)]\right] \leq \mathcal{O}(\eps^{-1}).
\end{equation*}
\end{restatable}
The key to Lemma~\ref{lem: korda} is to convert the term on $\tmu_1(t)$, $\eM_\eps(t)$, to a term on $\ta_1(t)$.
Since $\mu(\ka)=\infty$ holds for $\alpha \leq 1$, $\tmu_1=\mu(\tk_1, \ta_1)$ becomes infinity regardless of the value of $\tk_1$ if $\ta_1 \leq 1$ holds, which implies $\mathbb{P}[\eM_\eps^c(t), \ta_1(t)\leq 1] = 0$.
Therefore, it is enough to consider the case where $\ta_1(t) > 1$ holds to prove Lemma~\ref{lem: korda}.
Although density functions of $\ta_1$ under $\TS$ and $\TST$ are different, conditional CDFs of $\tk_1$ given $\alpha_1=\ta_1$ are the same, which is given in (\ref{eq: cdfkappa}) as
\begin{equation*}
    \mathbb{P}[\tk_1 \leq x | \Ft,\ta_1=\alpha_1 ] = \left( \frac{x}{\hk_1(N_1(t))} \right)^{\ta_1 N_1(t)}.
\end{equation*}
Therefore, for sufficiently large $N_1(t)$ and $\ta_1(t)~>~1$, $\tk_1(t)$ will concentrate on $\hk_1(N_1(t))$ with high probability, which is close to its true value $\kappa_1$ under the event $\{\eK_{1, N_1(t)}(\eps)\}$.
Thus, $\tmu_1 = \frac{\tk_1 \ta_1}{\ta_1 -1}\geq~\frac{\kappa_1 \ta_1}{\ta_1 -1} = \mu(\kappa_1, \ta_1)$ holds with high probability, which implies that $\mathbb{P}[\eK_{1,N_1(t)}(\eps), \eM_{\eps}^c(t)| \Ft]$ can be roughly bounded by $\mathbb{P}[\eK_{1,N_1(t)}(\eps), \ta_1(t) \geq c| \Ft]$ for some problem-dependent constants $c>1$.
Since $\eK_1$ is deterministic given $\Ft$, we have 
\begin{align*}
    \mathbb{P}[\eK_{1,N_1(t)}(\eps), \,&\ta_1(t) \geq c| \Ft] = \I[\eK_{1,N_1(t)}(\eps)]\mathbb{P}[ \ta_1(t) \geq c| \Ft],
\end{align*}
which implies $\tmu_1(t)$ is mainly determined by the value of $\ta_1(t)$ under the event $\{\eK_{1, N_1(t)}(\eps)\}$ for both policies.
In such cases, $\TS$ and $\TST$ behave like TS in the Pareto distribution with a known scale parameter, where $\tmu_1(t) := \mu(\kappa_1, \ta_1(t))$ for $t \in \mathbb{N}$.
Here, the Pareto distribution with the known scale parameter belongs to the one-dimensional exponential family, where \citet{KordaTS} showed the optimality of TS with the Jeffreys prior.
Since the posterior of $\alpha$ under the Jeffreys prior is given as the Erlang distribution with shape $N_1(t)+1$ in the one-parameter Pareto model, we can apply the results by \citet{KordaTS} to prove Lemma~\ref{lem: korda} by using some properties of the Erlang distribution such as (\ref{eq: CDF_dec}).

\begin{restatable}{lemma}{lemminor}\label{lem: minor}
Under $\TS$ and $\TST$ with $k \in \mathbb{Z}$, it holds that for any $a \ne 1$
\begin{equation*}
    \sum_{t=\bn K+1}^T \mathbb{E}\left[\I [j(t)=a, \eE_{a,N_a(t)}^c(\eps)]\right] \leq \mathcal{O}\left(\eps^{-2}\right).
\end{equation*}
\end{restatable}
Lemma~\ref{lem: minor} controls the regret induced when estimators of the played arm are not close to their true parameters, which is not difficult to analyze as in the usual analysis of TS.
In fact, the proof of this lemma is straightforward since the upper bounds of $\mathbb{P}[\eK_a^c]$ and $\mathbb{P}[\eK_a, \eA_a^c]$ can be easily derived based on the distributions of $\hk_a$ and $\ha_a$ in (\ref{eq: MLEdist}).

\section{Conclusion}\label{sec: conclusion}
We considered the MAB problems under the Pareto distribution that has a heavy tail and follows the power-law.
While most previous research on TS has focused on one-dimensional or light-tailed distributions, we focused on the Pareto distribution characterized by unknown scale and shape parameters.
By sequentially sampling parameters via their marginalized and conditional posterior distributions, we can realize an efficient sampling procedure.
We showed that TS with the appropriate choice of priors achieves a problem-dependent optimal regret bound in such a setting for the first time.
Although the Jeffreys prior and the reference prior are shown to be suboptimal under the direct implementation of TS, we showed that they can achieve the optimal regret bound if we add a truncation procedure.
Experimental results support our theoretical results, which show the optimality of conservative priors and the effectiveness of the truncation procedure for the Jeffreys prior and the reference prior.

\section*{Acknowledgement}
JL was supported by JST SPRING, Grant Number JPMJSP2108.
CC and MS were supported by the Institute for AI and Beyond, UTokyo.

\bibliography{ref}
\bibliographystyle{plainnat}

\newpage
\appendix

\section{Notations}
Tables~\ref{tab: 1}--\ref{tab: 4} summarize the symbols used in this paper.

\begin{table}[!ht]
\caption{Notations for the bandit problem.}
\label{tab: 1}
\centering
\begin{tabular}{ll}
    \toprule
      Symbol  &  Meaning\\
     \midrule
     $K$ & the number of arms. \\
     $T$ & time horizon. \\ 
     $j(t)$ & the index of the played arm at round $t$. \\
     $k\in\mathbb{Z}$ & prior parameter, see Section~\ref{sec: TS} for details. \\
     $\bn =\max(2, k+1)$ & initial plays to avoid improper posteriors. \\
     $N_a(t)$ & the number of playing arm $a$ until round $t$.\\
     $r_{a,n}$ & $n$-th reward generated from the arm $a$.\\
     $\mu(\theta)$ &  the expected value of the random variable following $\Par(\theta)$.\\
     $\mu_a = \mu(\theta_a)$ & the expected rewards of arm $a$. \\
     $\Delta_a$ & sub-optimality gap of arm $a$.\\
     $\delta_a = \frac{\Delta_a}{2}$ for $a\ne 1$ & a half of sub-optimality gap of arm $a$.\\
     $\delta_1 := \min_{a\ne 1} \delta_a$ & defined as the minimum of sub-optimality gap.\\
     $\Ft=(j(s),r_{j(s),N_{j(s)}})_{s=1}^{t-1}$ & the history until round $t$. \\
     $\mathbb{P}_t[\cdot]:=\mathbb{P}[\cdot | \Ft]$ & conditional probability given $\Ft$.\\
     $g_a(c, \alpha)$ & KL-divergence from $\Par\left( \frac{\kappa_a}{c},\alpha \right)$ to $\Par(\kappa_a, \alpha_a)$ for $c \geq 1$.\\
     $h_a(c,\bmu) = h_a(c)$ & the upper bound of $\alpha$ satisfying $\mu\left( \frac{\kappa_a}{c},\alpha \right) \geq \mu_1$ for $c\geq 1$. \\
\end{tabular}
\end{table}
\begin{table}[!ht]
\caption{Notations for probability distributions and estimators}
\label{tab: 2}
\centering
\resizebox{\columnwidth}{!}{\begin{tabular}{ll}
    \toprule
      Symbol  &  Meaning\\
     \midrule
    $\Par(\ka)$ & Pareto distribution with the scale $\kappa >0$ and shape $\alpha >0$. \\
      $f_{\ka}^{\mathrm{Pa}}(x)$ & density function of $\Par(\ka)$.\\
      $\Er(s,\beta)$ & Erlang distribution with the shape $s >0$ and rate $\beta>0$. \\
      $f_{s,\beta}^{\mathrm{Er}}(x)$ &  density function of $\Er(s,\beta)$. \\
      $F_{s,\beta}^{\mathrm{Er}}(x)$ &  CDF of $\Er(s,\beta)$ evaluated at $x>0$. \\
        $\InvG(s, \beta)$ & Inverse Gamma distribution with shape $s>0$ and scale $\beta>0$. \\
        $F_n(x)$ & CDF of the chi-squared distribution with $n$ degree of freedom.\\
      $\Gamma(s)$ & Gamma function. \\
      $\gamma(s,x)$ & the lower incomplete gamma function.\\
      $\Gamma(s,x)$ & the upper incomplete gamma function.\\
      \midrule
      $\hk_a(n),\,\ha_a(n)$ & MLEs of the scale and shape parameter of arm $a$ after $n$ observations, defined in (\ref{eq: MLEdist}). \\
      $\tk_a(t),\,\ta_a(t)$ & sampled parameters at round $t$ from posterior  distribution in (\ref{eq: cdfkappa}) and (\ref{eq: pdfalpha}).\\
      $\ba_a(n)=\max(\ha_a(n), n)$ & truncated estimator of the shape parameter. \\
      $\alpha_{a,n}$ & a temporary notation that can be replaced by both $\ha_a(n)$ ($\TS$) and $\ba_a(n)$ ($\TST$). \\
      $\hmu_a(n) = \mu(\hk_a(n),\ha_a(n))$ & computed mean rewards by the MLEs after $n$ observation.\\
      $\tmu_a(t) = \mu(\tk_a(t),\ta_a(t))$ & computed mean rewards by sampled parameters $\tk_a(t)$ and $\ta_a(t)$ at round $t$.\\
      $\theta_a=(\kappa_a, \alpha_a)$ & tuple of true parameters of arm  $a\in[K]$. \\
      $\hat{\theta}_{a,n} = (\hk_a(n), \ha_a(n))$ & tuple of MLEs of arm $a$ after $n$ observations.\\
      $\bar{\theta}_{a,n} = (\hk_a(n), \ba_a(n))$ & tuple of estimators with a truncation procedure of arm $a$ after $n$ observations. \\
      $\theta_{a,n} = (\hk_a(n), \alpha_{a,n})$ & a temporary notation that can be replaced by both $\hat{\theta}_{a,n}$ ($\TS$) and $\bar{\theta}_{a,n}$ ($\TST$).\\
\end{tabular}}
\end{table}

\begin{table}[!ht]
\caption{Notations for the regret analysis}
\label{tab: 3}
\centering
\resizebox{\columnwidth}{!}{\begin{tabular}{ll}
    \toprule
      Symbol  &  Meaning\\
     \midrule
      $D_{a,k}(\eps)$ & a function contributes to the main term of regret analysis defined in (\ref{eq: defDa}). \\
        $\eK_{a,n}(\eps)$ & an event where MLE of $\kappa$ is close to its true value at round $t$ after $n$ observations.\\
        $\eA_{a,n}(\eps)$& an event where MLE of $\alpha$ is close to its true value at round $t$ after $n$ observations.\\
        $\eE_{a,n}(\eps)$& intersection of $\eK_{a,n}(\eps)$ and $\eA_{a,n}(\eps)$. \\
        $\eM_\eps(t)$ & an event where sampled mean of the optimal arm is close to its true mean reward at round $t$. \\
        $p_n(x | \theta_{1,n})$ & probability of $\{\tmu_1(t) \leq \mu_1 - x\}$ after $n$ observation of arm $1$ given $\theta_{1,n}$.\\ 
        $G_k(x;n)$ & another expression of the CDF of the Erlang distribution in (\ref{eq: Gkdef}). \\
\end{tabular}}
\end{table}

\begin{table}[!ht]
\caption{Notations for (deterministic) constants}
\label{tab: 4}
\centering
\resizebox{\columnwidth}{!}{\begin{tabular}{ll}
    \toprule
      Symbol  &  Meaning\\
     \midrule
        $\veps_a$ & problem-dependent constants to satisfy $\hmu_a(n) \in [\mu_a - \delta_a, \mu_a + \delta_a]$ on $\eE_{a,n}(\eps)$ for any $\eps \leq \veps_a$. \\
        $\eps_{a,l}(\eps)$, $\eps_{a,l}(\eps)$ & constants to control a deviation of $\ha_a(\eps)$ under the event $\eA_{a,n}(\eps)$. \\
        $\rho_{\theta_1}(\eps), \bar{\rho}=\rho_{\theta_1}(\eps/2)$ & a difference from its true shape parameter $\alpha_1$ to satisfy $\mu(\kappa_a, \alpha+\rho_{\mu_1}(\eps)) \geq \mu_1 - \frac{\eps}{2}$. \\ 
        $C_1(\mu_1, \eps, n) = C_{1,n}$ & a constant smaller than $1$ in (\ref{eq: Gkdef}).\\
        $C_2(\mu_1, \eps, k)$ & an uniform bound of $p_n(\eps|\cdot)$ under $\eK_{1,n}^c(\eps) \cap \eA_{1,n}(\eps/2)$. \\
        $c_{\mu_1}(\eps)$ & a small constant with $\mathcal{O}(\eps^{-2})$. \\
\end{tabular}}
\end{table}

\newpage

\section{Proof of Lemma~\ref{lem: KLinf}}\label{sec: L1proof}
\lemKLinf*
\begin{proof}
Recall the definition
\begin{equation*}
    \KLinf(a) = \KLinf(\Par(\theta_a), \Par(\theta)) := \inf_{\theta\in \Theta_a} \log \frac{\alpha_a}{\alpha} +  \alpha \log \frac{\kappa_a}{\kappa} + \frac{\alpha}{\alpha_a} -1,
\end{equation*}
where $\theta = (\kappa, \alpha)$ and $\Theta_a$ defined in (\ref{eq: thetaset}).

Here, we consider the partition of $\Theta_a$, 
\begin{align*}
    \Theta_a^{(1)} &= \left\{ (\kappa, \alpha) \in (0, \kappa_a] \times (0,1] : \mu(\ka) > \mu_1 \right\} = (0, \kappa_a] \times (0,1] \\
    \Theta_a^{(2)} &= \left\{ (\kappa, \alpha) \in (0, \kappa_a] \times (1,\infty) : \mu(\ka)=\frac{\kappa\alpha}{\alpha-1} > \mu_1 \right\}, \numberthis{\label{eq: thetaset2}}
\end{align*}
where $\Theta_a^{(1)} \cup \Theta_a^{(2)} = \Theta_a$.
Therefore, it holds that
\begin{equation*}
    \KLinf(a) = \min\left(\inf_{\theta\in \Theta_a^{(1)}} \log \frac{\alpha_a}{\alpha} +  \alpha \log \frac{\kappa_a}{\kappa} + \frac{\alpha}{\alpha_a} -1, \inf_{\theta\in \Theta_a^{(2)}} \log \frac{\alpha_a}{\alpha} +  \alpha \log \frac{\kappa_a}{\kappa} + \frac{\alpha}{\alpha_a} -1 \right).
\end{equation*}
For $(\ka) \in \Theta_a^{(1)}$, $\mu(\ka) = \infty$ holds regardless of $\kappa$.
Therefore, we obtain
\begin{align*}
    \inf_{\theta \in \Theta_a^{(1)}} \log \frac{\alpha_a}{\alpha} +  \alpha \log \frac{\kappa_a}{\kappa} + \frac{\alpha}{\alpha_a} -1 &= \inf_{\alpha \in (0,1]} \log \frac{\alpha_a}{\alpha} + \frac{\alpha}{\alpha_a} -1 \\
    & = \log \alpha_a + \frac{1}{\alpha_a } -1,
\end{align*}
where the last equality holds since $\log x + \frac{1}{x} - 1$ is an increasing function for $x \geq 1$.

Let $\frac{\kappa_a}{\kappa} = c \geq 1$ to make KL divergence from $\Par(\theta_a)$ to $\Par(\ka)$ be well-defined.
From its definition of $\Theta_a^{(2)}$ in (\ref{eq: thetaset2}), any $\theta = (\ka) \in \Theta_a^{(2)}$ satisfies $\frac{\kappa\alpha}{\alpha-1} \geq \mu_1$, i.e., 
\begin{equation*}
    \frac{\kappa_a \alpha}{c(\alpha-1)} \geq \mu_1 \Leftrightarrow \alpha \leq \frac{c \mu_1}{c\mu_1 - \kappa_a} =:h_a(c,\bmu) = h_a(c).
\end{equation*}
Note that it holds that
\begin{equation*}
    h_a(1) = \frac{\mu_1}{\mu_1-\kappa_a} \leq  \frac{\mu_a}{\mu_a-\kappa_a} = \alpha_a
\end{equation*}
since $\frac{x}{x-y}$ is decreasing with respect to $x \geq y$.
Then, we can rewrite the infimum of KL divergence as
\begin{equation*}
    \KLinf(a) = \min \left(\log \alpha_a + \frac{1}{\alpha_a} -1, \inf_{c \geq 1} \inf_{\alpha \leq  h_a(c)}  g_a(\alpha, c) \right),
\end{equation*}
where $ g_a(\alpha, c) := \log \frac{\alpha_a}{\alpha} +  \alpha \log c + \frac{\alpha}{\alpha_a} -1$ satisfying
\begin{equation}\label{eq: derivativef}
    \frac{\partial  g_a(\alpha, c)}{\partial \alpha} = \frac{1}{\alpha_a} + \log c - \frac{1}{\alpha}.
\end{equation}
Then, the inner infimum can be obtained when $\alpha = \frac{\alpha_a}{1+\alpha_a \log c}$ holds if $\frac{\alpha_a}{1+\alpha_a \log c} < h_a(c)$, where $g_a\left(\frac{\alpha_a}{1+\alpha_a \log c}, c\right) = \log(1+\alpha_a \log c)$.

Let $c_{a}^* \geq 1$ be a deterministic constant such that 
\begin{equation}\label{eq: castar}
  h_a(c_{a}^*)  = \frac{c_{a}^* \mu_1}{c_{a}^* \mu_1 - \kappa_a} = \frac{\alpha_a}{1+\alpha_a \log c_{a}^*} \Leftrightarrow (\mu_1\alpha_a ) c_{a}^*\log c_{a}^* + (\mu_1 - \alpha_a \mu_1)c_{a}^* = -\alpha_a \kappa_a
\end{equation}
so that $h_a(c) \geq \frac{\alpha_a}{1+\alpha_a \log c}$ holds for any $c \geq c_{a}^*$.
Since the solution of $ax\log(x) + bx = -c$ is given as $\exp\left(W\left(- \frac{ce^{b/a}}{a}\right) - \frac{b}{a}\right)$ for principal branch of Lambert W function $W(\cdot)$, one can obtain $c_{a}^*$ by solving the equality in (\ref{eq: castar}), which is
\begin{equation}\label{eq: ca*}
    c_{a}^* = \exp\left(W\left(-\frac{\kappa_a}{\mu_1} e^{\frac{1}{\alpha_a}-1}\right) + 1 - \frac{1}{\alpha_a}\right). 
\end{equation}
Notice that $\frac{\kappa_a}{\mu_1} e^{\frac{1}{\alpha_a}-1} \leq \frac{\kappa_a}{\mu_a}e^{\frac{1}{\alpha_a}-1}  \leq \left( 1-\frac{1}{\alpha_a} \right)e^{-\left( 1- \frac{1}{\alpha_a}\right)} \leq e^{-1}$ holds so that $c_{a}^*$ is a real value.
Here, we consider the principal branch to ensure $c_{a}^* \geq 1$ since the solution on other branches, $W_{-1}(\cdot)$, is less than $1$, which is out of our interest.

Let $A_a = 1- \frac{1}{\alpha_a}$, which is positive as $\alpha_a > 1$ and $B_a = \frac{\kappa_a}{\mu_1} \leq \frac{\kappa_a}{\mu_a} = \frac{\alpha_a -1}{\alpha_a} = A_a$.
Then, we can rewrite $c_{a}^*$ as
\begin{equation*}
    c_{a}^* = e^{A_a} e^{W(-B_ae^{-A_a})} = e^{A_a} e^{-A_a} \frac{-B_a}{W(-B_a e^{-A_a})}. \hspace{3em} \because e^{W(x)} = \frac{x}{W(x)} 
\end{equation*}
Since the principal branch of Lambert W function, $W(x)$, is increasing for $x \geq - \frac{1}{e}$, we have
\begin{equation*}
  0> W(-B_ae^{-A_a}) \geq W(-B_ae^{-B_a}) = -B_a,
\end{equation*}
which implies that $c_{a}^* \geq 1$.
Therefore, the infimum of $g_a$ can be written as
\begin{align*}
   \inf_{c \geq 1} \inf_{\alpha \leq  h_a(c)}  g_a(\alpha, c) &= \min \left( \inf_{c \in [1, c_{a}^*]} g_a(h_a(c),c), \inf_{c \geq c_{a}^*} \log(1+\alpha_a \log c)  \right) \\
    &= \min \left( \inf_{c \in [1, c_{a}^*]} g_a(h_a(c),c), \log(1+\alpha_a \log c_{a}^*)  \right),
\end{align*}
where we follow the convention that the infimum over the empty set is defined as infinity.

By substituting $c_a^*$ in (\ref{eq: ca*}), we obtain
\begin{equation*}
    \log(1+\alpha_a \log c_{a}^*)  = \log \left(\alpha_a + W\left(-\frac{\kappa_a}{\mu_1} e^{\frac{1}{\alpha_a}-1}\right) \right).
\end{equation*}
Let us consider the following inequalities:
\begin{align*}
  \log \left(\alpha_a + W\left(-\frac{\kappa_a}{\mu_1} e^{\frac{1}{\alpha_a}-1}\right) \right) &\geq \log \left(\alpha_a + W\left(-\frac{\kappa_a}{\mu_a} e^{\frac{1}{\alpha_a}-1}\right) \right) \\
  &= \log \left(\alpha_a + W\left(\frac{1-\alpha_a}{\alpha_a} e^{\frac{1}{\alpha_a}-1}\right) \right) \\
  &= \log \left( \alpha_a + \frac{1}{\alpha_a} -1 \right), \numberthis{\label{eq: KLinfcstar}}
\end{align*}
where the first inequality holds since the principal branch of Lambert W function $W(x)$ is increasing and negative with respect to $x \in \left[ -1/e, 0 \right)$.

It remains to find the closed form of $\inf_{c \in [1, c_{a}^*]} g_a(h_a(c),c)$.
From the definition of $h_a(c) = \frac{c\mu_1}{c\mu_1 -\kappa_a}$, we have $h_a'(c) = -\frac{\mu_1 \kappa_a}{(c\mu_1 - \kappa_a)^2}$ and
\begin{align*}
    \frac{\partial  g_a(h_a(c), c)}{\partial c} &= \frac{\partial}{\partial c} \left( \log \frac{\alpha_a}{h_a(c)} + h_a(c) \log c + \frac{h_a(c)}{\alpha_a} - 1 \right) \hspace{2em} \because g_a(x, c) = \log \frac{\alpha_a}{x} + x \log c + \frac{x}{\alpha_a} - 1 \\
    &= - \frac{h_a'(c)}{h_a(c)} + h_a'(c) \log c + \frac{h_a(c)}{c} + \frac{1}{\alpha_a}h_a'(c) \\
    &= \frac{\kappa_a}{c(c\mu_1 - \kappa_a)} - \frac{\mu_1 \kappa_a}{(c\mu_1 -\kappa_a)^2} \log c + \frac{\mu_1}{c\mu_1 -\kappa_a} - \frac{\kappa_a \mu_1}{\alpha_a (c\mu_1 -\kappa_a)^2} \\
    &=  \frac{\kappa_a}{c(c\mu_1 - \kappa_a)} - \frac{\mu_1 \kappa_a}{(c\mu_1 -\kappa_a)^2} \log c + \mu_1 \frac{c\mu_1 - \kappa_a - \frac{\kappa_a}{\alpha_a}}{(c\mu_1 -\kappa_a)^2}. \numberthis{\label{eq: derivativeKLC}}
\end{align*}
Since the first term in (\ref{eq: derivativeKLC}) is positive for $c\geq 1$ and $\mu_1 \geq \mu_a > \kappa_a$, let us consider the last two terms for $c \in [1, c_{a}^*]$,
\begin{align*}
    - \frac{\mu_1 \kappa_a}{(c\mu_1 -\kappa_a)^2} \log c + \mu_1 \frac{c\mu_1 - \kappa_a - \frac{\kappa_a}{\alpha_a}}{(c\mu_1 -\kappa_a)^2} &= \frac{\mu_1}{(c\mu_1-\kappa_a)^2} \left( c\mu_1 - \kappa_a - \frac{\kappa_a}{\alpha_a} - \kappa_a \log c \right) \\
    &= \frac{\mu_1}{(c\mu_1-\kappa_a)^2} \left( \mu_1 - \kappa_a - \frac{\kappa_a}{\alpha_a} + (c-1)\mu_1 -\kappa_a \log c \right) \\
    &= \frac{\mu_1}{(c\mu_1-\kappa_a)^2} \left( \mu_1 - \kappa_a - \frac{\kappa_a}{\alpha_a} + \mu_1\left ( c -\frac{\kappa_a}{\mu_1} \log c -1 \right) \right).
\end{align*}
Here, 
\begin{equation*}
     \mu_1 - \kappa_a - \frac{\kappa_a}{\alpha_a} \geq \mu_a  - \kappa_a - \frac{\kappa_a}{\alpha_a} = \frac{\kappa_a \alpha_a}{\alpha_a-1} - \kappa_a - \frac{\kappa_a}{\alpha_a} = \frac{\kappa_a}{\alpha_a(\alpha_a-1)} > 0,
\end{equation*}
and $c -\frac{\kappa_a}{\mu_1} \log c -1$ is increasing with respect to $c$ so that $c -\frac{\kappa_a}{\mu_1} \log c -1 \geq 0$ for $c \geq 1$.
Therefore, $\frac{\partial }{\partial c}g_a(h_a(c), c)$ is positive for $c \geq 1$, i.e., $g_a(h_a(c), c)$ is an increasing function with respect to $c \geq 1$.

Thus, we have for $c \in [1, c_{a}^*]$,
\begin{equation*}
    \inf_{c \in [1, c_{a}^*]} g_a(h_a(c),c) = g_a\left( h_a(1), 1\right) = g_a\left( \frac{\mu_1}{\mu_1 - \kappa_a}, 1 \right) = \log\left(\alpha_a\frac{\mu_1-\kappa_a}{\mu_1}\right) + \frac{1}{\alpha_a}\frac{\mu_1}{\mu_1-\kappa_a} -1.
\end{equation*}
Denote $X_a = \alpha_a\frac{\mu_1-\kappa_a}{\mu_1} \in [1, \alpha_a)$ where $X_a=1$ happens only when $\mu_a = \mu_1$.
Then, we have for $\alpha_a >1$
\begin{equation*}
    \log (X_a) + \frac{1}{X_a} -1 \leq \log \alpha_a + \frac{1}{\alpha_a} -1 \leq \log \left( \alpha_a + \frac{1}{\alpha_a} -1 \right) \leq \log(1+\alpha_a \log c_{a}^*),
\end{equation*}
where the last inequality comes from the result in (\ref{eq: KLinfcstar}).
Therefore, we have
\begin{align*}
    \KLinf(a) &= \min \left(\log \alpha_a + \frac{1}{\alpha_a} - 1, \inf_{c \in [1, c_{a}^*]} g_a(h_a(c),c) , \log(1+\alpha_a \log c_{a}^*) \right) \\
    &= \log\left(\alpha_a\frac{\mu_1-\kappa_a}{\mu_1}\right) + \frac{1}{\alpha_a}\frac{\mu_1}{\mu_1-\kappa_a} -1,
\end{align*}
which concludes the proof.
\end{proof}

\section{Proofs of lemmas for Theorems~\ref{thm: RegOpt} and~\ref{thm: RegOptTruncation}}\label{sec: Opt}
In this section, we provide the proof of lemmas for the main results.

To avoid redundancy, we use a temporary notation $\alpha_{a,n}$ when the same result holds for both $\ha_a(n)$ and $\ba_a(n)$.
When $\alpha_{a,n}$ notation is used, one can replace it with either $\ha_a(n)$ or $\ba_a(n)$ depending on which policy we are considering.
For example, it holds that
\begin{align*}
    \alpha_{a,n} \leq 1 \Leftrightarrow  \begin{cases}
    \ha_a(n) \leq 1  &\text{under $\TS$ policy}, \\
    \ba_a(n) \leq 1  &\text{under $\TST$ policy}.
    \end{cases}
\end{align*}
Similarly, we use the notation $\theta_{a,n} := (\hk_a(n), \alpha_{a,n})$ when it can be replaced by both $\hat{\theta}_{a,n} = (\hk_a(n), \ha_a(n))$ and $\bar{\theta}_{a,n} = (\hk_a(n), \ba_a(n))$ for any arm $a \in[K]$ and $n \in \mathbb{N}$.
Based on $\theta_{a,n}$ notation, we provide an inequality on the posterior probability that the sampled mean is smaller than a given value with proofs in Appendix~\ref{sec: tmusmall}.
\begin{restatable}{lemma}{lemtmu}~\label{lem: tmusmall}
For any arm $a \in [K]$ and fixed $t \in \mathbb{N}$, let $N_a(t) =n$.
For any positive $\xi \leq \frac{y}{y - \kappa_a}$ and $k \in \mathbb{Z}$, it holds that
\begin{align*}
  \I[\hk_a(n) \leq y]\mathbb{P}[\tmu_a(t) \leq y| \theta_{a,n}] \leq \int_{\xi}^\infty f_{n-k, \frac{n}{\alpha_{a,n}}}^{\mathrm{Er}}(x) \dx x + \bigg( \frac{y}{\mu((\kappa_a, \xi))} \bigg)^{n} \int_{1}^\xi f_{n-k, \frac{n}{\alpha_{a,n}}}^{\mathrm{Er}}(x) \dx x,
\end{align*}
where $f_{s, \beta}^{\mathrm{Er}}(\cdot)$ denotes the probability density function of the Erlang distribution with shape $s\in \mathbb{N}$ and rate $\beta>0$.
\end{restatable}

Based on $\theta_{1,n}$ notation, we denote the probability of sample from the posterior distribution after $n$ times playing is smaller than $\mu_1 - x$ by
\begin{equation}\label{eq: pndef}
    p_n(x | \theta_{1,n}) := \mathbb{P}[\tmu_1 \leq \mu_1 - x | \hk_1(n), \alpha_{1,n}].
\end{equation}
Let $K(\eps) = (\kappa_1 + \eps, \mu_1 - \eps)$ be an open interval on $\mathbb{R}$.
The Lemma~\ref{lem: pn_cases} below shows the upper bound of $p_n$ conditioned on ${\theta}_{1,n}$.
\begin{restatable}{lemma}{lempn}\label{lem: pn_cases}
Given $\eps >0$, define a positive problem-dependent constant $\rho = \rho_{\theta_1}(\eps):= \frac{\kappa_1\eps}{2(\mu_1-\eps/2 - \kappa_1)(\mu_1-\kappa_1)}$.
If $n \geq \bn=\max(2, k+1)$, then for $k \in \mathbb{Z}_{\geq 0}$
\begin{align*}
  p_n(\eps|\theta_{1,n})
  \leq
    \begin{cases}
    e^{-n}, & \mathrm{if}~\hk_1(n) \geq \mu_1 - \eps, \\
    h(\mu_1, \eps,n), & \mathrm{if}~ \hk_1(n) \in K(\eps), \alpha_{1,n} \leq \alpha_1 + \rho,\\
     C_1(\mu_1,\eps, n)G_k(1/\alpha_{1,n};n)  + 1-G_k(1/\alpha_{1,n};n) & \mathrm{if}~ \hk_1(n) \in K(\eps), \alpha_{1,n} \geq \alpha_1 + \rho,
    \end{cases}
\end{align*}
where
\begin{align*}
     h(\mu_1, \eps,n) &:= e^{-n \frac{3\eps}{4\mu_1}} \left( 1- \frac{1}{2}e^{-nc_{\mu_1}(\eps)} \right) + \frac{1}{2}e^{-nc_{\mu_1}(\eps)} \\
     C_1(\mu_1,\eps, n) &:= \left( \frac{\mu_1 - \eps}{\mu_1 - \eps/2} \right)^{n} \leq e^{-n \frac{\eps}{2\mu_1 - \eps}} < 1.\\
     G_k(x; n) &:= F^{\mathrm{Er}}_{n-k, nx}(\alpha_1 + \rho) \numberthis{\label{eq: Gkdef}}
\end{align*}
for $F^{\mathrm{Er}}$ defined in (\ref{eq: def_inc_gamma}).
Here, $c_{\mu_1}(\eps) = \zeta - \log(1+\zeta) = \mathcal{O}(\eps^{-2})$ and $\zeta = \frac{\eps}{4\mu_1 - 2\eps} \in (0,1)$ are deterministic constants of $\mu_1$ and $\eps$.
\end{restatable}
Notice that $\mu((\kappa_1, \alpha_1 + \rho)) = \mu_1 - \eps/2$ holds and
there exists a problem-dependent constant $C_2(\mu_1, \eps, k)< 1$ such that for any $n \geq \bn = \max(2, k+1)$ and $\eps >0$
\begin{equation}\label{eq: upperbound}
     h(\mu_1, \eps,n)  \leq 1-C_2(\mu_1, \eps, k).
\end{equation}

\subsection{Proof of Lemma~\ref{lem: difficult}}\label{app: lem5}
Let us start by stating a well-known fact that is utilized in the proof.
\begin{fact}\label{eq: fact_IG}
When $X \sim \Er(n, \beta)$ with rate parameter $\beta$, then $\frac{1}{X}$ follows the inverse gamma distribution with shape $n \in \mathbb{N}$ and scale $\beta\in \mathbb{R}_{+}$, i.e., $\frac{1}{X} \sim \InvG(n, \beta)$.
\end{fact}
\lemdifficult*
\begin{proof}
Let us consider the following decomposition that holds under both $\TS$ and $\TST$:
\begin{align*}
    \sum_{t=K\bn +1}^T \I[j(t)\ne 1,\eK_{1,N_1(t)}^c(\eps), \eM_\eps^c(t)] \hspace{-9em} &\\
    &= \sum_{n=\bn}^T \sum_{t=K\bn + 1}^T \I\bigg[j(t)\ne 1,\eK_{1,N_1(t)}^c(\eps),\eM_\eps^c(t), N_1(t) =n\bigg] \\
    &= \sum_{n=\bn}^T \sum_{m=1}^T \I\Bigg[m \leq \sum_{t=K\bn + 1}^T \I\bigg[j(t)\ne 1, \eM_\eps^c(t), \eK_{1,N_1(t)}^c(\eps), N_1(t) =n\bigg]\Bigg].
\end{align*}
Notice that
\begin{equation*}
    m \leq \sum_{t=K\bn + 1}^T \I[j(t)\ne 1, \eK_{1,N_1(t)}^c(\eps), \eM_\eps^c(t), N_1(t) =n]
\end{equation*}
implies that $\tmu_1(t) \leq \mu_1 -\eps$ occurred $m$ times in a row on $\{t: \eK_{1,N_1(t)}^c(\eps), \eM_\eps^c(t), N_1(t) =n \}$.
Thus, we have
\begin{align*}
    \mathbb{E}\Bigg[  \sum_{t=K\bn +1}^T \I[j(t)\ne 1, \eK_{1,N_1(t)}^c(\eps), \eM_\eps^c(t)] \Bigg] &\leq \mathbb{E}\left[  \sum_{n=\bn}^T \sum_{m=1}^T \I[\eK_{1,n}^c(\eps)] p_n(\eps|{\theta}_{1,n})^m \right] \\
    &\leq \sum_{n=\bn}^T \mathbb{E}\left[ \I\left[\eK_{1,n}^c(\eps)\right] \frac{p_n(\eps|{\theta}_{1,n})}{1-p_n(\eps|{\theta}_{1,n})} \right] \numberthis{\label{eq: pn_decom}}
\end{align*}
for $p_n(\cdot|\cdot)$ defined in (\ref{eq: pndef}).
From now on, we fix $n \geq \bn$ and drop the argument $n$ of $\hk_1(n)$, $\ha_1(n)$ and $\ba_1(n)$ for simplicity.

\paragraph{Under $\TS$ with $k \in \mathbb{Z}_{\geq 2}$:}
By applying Lemma~\ref{lem: pn_cases} and (\ref{eq: upperbound}) under $\TS$ with $k \in \mathbb{Z}_{\geq 0}$, we can decompose the expectation in (\ref{eq: pn_decom}) as
\begin{align*}
    \mathbb{E}\left[ \I\left[\eK_{1,n}^c(\eps)\right] \frac{p_n(\eps|\hat{\theta}_{1,n})}{1-p_n(\eps|\hat{\theta}_{1,n})} \right] \leq  
    \mathbb{P}[\hk_1 \geq \mu_1 - \eps] \frac{e^{-n}}{1-e^{-n}} &+ \mathbb{P}[\hk_1 \in K(\eps), \ha_1 < \alpha_1 + \rho] \frac{ h(\mu_1, \eps,n) }{C_2(\mu_1, \eps, k)}   \\
    &\hspace{1em}+ \underbrace{\mathbb{E}_{\hat{\theta}_{1,n}}\left[\frac{\I[\hk_1\in K(\eps), \ha_1 > \alpha_1 + \rho]}{G_k(1/\ha_1 ;n)(1-C_{1,n})}\right]}_{(\divideontimes)}, \numberthis{\label{eq: pn_sum_STS}}
\end{align*}
where we denoted $C_{1,n} = C_1(\mu_1, \eps, n)$.
For simplicity, let us define $z := \frac{1}{\ha_1}$ where $z \sim \Er(n-1, n\alpha_1)$ holds from Fact~(\ref{eq: fact_IG}) since $\ha_1 \sim \InvG(n-1, n\alpha_1)$ in (\ref{eq: MLEdist}).
From the independence of $\hk$ and $\ha$ and distributions of $z$ and $\ha$ in (\ref{eq: pdf_erlang}) and (\ref{eq: MLEdist}), respectively, we have
\begin{align*}
    (\divideontimes) &= \int_0^{\frac{1}{\alpha_1 + \rho}} z^{n-2} e^{-n\alpha_1 z}\frac{(n\alpha_1)^{n-1}}{\Gamma(n-1) } \int_{\hk_1 \in K(\eps)} \frac{f^{\mathrm{Pa}}_{\kappa_1, n\alpha_1}(\hk_1)}{G_k(z;n)(1-C_{1,n})} \dx \hk_1 \dx z \\
    &= \mathbb{P}[\hk_1 \in K(\eps)] \int_0^{\frac{1}{\alpha_1 + \rho}} \frac{z^{n-2} e^{-n\alpha_1 z} }{G_k(z;n)(1-C_{1,n})} \frac{(n\alpha_1)^{n-1}}{\Gamma(n-1)}\dx z.
\end{align*}
By substituting the CDF in (\ref{eq: def_inc_gamma}), we obtain
\begin{align*}
    G_k(z;n) &= F^{\mathrm{Er}}_{n-k, nz}(\alpha_1 + \rho) \\
    &=\frac{1}{\Gamma(n-k)}\int_0^{n(\alpha_1 + \rho)z} e^{-t}t^{n-k-1} \dx t \\
    &\geq \frac{e^{-n(\alpha_1 + \rho)z}}{\Gamma(n-k)}  \int_0^{n(\alpha_1 + \rho)z} t^{n-k-1} \dx t 
  \\ &= \frac{e^{-n(\alpha_1 + \rho)z}}{\Gamma(n-k+1)} (n(\alpha_1 + \rho)z)^{n-k}. \numberthis{\label{eq: Gkbnd}}
\end{align*}
Therefore,
\begin{align*}
    \frac{(\divideontimes)}{\mathbb{P}[\hk \in K(\eps)]} &\leq 
    \int_0^{\frac{1}{\alpha_1 + \rho}} \frac{\Gamma(n-k+1)}{(n(\alpha_1 + \rho)z)^{n-k}(1-C_{1,n})} e^{n(\alpha_1 + \rho)z} \frac{(n\alpha_1)^{n-1}}{\Gamma(n-1)} z^{n-2} e^{-n\alpha_1 z} \dx z \\
    &=  \frac{\Gamma(n-k+1)}{\Gamma(n-1)(1-C_{1,n})} (\alpha_1 + \rho)^{k-1} \left( \frac{\alpha_1}{\alpha_1 + \rho} \right)^{n-1} n^{k-1} \int_0^{\frac{1}{\alpha_1 + \rho}} z^{k-2} e^{n\rho z} \dx z \\
    &\leq \frac{\Gamma(n-k+1)}{\Gamma(n-1)(1-C_{1,n})} (\alpha_1 + \rho)^{k-1} e^{-\frac{\rho}{\alpha_1 + \rho}(n-1)}  \frac{n^{k-1}}{(n\rho)^{k-2}} \int_0^{\frac{1}{\alpha_1 + \rho}} (n\rho z)^{k-2} e^{n\rho z} \dx z. \numberthis{\label{eq: A1bnd}}
\end{align*}
By letting $n\rho z = y$, we can bound the integral in (\ref{eq: A1bnd}) as
\begin{align*}
    \frac{n^{k-1}}{(n\rho)^{k-2}} \int_0^{\frac{1}{\alpha_1 + \rho}} (n\rho z)^{k-2} e^{n\rho z} \dx z &=  \rho^{-(k-1)} \int_0^{\frac{n\rho}{\alpha_1 + \rho}} y^{k-2} e^{y} \dx y \\
    &\leq  \rho^{-(k-1)} e^{n \frac{\rho}{\alpha_1 + \rho}} \int_0^{\frac{n\rho}{\alpha_1 + \rho}} y^{k-2}\dx y \numberthis{\label{eq: integral}} \\
    &=  \frac{e^{n \frac{\rho}{\alpha_1 + \rho}}}{k-1} \left( \frac{n}{\alpha_1 + \rho} \right)^{k-1}, \quad \text{ if } k \in \mathbb{Z}_{\geq 2}\numberthis{\label{eq: intgrslt}}
\end{align*}
where (\ref{eq: intgrslt}) holds only for $k \in \mathbb{Z}_{\geq 2}$ since the integral in (\ref{eq: integral}) diverges for $k \in \mathbb{Z}_{\leq 1}$.

By applying (\ref{eq: intgrslt}) to (\ref{eq: A1bnd}), we obtain for $k \in \mathbb{Z}_{\geq 2}$ 
\begin{align*}
    (\divideontimes) &\leq 
    \mathbb{P}[\hk \in K(\eps)]  \frac{e^{\frac{\rho}{\alpha_1 + \rho}}}{1-C_{1,n}} \frac{n^{k-1}}{k-1}\frac{\Gamma(n-k+1)}{\Gamma(n-1)} \hspace{3em}\\
    &\leq \frac{e^{1-\frac{\eps \alpha_1}{\kappa + \eps}n}}{1-C_{1,n}} \frac{\Gamma(n-k+1)}{\Gamma(n-1)} \frac{n^{k-1}}{k-1} = \mathcal{O}(ne^{-n\eps}),\hspace{-1em}\numberthis{\label{eq: a1_1}}
\end{align*}
where (\ref{eq: a1_1}) can be obtained by Lemma~\ref{lem: eKbound} and $\frac{\rho}{\alpha_1 + \rho} < 1$.
By combining (\ref{eq: a1_1}) with (\ref{eq: pn_sum_STS}) and (\ref{eq: pn_decom}), we obtain for $k\in \mathbb{Z}_{\geq 2}$
\begin{align*} 
    \mathbb{E}\left[  \sum_{t=K\bn +1}^T \I[j(t)\ne 1, \eK_{1,N_1(t)}^c(\eps), \eM_\eps^c(t)] \right] 
    &\leq \sum_{n=\bn}^T \bigg(\frac{e^{-n}}{1-e^{-n}}+ \frac{e^{-n \frac{3\eps}{4\mu_1}} + \frac{1}{2}e^{-nc_{\mu_1}(\eps)}}{C_2(\mu_1, \eps, k)} 
    + (\divideontimes) \bigg)\\
    & \leq \sum_{n=\bn}^T  \mathcal{O}(e^{-n}) + \mathcal{O}(e^{-n\eps}) + \mathcal{O}(e^{-n\eps^2}) + \mathcal{O}(ne^{- n\eps}) \\
  & = \mathcal{O}(1) + \mathcal{O}(\eps^{-1}) + \mathcal{O}(\eps^{-2}) + \mathcal{O}(\eps^{-2}),
\end{align*}
which concludes the proof under $\TS$ with $k \in \mathbb{Z}_{\geq 2}$.

\paragraph{Under $\TST$ with $k \in \mathbb{Z}_{\geq 0}$:}
Under $\TST$, we have the following inequality instead of (\ref{eq: pn_sum_STS}):
\begin{align*}
    \mathbb{E}\left[ \I\left[\eK_{n}^c(\eps)\right] \frac{p_n(\eps|\bar{\theta}_{1,n})}{1-p_n(\eps|\bar{\theta}_{1,n})} \right]\leq  
    \mathbb{P}[\hk_1 \geq \mu_1 - \eps] \frac{e^{-n}}{1-e^{-n}} &+ \mathbb{P}[\hk_1 \in K(\eps), \ba_1 < \alpha_1 + \rho] \frac{ h(\mu_1, \eps,n) }{C_2(\mu_1, \eps, k)}   \\
    &\hspace{1em} + \underbrace{\mathbb{E}_{\bar{\theta}_{1,n}} \left[\frac{\I[\hk_1 \in K(\eps), \ba_1 \in (\alpha_1 + \rho, n]]}{G_k(1/\ba_1 ;n)(1-C_{1,n})}\right]}_{(\star)}. \numberthis{\label{eq: pn_sum_STSR}}
\end{align*}
From $\I[\ba_1(n) < n]= \I[\ba_1(n) = \ha_1(n)]$, it holds that
\begin{align*}
    (\star) &= \mathbb{E}_{\hat{\theta}_{1,n}}\left[\frac{\I[\hk_1 \in K(\eps), \ha_1 \in (\alpha_1 + \rho, n)]}{G_k(1/\ha_1 ;n)(1-C_{1,n})} \right] + \mathbb{E}_{\bar{\theta}_{1,n}}\left[\frac{\I[\hk_1 \in K(\eps), \ba_1 =n ]}{G_k(1/\ba_1 ;n)(1-C_{1,n})} \right].
\end{align*}
Since $\I[\ba_1(n) = n]=\I[\ha_1(n) \geq n]$ holds and $\hk$ and $\ha$ are independent, we have for $z= \frac{1}{\ha_1} \sim \Er(n-1, n\alpha_1)$
\begin{align*}
    \frac{(\star)}{\mathbb{P}[\hk_1 \in K(\eps)]} &\leq \underbrace{\int_{\frac{1}{n}}^{\frac{1}{\alpha_1 + \rho}} \frac{z^{n-2} e^{-n\alpha_1 z} }{G_k(z;n)(1-C_{1,n})} \frac{(n\alpha_1)^{n-1}}{\Gamma(n-1)}\dx z}_{(\dagger)} + \underbrace{\frac{1}{G_k(1/n;n)(1-C_{1,n})} \mathbb{P}\left[ \frac{1}{\ha_1} \leq \frac{1}{n}\right]}_{(\diamond)}, \numberthis{\label{eq: star}}
\end{align*}
where the same techniques on $(\divideontimes)$ can be applied to derive an upper bound of $(\dagger)$.
By following the same steps as (\ref{eq: A1bnd}) and (\ref{eq: integral}), we obtain
\begin{equation*}
    \int_{\rho}^{\frac{n\rho}{\alpha + \rho}} y^{k-2}\dx y \leq \begin{cases} 
    \left( \frac{n\rho}{\alpha+\rho} \right)^{k-1}, & \mathrm{if}~k \in \mathbb{Z}_{\geq 2}, \\
    \log\left( \frac{n}{\alpha + \rho} \right), & \mathrm{if}~k = 1, \\
    \rho^{k-1}/(1-k), &\mathrm{if}~ k \in \mathbb{Z}_{k \leq 0},
    \end{cases}
\end{equation*}
as a counterpart of the integral in (\ref{eq: integral}).
By following the same steps as (\ref{eq: intgrslt}) and (\ref{eq: a1_1}), we have
\begin{equation}\label{eq: integral_restricted}
    (\dagger) \leq \begin{cases} 
    \frac{\Gamma(n-k+1)}{\Gamma(n-1)} \frac{e n^{k-1}}{k-1}, &\mathrm{if}~k \in \mathbb{Z}_{\geq 2}, \\
    (n-1)\log\left( \frac{n}{\alpha + \rho} \right) , &\mathrm{if}~k = 1, \\
    \frac{\Gamma(n-k+1)}{\Gamma(n-1)} \frac{e}{(1-k)(\alpha+\rho)^{1-k}} , &\mathrm{if}~k \in \mathbb{Z}_{k \leq 0}. \\
    \end{cases}
\end{equation}
Note that the probability term in $(\diamond)$ is the same as the CDF of the $\Er(n-1, n\alpha_1)$ with rate $n\alpha_1$ evaluated at $\frac{1}{n}$ since $\ha_1 \sim \InvG(n-1, n\alpha_1)$ from (\ref{eq: MLEdist}).
Thus, we have
\begin{align*}
    (\diamond) &= \frac{1}{1-C_{1,n}}\frac{1}{G_k(1/n;n)} \frac{\gamma(n-1,\alpha_1)}{\Gamma(n-1)}\\
    &\leq \frac{e^{\alpha_1 + \rho}}{1-C_{1,n}}\frac{\Gamma(n-k+1)}{(\alpha_1+\rho)^{n-k}} \frac{\gamma(n-1,\alpha_1)}{\Gamma(n-1)} \hspace{1em} \text{by (\ref{eq: Gkbnd})}\\
    &\leq \frac{e^{\alpha_1+\rho}}{1-C_{1,n}}\frac{\Gamma(n-k+1)}{\Gamma(n-1)} \frac{\alpha_1^{n-1}}{(\alpha_1 + \rho)^{n-k}} \numberthis{\label{eq: incgam}} \\
    &\leq  \frac{e^{\alpha_1+\rho}}{1-C_{1,n}} \frac{\Gamma(n-k+1)}{\Gamma(n-1)}\frac{1}{(\alpha_1+\rho)^{1-k}} \\
    &= \mathcal{O}(n^{2-k}), \numberthis{\label{eq: diamon_result}}
\end{align*}
where (\ref{eq: incgam}) holds from $\gamma(s, x) \leq x^{s}$ for any $s \geq 1$ and $x >0$.
Therefore, by combining (\ref{eq: integral_restricted}) and (\ref{eq: diamon_result}) with (\ref{eq: star}) and $\mathbb{P}[\hk \in K(\eps)] = \mathcal{O}(e^{-n\eps})$, we have
\begin{equation*}
    (\star) \leq \begin{cases}
        \mathcal{O}(ne^{-n\eps}), & \mathrm{if}~k \in \mathbb{Z}_{\geq 2}\\
        \mathcal{O}(n\log(n)e^{-n\eps}), & \mathrm{if}~k=1, \\
        \mathcal{O}(n^{2-k}e^{-n\eps}), & \mathrm{if}~k \in \mathbb{Z}_{\leq 0}. \numberthis{\label{eq: a1_2}}
    \end{cases}
\end{equation*}
By combining (\ref{eq: a1_2}) with (\ref{eq: pn_sum_STSR}) and (\ref{eq: pn_decom}), we obtain for $k\in \mathbb{Z}_{\geq 0}$
\begin{align*} 
    \mathbb{E}\bigg[  \sum_{t=K\bn +1}^T \I[j(t)\ne 1, \eK_{1,N_1(t)}^c(\eps), \eM_\eps^c(t)] \bigg] & \leq \sum_{n=\bn}^T \bigg(\frac{e^{-n}}{1-e^{-n}}+ \frac{e^{-n \frac{3\eps}{4\mu_1}} + \frac{1}{2}e^{-nc_{\mu_1}(\eps)}}{C_2(\mu, \eps, k)} 
    + (\star) \bigg)\\
    &\leq \sum_{n=\bn}^T \bigg( \mathcal{O}(e^{-n}) + \mathcal{O}(e^{-n\eps}) + \mathcal{O}\left(e^{-n\eps^2}\right) \\
    &\hspace{12em}+ \mathcal{O}\left(\psi(n,k) e^{-n\eps}\right) \bigg) \\
  &= \mathcal{O}(1) + \mathcal{O}(\eps^{-1}) + \mathcal{O}(\eps^{-2}) + \mathcal{O}\left(\eps^{-\max(2, 3-k)}\right),
\end{align*}
where 
\begin{equation*}
    \psi(n,k) = n \I[k \geq 2] + n\log(n)\I[k=1] + n^{2-k} \I[ k \leq 0].
\end{equation*}
Note that the analysis on term $(\star)$ also holds for $\TST$ with $k \in \mathbb{Z}_{<0}$.
However, differently from the case of $k \in \{ 0, 1\}$, priors with $k \in \mathbb{Z}_{<0}$ have additional problems in Lemma~\ref{lem: pn_cases} under the event $\{\hk_1 \in K(\eps), \ba_{1}(n) \leq \alpha_1 + \rho\}$, where the upper bound becomes a constant $\frac{1}{2}$.
\end{proof}

\subsection{Proof of Lemma~\ref{lem: mainreg}}\label{app: lem6}
Firstly, we state a well-known fact that is utilized in the proof. 
\begin{fact}~\label{fact1}
When $X \sim\Er(n, \beta)$ with rate parameter $\beta$, then $2\beta X$ follows the chi-squared distribution with $2n$ degree of freedom, i.e., $2\beta X \sim \chi_{2n}^2$.
\end{fact}

\lemmainreg*

\begin{proof}
From the sampling rule, it holds that
\begin{align*}
  \sum_{t=\bn K+1}^T \mathbb{P}\left[j(t)=a, \tmu_1(t) \geq \mu_1 - \eps, \eE_{a,N_a(t)}(\eps)\right]
    &\leq  \sum_{t=\bn K+1}^T \mathbb{P}\left[j(t)=a, \tmu_a(t) \geq \mu_1 -\eps, \eE_{a,N_a(t)}(\eps)\right].
\end{align*}
Fix a time index $t$ and denote $\mathbb{P}_t[\cdot] = \mathbb{P}[\cdot \mid \Ft]$ and $N_a(t) =n$.
To simplify notations, we drop the argument $t$ of $\tk_a(t), \ta_a(t)$ and $\tmu_a(t)$ and the argument $n$ of $\hk_a(n), \ha_a(n), \ba_a(n)$.

Since $\tk_a \in (0, \hk_a]$ holds from its posterior distribution, if $\ta_a \geq \frac{ \mu_1 - \eps}{ \mu_1 - \eps-\hk_a}$ holds, then $\tmu_a = \frac{\tk_a \ta_a}{\ta_a -1} \leq  \mu_1 - \eps$ holds for any $\tk_a$.
Recall that $f_{n-k, \frac{n}{\alpha_{a,n}}}^{\text{Er}}(\cdot)$ denotes a density function of $\Er\left(n-k , \frac{n}{\alpha_{a,n}}\right)$ with rate parameter $\frac{n}{\alpha_{a,n}}$, which is the marginalized posterior distribution of $\ta$ under $\TS$ and $\TST$.
From the CDF of $\tk$ in (\ref{eq: cdfkappa}), if $\hk_a < \mu_1 - \eps$, then
\begin{align*}
    \mathbb{P}_t\left[\tmu_a \geq \mu_1 - \eps\right] &= \mathbb{P}_t\left[\ta_a \leq 1\right] +\mathbb{P}_t\left[\tk_a \geq \frac{\ta_a-1}{\ta_a} (\mu_1 - \eps) \cap \ta_a \in \left(1, \frac{\mu_1 - \eps}{\mu_1 - \eps - \hk_a} \right) \right] \\
     &=\int_{0}^1 f_{n-k, \frac{n}{\alpha_{a,n}}}^{\text{Er}}(x) \dx x + \int_{1}^{\frac{\mu_1 - \eps}{\mu_1 - \eps - \hk}}f_{n-k, \frac{n}{\alpha_{a,n}}}^{\text{Er}}(x) \mathbb{P}_t\left[\tk_a \geq \frac{x-1}{x} (\mu_1 - \eps) \right] \dx x\\
    &=\int_{0}^1 f_{n-k, \frac{n}{\alpha_{a,n}}}^{\text{Er}}(x) \dx x + \int_{1}^{\frac{\mu_1 - \eps}{\mu_1 - \eps - \hk}}f_{n-k, \frac{n}{\alpha_{a,n}}}^{\text{Er}}(x) \left( 1 - \left( \frac{x-1}{\hk x} (\mu_1 - \eps) \right)^{nx} \right) \dx x\\
    &= \int_{0}^{\frac{\mu_1 - \eps}{\mu_1 - \eps - \hk}} f_{n-k, \frac{n}{\alpha_{a,n}}}^{\text{Er}}(x) \dx x - \int_{1}^{\frac{\mu_1 - \eps}{\mu_1 - \eps - \hk}}f_{n-k, \frac{n}{\alpha_{a,n}}}^{\text{Er}}(x) \left( \frac{x-1}{\hk x} (\mu_1 - \eps) \right)^{nx}\dx x.
\end{align*}
Since $\frac{x}{x-y}$ is increasing with respect to $y < x$ and $\hk \leq \kappa+ \eps$ holds on $\eE$, we have for $\eE$
\begin{equation*}
    \frac{\mu_1 - \eps}{\mu_1 - \eps - \hk} \leq  \frac{\mu_1 - \eps}{\mu_1 - \eps - (\kappa + \eps)}.
\end{equation*}
Let 
\begin{equation} \label{eq: defeta}
    \eta_a(\eps) = \frac{\kappa_a (\Delta_a - \eps) - \eps \mu_a }{(\mu_a-\kappa_a)(\mu_1 - \kappa_a - 2\eps ) } > 0
\end{equation}
be a deterministic constant that depends only on the model and $\eps$ and satisfies
\begin{align*}
    \alpha_a - \eta_a(\eps) &= \frac{\mu_a}{\mu_a -\kappa_a} - \frac{\kappa_a (\Delta_a - \eps) - \eps \mu_a }{(\mu_a-\kappa_a)(\mu_1 - \kappa_a - 2\eps ) }  \\
    &= \frac{\mu_a \mu_1 - \kappa_a \mu_a -2\eps \mu_a - \kappa_a (\mu_1 -\mu_a - \eps) + \eps \mu_a }{(\mu_a-\kappa_a)(\mu_1 - \kappa_a - 2\eps ) } \\
    &= \frac{\mu_1 (\mu_a - \kappa_a) - \eps(\mu_a - \kappa_a)}{(\mu_a-\kappa_a)(\mu_1 - \kappa_a - 2\eps )}\\
    &= \frac{\mu_1 - \eps}{\mu_1 - \kappa_a - 2\eps}.
\end{align*}
Since $\eta_a(\eps)>0$, it holds that for any $\eps \in (0, \veps_a)$
\begin{equation*}
  \alpha_a - \eta_{a}(\eps) = \frac{\mu_1 - \eps}{\mu_1 - \kappa_a - 2\eps} \leq \frac{\mu_a}{\mu_a - \kappa_a} = \alpha_a.
\end{equation*}
Note that $\frac{\mu}{\mu - \kappa}=\alpha$ holds and the change of $\mu$ to $\mu'$ with fixed $\kappa$ that is $\frac{\mu'}{\mu' - \kappa}$, implies how the value of the shape parameter $\alpha'$ should be to satisfy $\mu((\kappa, \alpha')) = \mu'$.
For example, $\theta= (\kappa_a + \veps_a, \alpha_a)$ satisfies $\mu(\theta) \leq \mu_a + \frac{\delta_a}{2}$.
Thus, if $\mu((\kappa_a+\veps_a, \alpha)) = \mu_1 - \eps > \mu_a + \frac{\delta_a}{2}$, then $\alpha$ should be smaller than $\alpha_a$.
Hence, we have
\begin{align*}
    \I[\eE_{a,n}(\eps)]\mathbb{P}_t\bigg[\tmu_a &\geq \mu_1 - \eps\bigg] \\
    &\leq \I[\eE_{a,n}(\eps)] \bigg( \int_{0}^{\frac{\mu_1 - \eps}{\mu_1 - \eps - \hk}} f_{n-k, \frac{n}{\alpha_{a,n}}}^{\text{Er}}(x) \dx x  \\
    & \hspace{9em} - \int_{1}^{\frac{\mu_1 - \eps}{\mu_1 - \eps - \hk}}f_{n-k, \frac{n}{\alpha_{a,n}}}^{\text{Er}}(x) \left( \frac{x-1}{\hk x} (\mu_1 - \eps) \right)^{nx}\dx x \bigg)\\
    &\leq \I[\eE_{a,n}(\eps)]\int_{0}^{\frac{\mu_1 - \eps}{\mu_1 - \eps - \hk}} f_{n-k, \frac{n}{\alpha_{a,n}}}^{\text{Er}}(x) \dx x \numberthis{\label{eq: tmusmall}}\\
    &\leq \I[\eE_{a,n}(\eps)]  \int_{0}^{\alpha_a -\eta_a(\eps)} f_{n-k, \frac{n}{\alpha_{a,n}}}^{\text{Er}}(x) \dx x = \I[\eE_{a,n}(\eps)]  \mathbb{P}_t[ \ta_a(t) \leq \alpha_a - \eta_a(\eps)]. \numberthis{\label{eq: overestimationtmu1}}
\end{align*}
Therefore, by taking expectation and using Fact~\ref{fact1}, we have
\begin{align*}
    \mathbb{P}\left[\tmu_a(t) \geq \mu_1 - \eps, \eE_{a,n}(\eps) \right] &\leq \mathbb{P}[ \ta_a \leq \alpha_a - \eta_a(\eps), \eE_{a,n}(\eps)], \\
    &= \mathbb{P}\left[ Z \leq \frac{2n}{\alpha_{a,n}}\left( \alpha - \eta_a(\eps) \right), \eE_{a,n}(\eps)\right] \numberthis{\label{eq: l7bungi}}
\end{align*}
where $Z$ is a random variable following the chi-squared distribution with $2(n-k)$ degrees of freedom, i.e., $Z \sim \chi_{2n-2k}^2$.

\subsubsection{Under $\TS$}~\label{sec: L7STS}
Here, we first consider the case of $\TS$ where we replace $\alpha_{a,n}$ with $\ha_a(n)$.

Since $\ha_a \in [\alpha_a - \eps_{a,l}, \alpha_a + \eps_{a,u}] $ holds on $\eE_{a,n}(\eps)$, we have
\begin{equation}\label{eq: haonE}
    \frac{1}{\alpha_a} - \eps \left( 1 + \frac{1}{\kappa_a} \right)=\frac{1}{\alpha_a + \eps_{a,u}}\leq \frac{1}{\ha_a} \leq \frac{1}{\alpha_a - \eps_{a,l}} =  \frac{1}{\alpha_a} + \eps
\end{equation}
by the definition of $\eps_{a,l}(\eps)$ and $\eps_{a,u}(\eps)$ in (\ref{eq: epslu}).

By replacing $\alpha_{a,n}$ with $\ha_a(n)$ in (\ref{eq: l7bungi}) and applying (\ref{eq: haonE}), we have
\begin{align*}
     \mathbb{P}\left[\tmu_a(t) \geq \mu_1 - \eps, \eE_{a,n}(\eps) \right] &\leq \mathbb{P}\left[ Z \leq \frac{2n}{\ha_a}\left( \alpha_a - \eta_a(\eps) \right), \eE_{a,n}(\eps)\right] \\
     &\leq \mathbb{P}\left[ Z \leq 2n\left( \frac{1}{\alpha_a} +\eps \right)\left( \alpha_a - \eta_a(\eps) \right)\right] \\
     &= \mathbb{P}\left[ Z \leq 2(n-k) \frac{n}{n-k}\left( \frac{1}{\alpha_a} +\eps \right)\left( \alpha_a - \eta_a(\eps) \right)\right]. \numberthis{\label{eq: tildealpha}}
\end{align*}

\paragraph{Priors with $k \in\mathbb{Z}_{\leq 0}$.}
Let us first consider the case $k \in\mathbb{Z}_{\leq 0}$, where we have $\frac{n}{n-k} \leq 1$.
It holds that
\begin{align*}
     \mathbb{P}\left[\tmu_a(t) \geq \mu_1 - \eps, \eE_{a,n}(\eps) \right] &\leq \mathbb{P}\left[ Z \leq 2(n-k) \frac{n}{n-k}\left( \frac{1}{\alpha_a} +\eps \right)\left( \alpha_a - \eta_a(\eps) \right)\right] \\
     &\leq \mathbb{P}\left[ Z \leq 2(n-k)\left( \frac{1}{\alpha_a} +\eps \right)\left( \alpha_a - \eta_a(\eps) \right)\right].
\end{align*}
Note that the definition of $\veps_a$ in Theorem~\ref{thm: RegOpt} is set to satisfy $\left( \frac{1}{\alpha} +\eps \right)\left( \alpha - \eta_a(\eps) \right)<1$ for any $\eps \leq \veps_a$.
Thus, we can apply Lemma~\ref{lem: chicramer}, which shows
\begin{equation}\label{eq: overestimationtmu2}
    \mathbb{P}\left[ Z \leq 2(n-k)\left( 1 - \frac{\eta_a(\eps)}{\alpha_a} + \eps (\alpha_a -\eta_a(\eps) ) \right)\right] \leq e^{-(n-k) D_{a,k}(\eps)},
\end{equation}
where
\begin{multline}\label{eq: defDa}
    D_{a,k}(\eps) := -\log\left( 1 - \frac{\eta_a(\eps)}{\alpha_a} + (\max(0, k) + 1)\eps (\alpha_a -\eta_a(\eps) )\right) \\ - \frac{\eta_a(\eps)}{\alpha_a} +(\max(0, k)+1) \eps (\alpha_a -\eta_a(\eps) )
\end{multline}
is a finite constant that only depends on the model parameters, $\eps$, and prior parameter $k$.

For arbitrary $n_a > 0$, applying (\ref{eq: overestimationtmu2}) to (\ref{eq: l7bungi}) gives
\begin{align*}
  \sum_{t=\bn K+1}^T \mathbb{E}[\I [j(t)=a, \tmu_1(t) \geq \mu_1 -  \eps, \eE_{a,N_a(t)}(\eps) ]] \hspace{-5em}& \\
  &\leq \sum_{t=\bn K+1}^T\mathbb{P}[j(t)=a, \tmu_a(t) \geq \mu_1 - \eps, \eE_{a,n}(\eps)] \\
    &\leq n_a +  \sum_{t=\bn K+1}^{T}\mathbb{P}[\tmu_a(t) \geq \mu_1 - \eps, \eE_{a,N_a(t)}(\eps), N_a(t) \geq n_a ] \\
    &\leq n_a + \sum_{t=\bn K+1}^T e^{-(n_a-k) D_{a,k}(\eps)}  \\
    &\leq n_a + \sum_{t=\bn K+1}^T e^{-n_a D_{a,k}(\eps)} = n_a + Te^{- n_a D_{a,k}(\eps)}.
\end{align*}
Letting $n_a = \frac{\log T }{D_{a,k}(\eps)}$ concludes the cases of priors with $k \in \mathbb{Z}_{\leq 0}$.

\paragraph{Priors with $k \in \mathbb{Z}_{>0}$}
Next, consider the case $k \in \mathbb{Z}_{>0}$.
Recall that we first play every arm $k+1$ times if $k >0$, which implies that $n - k >0$.
For $n \geq \frac{1}{\alpha \eps} + k + 1$, it holds that
\begin{equation}\label{eq: positivek2}
    \frac{n}{n-k}\left( \frac{1}{\alpha} +\eps \right) \leq \frac{1}{\alpha} + (k+1)\eps.
\end{equation}
By applying (\ref{eq: positivek2}) to (\ref{eq: l7bungi}), we have for $n \geq \frac{1}{\alpha \eps} + k + 1$,
\begin{align*}
    \mathbb{P}[ \ta_a \leq \alpha_a - \eta_a(\eps), \eE_{a,N_a(t)}(\eps)] \leq \mathbb{P}\left[ Z \leq 2(n-k)\left( 1 - \frac{\eta_a(\eps)}{\alpha_a} + (k+1)\eps (\alpha_a -\eta_a(\eps) ) \right)\right].
\end{align*}
Similarly, by applying Lemma~\ref{lem: chicramer}, one can see that for $n \geq  \frac{1}{\alpha_a \eps} + k + 1$
\begin{equation}\label{eq: overestimationtmu3}
  \mathbb{P}[ \ta_a \leq \alpha_a - \eta_a(\eps), \eE_{a,N_a(t)}(\eps)] \leq e^{-(n-k)D_{a,k}(\eps)},
\end{equation}
where $D_{a,k}(\eps)$ is defined in (\ref{eq: defDa}).

When $k \in \mathbb{Z}_{>0}$, let $n_a \geq \frac{1}{\alpha_a \eps} + k + 1 $ be arbitrary. By applying (\ref{eq: overestimationtmu3}) to (\ref{eq: l7bungi}) again, we have 
\begin{align*}
  \sum_{t=\bn K+1}^T \mathbb{E}[\I [j(t)=a, \tmu_1(t) \geq \mu_1 - \eps, \eE_{a,N_a(t)}(\eps) ]] \hspace{-10em}&\\ &\leq \sum_{t=\bn K+1}^T\mathbb{P}[j(t)=a, \tmu_a(t) \geq \mu_1 - \eps, \eE_{a,n}(\eps)] \\
    &\leq n_a +  \sum_{t=\bn K+1}^{T}\mathbb{P}[\tmu_a(t) \geq \mu_1 - \eps, \eE_{a,N_a(t)}(\eps), N_a(t) \geq n_a ] \\
    &\leq n_a + \sum_{t=\bn K+1}^T e^{-(n_a-k) D_{a,k}(\eps)} = n_a + Te^{- (n_a-k) D_{a,k}(\eps)}.
\end{align*}
Letting $n_a = k+1 + \frac{1}{\alpha_a \eps} + \frac{\log T }{D_{a,k}(\eps)}$ concludes the cases of priors with $k > 0$.

\subsubsection{Under $\TST$}
Here, we consider the case of $\TST$ where we replace $\alpha_{a,n}$ with $\ba_a(n) = \min(\ha_a(n), n)$.
From the definition of $\ba_a(n)$, it holds that for $\eps \leq \veps_a$
\begin{equation*}
    \forall n \geq \alpha_a + 1 :  \I[\ba_a(n) = \ha_a(n), \eA_{a,n}(\eps) ] = 1.
\end{equation*}
Therefore, for $n \geq \alpha_a + 1$, the analysis on $\TS$ can be applied to $\TST$ directly.

Let us consider the case where $\ba_a(n) = n < \alpha_a + 1$ holds under the condition $\eA_{a,n}(\eps)$.
By replacing $\alpha_{a,n}$ with $n$ in (\ref{eq: l7bungi}) and following the same steps as in (\ref{eq: l7bungi}) and (\ref{eq: overestimationtmu2}), we have for any $k \in \mathbb{Z}$,
\begin{align*}
    \mathbb{P}\left[\tmu_a(t) \geq \mu_1 - \eps, \eE_{a,n}(\eps) \right] &\leq \mathbb{P}\left[ Z \leq \frac{2n}{n}\left( \alpha_a - \eta_a(\eps) \right), \eE_{a,n}(\eps)\right] \\
     &\leq \mathbb{P}\left[ Z \leq 2(n-k) \frac{1}{n-k}\left( \frac{1}{\alpha_a} +\eps \right)\left( \alpha_a - \eta_a(\eps) \right), \eE_{a,n}(\eps)\right] \\ 
     &\leq \mathbb{P}\left[ Z \leq 2(n-k) \left( \frac{1}{\alpha_a} +\eps \right)\left( \alpha_a - \eta_a(\eps) \right), \eE_{a,n}(\eps)\right] \\
     &\leq e^{-(n-k) D_{a,k}(\eps)},
\end{align*}
where $D_{a,k}(\eps)$ defined in (\ref{eq: defDa}).
Therefore, the same result follows by the analysis in Section~\ref{sec: L7STS}.

\subsubsection{Asymptotic behavior of $D_{a,k}(\eps)$}
Finally, we show that $\lim_{\eps \to 0} D_{a,k}(\eps) = \KLinf(a)$.
From their definitions of $\eta_a(\eps)$ in (\ref{eq: defeta}) and $\Delta_a = \mu_1 - \mu_a$, we have
\begin{align*}
    \lim_{\eps \to 0}\eta_a(\eps) &= \lim_{\eps \to 0} \frac{\kappa_a(\Delta_a -\eps) - \eps \mu_a}{(\mu_a - \kappa_a)(\mu_1 - \kappa_a - 2\eps)}
    \\ &=\frac{\kappa_a \Delta_a}{(\mu_a - \kappa_a)(\mu_1 - \kappa_a)} =  \frac{\kappa_a (\mu_1 - \mu_a)}{(\mu_a - \kappa_a)(\mu_1 - \kappa_a)} = (\alpha_a - 1 )\frac{\mu_1 -\mu_a}{\mu_1 -\kappa_a}.
\end{align*}
Then, it holds that
\begin{align*}
  \lim_{\eps \to 0} 1 - \frac{\eta_a(\eps)}{\alpha_a} &= 1 - \left( \frac{\alpha_a - 1}{\alpha_a} \frac{\mu_1 - \mu_a}{\mu_1 - \kappa_a} \right) \\
  &= \frac{\alpha_a (\mu_1 -\kappa_a) - (\alpha_a-1)(\mu_1 -\mu_a)}{\alpha_a(\mu_1 -\kappa_a)} \\
  &= \frac{\alpha_a(\mu_a - \kappa_a) + \mu_1 - \mu_a}{\alpha_a(\mu_1 - \kappa_a)} \\ 
  &=  \frac{\mu_a}{\mu_a - \kappa_a} \frac{(\mu_a - \kappa_a) }{\alpha_a(\mu_1 - \kappa_a)} + \frac{\mu_1 - \mu_a }{\alpha_a(\mu_1 - \kappa_a)} \hspace{2em} \because \alpha_a = \frac{\mu_a}{\mu_a - \kappa_a}\\
  &= \frac{\mu_1}{\alpha_a(\mu_1 - \kappa_a)} =: \frac{1}{X_a},
\end{align*}
Therefore, from the definition of $D_{a,k}(\eps)$ in (\ref{eq: defDa})
\begin{align*}
    \lim_{\eps \to 0} D_{a,k}(\eps)&= \lim_{\eps \to 0} \bigg[
    -\log\left( 1 - \frac{\eta_a(\eps)}{\alpha_a} + (\max(0, k) + 1) (\alpha_a -\eta_a(\eps) )\eps\right) \\ 
    &\hspace{11em}- \frac{\eta_a(\eps)}{\alpha_a} +(\max(0, k)+1)  (\alpha_a -\eta_a(\eps) ) \eps\bigg] \\
    &= - \log (1- \lim_{\eps \to 0} \frac{\eta_a(\eps)}{\alpha_a}) - \lim_{\eps \to 0} \frac{\eta_a(\eps)}{\alpha_a}\\
    &= -\log \frac{1}{X_a} + \frac{1}{X_a} - 1 = \log X_a +  \frac{1}{X_a} - 1 \\
    &= \log\left( \alpha_a \frac{\mu_1 - \kappa_a}{\mu_1} \right) + \frac{\mu_1}{\alpha_a(\mu_1 -\kappa_a)} - 1 = \KLinf(a),
\end{align*}
where the last equality comes from Lemma~\ref{lem: KLinf}.
\end{proof}

\subsection{Proof of Lemma~\ref{lem: korda}}\label{app: lem7}
Firstly, we state two well-known facts that are utilized in the proof. 
\begin{fact}~\label{fact2}
When $X \sim\Par(\kappa, \alpha)$ with the scale parameter $\kappa \in \mathbb{R}$ and rate parameter $\alpha \in \mathbb{R}_{+}$, then $\log\left( \frac{X}{\kappa} \right)$ follows the exponential distribution with rate $\alpha$, i.e., $\log\left( \frac{X}{\kappa} \right) \sim \Exp(\alpha)$.
\end{fact}
\begin{fact}~\label{fact3}
Let $X_1, \ldots, X_n$ be identically independently distributed with the exponential distribution with the rate parameter $\alpha$, i.e., $X_i \stackrel{\text{i.i.d.}}{\sim} \Exp(\alpha)$ for any $i \in [n]$.
Then, their sum follows the Erlang distribution with the shape parameter $n\in\mathbb{N}$ and rate parameter $\alpha$, i.e., $\sum_{i=1}^n X_i \sim \Er(n, \alpha)$. 
\end{fact}

\lemkordaused*

\begin{proof}
When one considers the Pareto distribution with known scale parameter $\kappa$ that belongings to the one-dimensional exponential family, the posterior on the shape parameter $\alpha^{\mathrm{one}} >0$ after observing $n=N_1(t)$ rewards is given for $k \in \mathbb{Z}$
\begin{equation}\label{eq: betapost}
    \alpha^{\mathrm{one}}~|~\Ft \sim \Er\left(n-k+1, X_n \right),
\end{equation}
where $X_n = \sum_{s=1}^n \log(r_{1, s}) - n \log (\kappa_1)$.
Note that $X_n \sim \Er(n, \alpha_1)$ from Facts~\ref{fact2} and~\ref{fact3}.
Let $\ta_1^{\mathrm{one}}$ be a sample from the posterior distribution in (\ref{eq: betapost}).
Then, for one-dimensional Pareto bandits, it holds from (\ref{eq: def_inc_gamma}) that
\begin{equation*}
    \mathbb{P}[\tmu_1(t) \leq \mu_1 - \eps| \Ft] = \mathbb{P}\left[ \ta_1^{\mathrm{one}} \geq \beta \Lmid \Ft \right] = \frac{\Gamma\left(n-k+1, \beta  X_n \right)}{\Gamma(n-k+1)},
\end{equation*}
where we denoted $\beta = \frac{\mu_1 - \eps}{\mu_1-\eps -\kappa_1}$ satisfying $\mu(\kappa_1, \beta) = \mu_1 - \eps$.
Therefore, Lemma~\ref{lem: kordaresult} can be written as
\begin{align*}
    \sum_{t=1}^T \mathbb{E}\left[ \I \left[ j(t)\ne 1, \eM_\eps^c(t)\right] \right] &= \sum_{t=1}^T \sum_{n=1}^T \mathbb{E}\left[ \I \left[ j(t)\ne 1, \eM_\eps^c(t), N_1(t) = n \right] \right] \\
    &=  \sum_{t=1}^T \sum_{n=1}^T \mathbb{E}\left[\mathbb{P}\left[ j(t)\ne 1, \eM_\eps^c(t), N_1(t) = n \Lmid \Ft \right] \right] \\
    &= \sum_{t=1}^T \sum_{n=1}^T \int_0^\infty \frac{\Gamma(n+1, \beta x)}{\Gamma(n+1)} \frac{\alpha_1^n}{\Gamma(n)}x^{n-1} e^{-\alpha_1 x} \dx x \leq \mathcal{O}(\eps^{-1}),
\end{align*}
where we injected the density function of the Erlang distribution into the last equality.

On the other hand, for two-parameter Pareto bandits where the scale parameter is unknown, it holds by the law of total expectation that
\begin{align*}
    \mathbb{E}[\I[j(t)\ne 1, \eK_{1,N_1(t)}(\eps), \eM_\eps^c(t)]] &= \mathbb{E}_{\hk_1, \ha_1}\left[ \mathbb{P}[j(t)\ne 1, \eK_{1,N_1(t)}(\eps), \eM_{\eps}^c(t) | \Ft] \right] \\
    &= \mathbb{E}_{\hk_1, \ha_1}\left[ \I[\eK_{1,N_1(t)}(\eps)]\mathbb{P}[j(t)\ne 1,\eM_{\eps}^c(t) | \Ft] \right],
\end{align*}
where the last equality holds since $\eK$ is determined by the history $\Ft$.

From Lemma~\ref{lem: tmusmall} with $y = \mu_1 -\eps $, it holds for any $\xi \leq \frac{\mu_1 -\eps}{\mu_1-\eps - \kappa_1} = \beta$ that
\begin{align*}
    \I[\eK_{1,n}(\eps)]\mathbb{P}[\tmu_1(t) \leq \mu_1 - \eps| \Ft] \hspace{-8em}& \\
    &\leq  \I[\eK_{1,n}(\eps)] \left( \left( \frac{\mu_1 - \eps}{\mu((\kappa_1, \xi))}\right)^n \int_{1}^{\xi} f_{n-k, \frac{n}{\alpha_{1,n}}}^{\text{Er}}(x) \dx x + \int_{\xi}^\infty f_{n-k, \frac{n}{\alpha_{1,n}}}^{\text{Er}}(x) \dx x \right)\\
    &\leq \I[\eK_{1,n}(\eps)] \left(  \left( \frac{\mu_1 - \eps}{\mu((\kappa_1, \xi))}\right)^n \left( 1- \frac{\Gamma\left(n-k, \frac{n}{\alpha_{1,n}}\xi\right)}{\Gamma(n-k)}\right) + \frac{\Gamma\left(n-k, \frac{n}{\alpha_{1,n}}\xi\right)}{\Gamma(n-k)} \right) \numberthis{\label{eq: cvx}}
\end{align*}
which is a convex combination of $1$ and $\left( \frac{\mu_1 - \eps}{\mu((\kappa_1, \xi))}\right)^n$.
Therefore, RHS of (\ref{eq: cvx}) increases as $\frac{\Gamma\left(n-k, \frac{n}{\alpha_{1,n}}\xi\right)}{\Gamma(n-k)}$ increases.
From the definition of $\Gamma(n,x)$, it holds that $\Gamma(n, x) \geq \Gamma(n, x+y)$ for any positive $y>0$ and $\Gamma(n+1, x) = n\Gamma(n,x) + x^n e^{-x}$.
Since $\frac{n}{\ha_1(n)} \leq \frac{n}{\ba_1(n)}$ holds for any $n \in \mathbb{N}$, it holds for $k \in \mathbb{Z}_{\geq 0}$ that
\begin{equation*}
    \frac{\Gamma\left(n-k, \frac{n}{\ba_1(n)}\xi\right)}{\Gamma(n-k)}  \leq \frac{\Gamma\left(n-k, \frac{n}{\ha_1(n)}\xi\right)}{\Gamma(n-k)} \leq \frac{\Gamma\left(n, \frac{n}{\ha_1(n)}\xi\right)}{\Gamma(n)}.
\end{equation*}

Let us denote $Y_n := \frac{n}{\ha_1(n)} = \sum_{i=1}^n \log (r_{1,s}) - n \log(\hk_1(n))$, which follows the Erlang distribution with shape $n-1$ and rate $\alpha_1$~\citep{malik1970estimation}.
By taking expectation with respect to $\hk_1(n)$, we have for any $\xi \leq \beta$ that
\begin{align*}
  \mathbb{E}_{\hk_1}[ \I[\eK_{1,n}(\eps)]\mathbb{P}[\tmu_1(t) \leq \mu_1 - \eps| \Ft] ] \hspace{-8em}& \\
  &\leq \int_{\kappa_1}^{\kappa_1+\eps}  \left(  \left( \frac{\mu_1 - \eps}{\mu((\kappa_1, \xi))}\right)^n \left( 1- \frac{\Gamma\left(n, \xi Y_n\right)}{\Gamma(n)}\right) + \frac{\Gamma\left(n, \xi Y_n\right)}{\Gamma(n)} \right) \mathbb{P}[\hk_1(n) = x] \dx x  \\
  &= \mathbb{P}[\eK_{1,n}(\eps)] \left(  \left( \frac{\mu_1 - \eps}{\mu((\kappa_1, \xi))}\right)^n \left( 1-\frac{\Gamma\left(n, \xi Y_n\right)}{\Gamma(n)} \right) + \frac{\Gamma\left(n, \xi Y_n\right)}{\Gamma(n)} \right) \\
  &= \left(1 - \left( \frac{\kappa_1}{\kappa_1 + \eps} \right)^{n\alpha_1} \right) \left(  \left( \frac{\mu_1 - \eps}{\mu((\kappa_1, \xi))}\right)^n \left( 1- \frac{\Gamma\left(n, \xi Y_n\right)}{\Gamma(n)} \right) + \frac{\Gamma\left(n, \xi Y_n\right)}{\Gamma(n)} \right),
\end{align*}
where we used $\hk_1(n) \sim \Par(\kappa_1, n\alpha_1)$ in (\ref{eq: MLEdist}) for the last equality.

Therefore, under the two-parameter Pareto distribution, the following holds for any $\xi \leq \beta$ under both $\TS$ and $\TST$ with $k\in \mathbb{Z}_{\geq 0}$ that
\begin{multline*}
  \mathbb{E}_{\hk_1,\ha_1}[ \I[\eK_{1,n}(\eps)]\mathbb{P}[\tmu_1(t) \leq \mu_1 - \eps| \Ft] ] \\ 
    \leq \left(1 - \left( \frac{\kappa_1}{\kappa_1 + \eps} \right)^{n\alpha_1} \right) \int_0^{\infty} \left(  \left( \frac{\mu_1 - \eps}{\mu((\kappa_1, \xi))}\right)^n \left( 1- \frac{\Gamma\left(n, \xi y\right)}{\Gamma(n)} \right) + \frac{\Gamma\left(n, \xi y\right)}{\Gamma(n)} \right) \frac{\alpha_1^{n-1}}{\Gamma(n-1)}y^{n-2} e^{-\alpha_1 y} \dx y.
\end{multline*}
Therefore, Lemma~\ref{lem: kordaresult} concludes the proof for any $n \in \mathbb{N}$, by carefully selecting $\xi \leq \beta$ satisfying
\begin{multline*}
    \left(1 - \left( \frac{\kappa_1}{\kappa_1 + \eps} \right)^{n\alpha_1} \right) \int_0^{\infty} \left(  \left( \frac{\mu_1 - \eps}{\mu((\kappa_1, \xi))}\right)^n \left( 1- \frac{\Gamma\left(n, \xi y\right)}{\Gamma(n)} \right) + \frac{\Gamma\left(n, \xi y\right)}{\Gamma(n)} \right) \frac{\alpha_1^{n-1}}{\Gamma(n-1)}y^{n-2} e^{-\alpha_1 y} \dx y \\
    \leq \int_0^\infty \frac{\Gamma(n+1, \beta y)}{\Gamma(n+1)} \frac{\alpha_1^n}{\Gamma(n)}y^{n-1} e^{-\alpha_1 y} \dx y .
\end{multline*}
Note that when we consider $\TS$ with $k=-1$, we have to find $\xi' \leq \beta$ such that
\begin{align*}
    \left(1 - \left( \frac{\kappa_1}{\kappa_1 + \eps} \right)^{n\alpha_1} \right)  \hspace{-7em}& \\ 
    &\times \int_0^{\infty} \left(  \left( \frac{\mu_1 - \eps}{\mu((\kappa_1, \xi'))}\right)^n \left( 1- \frac{\Gamma\left(n+1, \xi' y\right)}{\Gamma(n+1)} \right) + \frac{\Gamma\left(n+1, \xi' y\right)}{\Gamma(n+1)} \right) \frac{\alpha_1^{n-1}}{\Gamma(n-1)}y^{n-2} e^{-\alpha_1 y} \dx y \\
    &\hspace{15em}\leq \int_0^\infty \frac{\Gamma(n+1, \beta y)}{\Gamma(n+1)} \frac{\alpha_1^n}{\Gamma(n)}y^{n-1} e^{-\alpha_1 y} \dx y .
\end{align*}
From $\Gamma(n, x) \geq \Gamma(n, x+y)$ for any positive $x, y>0$ and $\xi' \leq \beta$, we have for any $x > 0$ that
\begin{equation*}
    \frac{\Gamma\left(n+1, \xi' x \right)}{\Gamma(n+1)} \geq  \frac{\Gamma(n+1, \beta x)}{\Gamma(n+1)}.
\end{equation*}
Therefore, for $k \in \mathbb{Z}_{\leq -1}$, we might not be able to apply the results by \citet{KordaTS}.
\end{proof}

\subsection{Proof of Lemma~\ref{lem: minor}}\label{sec: minor}
We first state two lemmas on the event $\eK$ and $\eA$.
\begin{restatable}{lemma}{lemeKbound}\label{lem: eKbound}
For any algorithm and $a\in [K]$, it holds that for all $\eps >0$, $t >0 $, and $n \in \mathbb{N}$
\begin{equation*}
    \mathbb{P}\left[\eK_{a,N_a(t)}^c(\eps), N_a(t)=n\right]  \leq \exp\left( -\frac{\alpha_a \eps}{\kappa_a+\eps}  n \right).
\end{equation*}
\end{restatable}
\begin{restatable}{lemma}{lemeAbound}\label{lem: eAbound}
For any algorithm and for any $a\in [K]$, it holds that for all $\eps \in \left(0, \frac{\kappa_a}{\alpha_a( \kappa_a+1)}\right)$ and $t >0$, and $n\geq \bn$
\begin{equation*}
  \mathbb{P}\left[\eA_{a,N_a(t)}^c(\eps), \eK_{a,N_a(t)}(\eps), N_a(t)=n\right] \leq 2\exp\left( -\frac{\alpha_a^2 \eps^2}{4}n \right),
\end{equation*}
\end{restatable}

\lemminor*

\begin{proof}
From the Lemmas~\ref{lem: eKbound} and~\ref{lem: eAbound}, one can see that for $n \geq \bn$,
\begin{align*}
    \mathbb{P}\left[ \eE_{a,N_a(t)}^c(\eps), N_a(t)=n \right] &= \mathbb{P}\left[\eK_{a,N_a(t)}^c(\eps), N_a(t)=n\right] + \mathbb{P}\left[\eA_{a,N_a(t)}^c(\eps) , \eK_{a,N_a(t)}(\eps), N_a(t)=n \right] \\
    &\leq \exp\left( -\frac{ \alpha_a \eps}{\kappa_a+\eps}n \right) + 2 \exp\left( -\frac{\alpha_a^2 \eps^2}{4}n \right).
\end{align*}
Since $\left\{j(t)=a, \eE_{a,n}^c(\eps), N_a(t)=n \right\}$ occurs only once for any $n \in \mathbb{N}$, it holds that
\begin{align*}
    \sum_{t= \bn K + 1}^T \mathbb{E}\left[\I\left[ j(t)=a, \eE^c_{a,N_a(t)}(\eps) \right] \right] &= 
     \sum_{t= \bn K + 1}^T \sum_{n=\bn}^{T} \mathbb{E}\left[\I\left[ j(t)=a, \eE^c_{a,N_a(t)}(\eps), N_a(t) =n \right] \right] \\
    &\leq \sum_{n=\bn}^{\infty} \mathbb{E}\left[\I\left[ \eE_{a,N_a(t)}^c(\eps), N_a(t) =n\right]\right] \\
    & = \sum_{n=\bn}^{\infty} \mathbb{P}\left[\eK_{a,N_a(t)}^c(\eps), N_a(t) =n\right] \\
    & \hspace{3em}+ \mathbb{P}\left[\eA_{a,N_a(t)}^c(\eps) \cap \eK_{a,N_a(t)}(\eps), N_a(t) =n \right] \\
    &\leq\sum_{n=\bn}^{\infty} \exp\left( -\frac{ \alpha_a \eps}{\kappa_a+\eps}n \right) + 2 \exp\left( -\frac{ \alpha_a^2 \eps^2}{4}n \right).
\end{align*}
Since $\exp(-an)$ with $a >0$ is a decreasing function with respect to $n$, it holds that
\begin{equation*}
    \sum_{n=2}^{\infty} \exp(-an) \leq \int_{1}^{\infty} \exp(-an) \dx n = \frac{\exp(-a)}{a},
\end{equation*}
which concludes the proof.
\end{proof}

\subsubsection{Proof of Lemma~\ref{lem: eKbound}}
\lemeKbound*
\begin{proof}
Since $\hk_a(n) \sim \Par(\kappa_a, n\alpha_a)$ holds for any $n \in \mathbb{N}$ in (\ref{eq: MLEdist}), it holds that
\begin{align*}
    \mathbb{P}\left[ \eK^c_{a,N_a(t)}, N_a(t) =n \right] &= \mathbb{P}\left[ \hk_a(N_a(t)) \geq \kappa_a + \eps, N_a(t) =n \right] \\
    &= \left( \frac{\kappa_a}{\kappa_a + \eps} \right)^{n\alpha_a} \leq \exp\left( -\frac{ \alpha_a \eps}{\kappa_a + \eps }  n \right),
\end{align*}
which concludes the proof.
\end{proof}

\subsubsection{Proof of Lemma~\ref{lem: eAbound}}                  
\lemeAbound*
\begin{proof}
Fix a time index $t$ and denote $\mathbb{P}_t[\cdot] = \mathbb{P}[\cdot \mid \Ft]$ and $N_a(t) =n$.
To simplify notations, we drop the argument $n$ of $\hk_a(n)$ and $\ha_a(n)$. 

Let $r_{a,k}'$ be the $k$-th order statistics of $(r_{a,s})_{s=1}^n$ for arm $a$ such that $r_{a,1}' \leq r_{a,2}' \ldots \leq r_{a,n}'$.
From the definition of MLE of $\alpha_a$,
\begin{align*}
    \mathbb{P}[\ha_a \leq \alpha_a - \eps_{a,l}(\eps), \eK_{a,N_a(t)}(\eps), N_a(t)=n ] &\leq \mathbb{P} \left[\frac{n}{\sum_{s=1}^n \log r_{a,s}' - n \log r_{a,1}'} \leq \alpha_a -\eps_{a,l}(\eps)  \right] \\
    &= \mathbb{P}\left[\frac{n}{\alpha_a -\eps_{a,l}(\eps)} \leq  \sum_{s=1}^{n} \log \frac{r_{a,s}'}{r_{a,1}'} \right] \\
    &= \mathbb{P}\left[\frac{n}{\alpha_a -\eps_{a,l}(\eps)} \leq n \log \frac{\kappa}{r_{a,1}'} + \sum_{s=1}^{n} \log \frac{r_{a,s}'}{\kappa}\right] \\
    &\leq \mathbb{P}\left[\frac{n}{\alpha_a -\eps_{a,l}(\eps)} \leq  \sum_{s=1}^{n} \log \frac{r_{a,s}}{\kappa_a} \right] \\ 
    &\leq \mathbb{P}\left[\eps \leq \frac{1}{n} \sum_{s=1}^{n} \log \frac{r_{a,s}}{\kappa_a} -\frac{1}{\alpha_a} \right],
\end{align*}
where the first equality holds from the definition of MLEs in (\ref{eq: MLEdist}), the second inequality holds since any sample generated from the Pareto distribution cannot be smaller than its scale parameter $\kappa$, and the last inequality holds from the definition of $\eps_{a,l}(\eps)$ in (\ref{eq: epslu}).

Similarly, one can derive that
\begin{align*}
        \mathbb{P}[\ha_a \geq \alpha_a + \eps_{a,u}(\eps) , \eK_{a,N_a(t)}(\eps) , N_a(t)=n] &\leq \mathbb{P} \left[ \sum_{s=1}^n \log \frac{r_{a,s}}{\kappa_a} \leq \frac{n}{\alpha_a + \eps_{a,u}(\eps)} + n\log \frac{r_1'}{\kappa} \cap \eK\right] \\
        &\leq \mathbb{P} \left[ \sum_{s=1}^n \log \frac{r_{a,s}}{\kappa} \leq \frac{n}{\alpha_a + \eps_{a,u}(\eps)} + n\log \frac{\kappa_a + \eps}{\kappa_a} \right] \\
        &\leq \mathbb{P} \left[ \sum_{s=1}^n \log \frac{r_{a,s}}{\kappa_a} \leq \frac{n}{\alpha_a + \eps_{a,u}(\eps)} + \frac{n \eps}{\kappa_a} \right] \\
        &\leq \mathbb{P} \left[ \frac{1}{n}\sum_{s=1}^n \log \frac{r_{a,s}}{\kappa_a} - \frac{1}{\alpha_a} \leq -\eps \right],
\end{align*}
where the second inequality holds since $r_{a,1}' = \hk_a \leq \kappa_a + \eps$ holds on $\eK_{a,n}$, the third inequality from $\log(1+x) \leq x $ for $x>-1$, and the last inequality comes from the definition of $\eps_{a,u}(\eps)$.
From Fact~\ref{fact2}, $y_{a,s} := \log \left( \frac{r_{a,s}}{\kappa_a} \right) \sim \Exp (\alpha_a)$ and the last probability can be considered as a deviation of the sum of exponentially distributed random variables.

For the exponential distribution $\Exp(\alpha)$, we say that Bernstein's condition with parameter $b$ holds if
\begin{equation*}
    \mathbb{E}\left[M_k\right] \leq \frac{1}{2}k! \frac{1}{\alpha^2} b^{k-2} \quad \text{ for } k =3, 4, \ldots,
\end{equation*}
where $M_k$ implies the $k$-th central moment.
For $\Exp(\alpha_a)$, it holds that
\begin{equation*}
    \mathbb{E}\left[M_k\right] = \frac{!k}{\alpha_a^k} \leq \frac{k!}{2} \frac{1}{\alpha_a^2} \left( \frac{1}{\alpha_a}\right)^{k-2},
\end{equation*}
where $!k$ is the subfactorial of $k$ such that $!k \leq \frac{k!}{e} + \frac{1}{2} \leq \frac{k!}{2}$ for $k\geq 3$.
Hence, the exponential distribution with parameter $\alpha_a$ satisfies Bernstein's condition with parameter $\frac{1}{\alpha_a}$, so that it is subexponential with parameters $\left( \frac{2}{\alpha_a^2} , \frac{2}{\alpha_a} \right)$.
Therefore, by applying Lemma~\ref{lem: subexp}, we have
\begin{equation*}
    \mathbb{P}\left( \abs{\frac{1}{n} \sum_{s=1}^n y_{a,s} - \frac{1}{\alpha_a}} \geq \eps \right) \leq 2 \exp \left( - \frac{n}{4}\min \{ \alpha_a^2 \eps^2, \alpha_a \eps \}\right). 
\end{equation*}

Note that it holds for $\eps < \frac{\kappa_a}{\alpha_a(\kappa_a+1)}$ that
\begin{align*}
    \mathbb{P}[\ha_a \leq \alpha_a - \eps_{a,l}(\eps) \cap \eK_{a,n} ] &\leq \mathbb{P}\left( \frac{1}{n} \sum_{s=1}^n y_{a,s} - \frac{1}{\alpha_a} \geq \eps \right)\\
    \mathbb{P}[\ha_a \geq \alpha_a + \eps_{a,u}(\eps) \cap \eK_{a,n} ]&\leq \mathbb{P}\left( \frac{1}{n} \sum_{s=1}^n y_{a,s} - \frac{1}{\alpha_a} \leq -\eps \right),
\end{align*}
for $\eps_{a,l}(\eps) = \frac{\eps \alpha_a^2}{1+\eps \alpha_a}$ and $\eps_{a,u}(\eps) = \frac{\eps \alpha_a^2(\kappa_a+1)}{\kappa_a - \eps\alpha_a(\kappa_a+1)}$, which satisfy $\lim_{\eps \to 0} \max\{ \eps_{a,l}(\eps), \eps_{a,u}(\eps) \} = 0_+$.
Hence, by recovering the original notations, we obtain
\begin{align*}
    \mathbb{P}[\eA_{a,N_a(t)}^c(\eps), \eK_{a,N_a(t)}(\eps), N_a(t) =n]&=\mathbb{P}[\ha_a(n) \leq \alpha_a - \eps_{a,l}(\eps) , \eK_{a,N_a(t)}, N_a(t)=n ] \\
    &\hspace{1em}+  \mathbb{P}[\ha_a(n) \geq \alpha_a + \eps_{a,u}(\eps), \eK_{a, N_a(t)}, N_a(t)=n ] \\
    &\leq 2\exp \left( - \frac{\alpha^2_a \eps^2}{4}n \right),
\end{align*}
for $\eps < \frac{1}{\alpha_a}$ with $\alpha_a >1$.
\end{proof}

\subsection{Proof of Lemma~\ref{lem: tmusmall}}\label{sec: tmusmall}
\lemtmu*
\begin{proof}
Fix a time index $t$ with $N_a(t) =n$ and denote $\mathbb{P}_t[\cdot] = \mathbb{P}[\cdot \mid \Ft]$.
To simplify notations, we drop the argument $n$ of $\hk_a(n)$ and $\ha_a(n)$ and the argument $t$ of $\ta_k(t)$, $\ta_a(t)$, and $\tmu_a(t)$.

When $\hk_a < y$ holds, $\tmu_a \leq y$ can hold regardless of the value of $\tk_a$ if $\hk_a \frac{\ta_a}{\ta_a-1} \leq y$ holds since $\tk_a \in (0, \hk_a]$ holds from its posterior distribution in (\ref{eq: cdfkappa}).
Hence, if $\hk_a < y$, then
\begin{equation}\label{eq: tmuy1}
  \ta(t) \geq \frac{y}{y-\hk_a} \implies \tmu_a \leq y.
\end{equation}
When $1 < \ta(t) < \frac{y}{y-\hk_a}$,
\begin{equation}\label{eq: tk}
    \tmu_a = \tk_a \frac{\ta_a}{\ta_a-1} \leq y  \, \Leftrightarrow  \, \tk_a \leq y \frac{\ta_a-1}{\ta_a} .
\end{equation}

Since $\ta_a \leq 1$ implies $\tmu_a = \infty$, from (\ref{eq: tmuy1}) and (\ref{eq: tk}), it holds that
\begin{align*}
    \mathbb{P}_t[\tmu_a \leq y]& = \int_{1}^{\frac{y}{y - \hk_a}} f_{n-k, \frac{n}{\alpha_{a,n}}}^{\text{Er}}(x) \mathbb{P}_t\left[\tk_a \leq \frac{x-1}{x}y\right] \dx x + \int_{\frac{y}{y - \hk_a}}^{\infty} f_{n-k, \frac{n}{\alpha_{a,n}}}^{\text{Er}}(x) \dx x \\
    &= \int_{1}^{\frac{y}{y-\hk_a}} f_{n-k, \frac{n}{\alpha_{a,n}}}^{\text{Er}}(x) \left( \frac{x-1}{\hk_a x} y \right)^{nx} \dx x + \int_{\frac{y}{y-\hk_a}}^{\infty} f_{n-k, \frac{n}{\alpha_{a,n}}}^{\text{Er}}(x) \dx x, \numberthis{\label{eq: tmucase1}}
\end{align*}
where we denoted $\mathbb{P}_t[\cdot] = \mathbb{P}[\cdot | \Ft]$ and recovered the CDF in (\ref{eq: cdfkappa}) in (\ref{eq: tmucase1}).
Take any finite $y' > y$ and let $\xi := \frac{y'}{y'-\kappa_a}< \frac{y}{y-\kappa_a}$ such that $\mu((\kappa_a, \xi)))=y'$.
Since $\frac{a}{a-b}$ is decreasing with respect to $a > b >0$, one can see that
\begin{align*}
    \mathbb{P}_t[\tmu_a \leq y] &= \int_{1}^{\frac{y}{y-\hk_a}} f_{n-k, \frac{n}{\alpha_{a,n}}}^{\text{Er}}(x) \left( \frac{x-1}{\hk_a x} y \right)^{nx} \dx x + \int_{\frac{y}{y-\hk_a}}^{\infty} f_{n-k, \frac{n}{\alpha_{a,n}}}^{\text{Er}}(x) \dx x \\
    &\leq \int_{1}^{\frac{y'}{y'-\hk_a}} f_{n-k, \frac{n}{\alpha_{a,n}}}^{\text{Er}}(x) \left( \frac{x-1}{\hk_a x} y \right)^{nx} \dx x + \int_{\frac{y'}{y'-\hk_a}}^{\infty} f_{n-k, \frac{n}{\alpha_{a,n}}}^{\text{Er}}(x) \dx x
    \\
    &\leq \int_{1}^{\frac{y'}{y'-\kappa}} f_{n-k, \frac{n}{\alpha_{a,n}}}^{\text{Er}}(x) \left( \frac{x-1}{\hk_a x} y \right)^{nx} \dx x + \int_{\frac{y'}{y'-\kappa}}^{\infty} f_{n-k, \frac{n}{\alpha_{a,n}}}^{\text{Er}}(x) \dx x\\
    &\leq \int_{1}^{\xi} f_{n-k, \frac{n}{\alpha_{a,n}}}^{\text{Er}}(x) \left( \frac{x-1}{\kappa x} y \right)^{n} \dx x + \int_{\xi}^{\infty} f_{n-k, \frac{n}{\alpha_{a,n}}}^{\text{Er}}(x) \dx x\numberthis{\label{eq: star_5}}\\
    &\leq \left( \frac{\xi-1}{\kappa \xi} y\right)^{n} \int_{1}^{\xi} f_{n-k, \frac{n}{\alpha_{a,n}}}^{\text{Er}}(x) \dx x + \int_{\xi}^{\infty} f_{n-k, \frac{n}{\alpha_{a,n}}}^{\text{Er}}(x) \dx x \numberthis{\label{eq: star_6}}\\
    &= \left( \frac{y}{\mu((\kappa, \xi))}\right)^{n} \int_{1}^{\xi} f_{n-k, \frac{n}{\alpha_{a,n}}}^{\text{Er}}(x) \dx x + \int_{\xi}^{\infty} f_{n-k, \frac{n}{\alpha_{a,n}}}^{\text{Er}}(x) \dx x,
\end{align*}
where (\ref{eq: star_5}) comes from $\hk_a \geq \kappa$ and we used the increasing property of $\frac{x-1}{x}$ in (\ref{eq: star_6}).
\end{proof}

\subsection{Proof of Lemma~\ref{lem: pn_cases}}

\lempn*

\begin{proof}
Similarly to other proofs, fix $t$ and let $N_1(t) =n$.
To simplify notations, we drop the argument $t$ of $\tk_1(t), \ta_1(t)$ and $\tmu_1(t)$ and the argument $n$ of $\hk_1(n), \ha_1(n), \ba_1(n)$.

\paragraph{Case 1. On $\{\hk_1 \geq \mu_1 - \eps\}$}
Under the condition $\{\hk_1 \geq \mu_1 - \eps\}$, the event $\{ \tmu_1 \leq \mu_1 - \eps \}$ is eventually determined by the value of $\tk$ since $\{\tk_1 \in (\mu_1 - \eps, \hk_1]\}$ is a sufficient condition to $\{\tmu_1 > \mu_1 - \eps\}$.
Therefore, if $\hk_1 \geq \mu_1 - \eps$, then
\begin{equation*}
     p_n(\eps|{\theta}_{1,n}) = \int_{1}^{\infty} f_{n-k, \frac{n}{\alpha_{1,n}}}^{\text{Er}}(x) \left( \frac{\mu_1 - \eps}{\hk_1} \frac{x-1}{x}\right)^{nx} \dx x.
\end{equation*}
Then, 
\begin{align*}
    \I[\hk_1 \geq \mu_1 -  \eps]p_n(\eps|{\theta}_{1,n}) &=  \I[\hk_1 \geq \mu_1 - \eps] \left( \int_{1}^{\infty} f_{n-k, \frac{n}{\alpha_{1,n}}}^{\text{Er}}(x) \left( \frac{\mu_1 -\eps}{\hk_1} \frac{x-1}{x}\right)^{nx} \dx x \right) \\
    &\leq \I[\hk_1 \geq \mu_1 - \eps]\int_{1}^{\infty} f_{n-k, \frac{n}{\alpha_{1,n}}}^{\text{Er}}(x) \left( 1- \frac{1}{x} \right)^{nx} \dx x \\
    &  \leq \I[\hk_1(n) \geq \mu_1 - \eps] \int_{1}^{\infty} f_{n-k, \frac{n}{\alpha_{1,n}}}^{\text{Er}}(x)   e^{-n} \dx x   \\
    &\leq \I[\hk_1(n) \geq \mu_1 - \eps] e^{-n},
\end{align*}
where the second inequality holds from $\mu_1 -\eps \leq \hk_1$.

\paragraph{Case 2. On $\{\hk_1 \in K(\eps), \alpha_{1,n} \leq \alpha_1 + \rho\}$}
By applying Lemma~\ref{lem: tmusmall} with $y=\mu_1 - \eps$, we have for any $\xi \leq \frac{\mu_1- \eps}{\mu_1 -\eps - \kappa_1}$ that
\begin{align*}
    \I[\hk_1 < \mu_1 - \eps, \alpha_{1,n} \leq \alpha + \rho]p_n(\eps|{\theta}_{1,n}) &\leq \left( \frac{\mu_1 -\eps}{\mu((\kappa_1, \xi))} \right)^{n} \int_{1}^\xi f_{n-k, \frac{n}{\alpha_{1,n}}}^{\text{Er}}(x) \dx x   + \int_{\xi}^\infty f_{n-k, \frac{n}{\alpha_{1,n}}}^{\text{Er}}(x) \dx x \\
    &\leq \left( \frac{\mu_1 -\eps}{\mu((\kappa_1, \xi))} \right)^{n} \int_{0}^\xi f_{n-k, \frac{n}{\alpha_{1,n}}}^{\text{Er}}(x) \dx x + \int_{\xi}^\infty f_{n-k, \frac{n}{\alpha_{1,n}}}^{\text{Er}}(x) \dx x. \numberthis{\label{eq: L9case2_1}}
\end{align*}
Let us define $\bar{\rho} := \rho_{\theta}(\eps/2)$.
Then, it satisfies $\mu((\kappa_1, \alpha_1+\bar{\rho})) = \mu_1 - \frac{\eps}{4}$ and
\begin{equation*}
    \alpha_1 + \bar{\rho} = \frac{\mu -\eps/4}{\mu -\eps/4 -\kappa_1} < \frac{\mu -\eps}{\mu -\eps -\kappa_1},
\end{equation*}
where the inequality holds from the decreasing property of the function $\frac{x}{x-y}$ with respect to $x>y$.
By replacing $\xi$ with $\alpha_1 + \bar{\rho}$ in (\ref{eq: L9case2_1}), we have
\begin{align*}
  \I[\hk_1 &< \mu_1 - \eps, \alpha_{1,n} \leq \alpha_1 + \rho]p_n(\eps|\bar{\theta}_{1,n}) \\ 
  &\leq \I[\hk_1 < \mu_1 - \eps, \alpha_{1,n} \leq \alpha_1 + \rho] \bigg( \left( \frac{\mu_1 - \eps}{\mu_1 -\eps/4} \right)^n  \int_{0}^\xi f_{n-k, \frac{n}{\alpha_{1,n}}}^{\text{Er}}(x) \dx x
    + \int_{\alpha_1 + \bar{\rho}}^\infty f_{n-k, \frac{n}{\alpha_{1,n}(n)}}^{\text{Er}}(x) \dx x \bigg) \\
    &\leq \I[\hk_1 < \mu_1 - \eps, \alpha_{1,n} \leq \alpha_1 + \rho] \bigg(e^{-n \left( \frac{3\eps}{4\mu_1 -\eps} \right)} \left( 1- \mathbb{P}[\ta_1 \geq \alpha_1 + \bar{\rho}]\right)
    + \mathbb{P}[\ta_1 \geq \alpha_1 + \bar{\rho}] \bigg).
    \numberthis{\label{eq: case21}}
\end{align*}

Let $Z_n$ be a random variable that follows the chi-squared distribution with $n$ degree of freedom and $F_n(\cdot)$ denote the CDF of $Z_n$.
Then, it holds that
\begin{align*}
  \mathbb{P}\bigg[ \ta_1 \geq \alpha_1 + \bar{\rho}, \alpha_{1,n} \leq \alpha_1 + \rho \bigg]  &=\mathbb{P}\left[ Z \geq \frac{2n}{\alpha_{1,n}}(\alpha_1 + \bar{\rho}), \alpha_{1,n} \leq \alpha_1 + \rho \right]  \hspace{1em}\text{by Fact~\ref{fact1}}
    \\
    & 
    \leq \mathbb{P}\left[ Z \geq 2n \frac{\alpha_1 + \bar{\rho}}{\alpha_1 + \rho} \right]
    \\ 
    &\leq \mathbb{P}\left[ Z \geq  2n \frac{\mu_1 - \eps/4}{\mu_1 - \eps/2} \right] \\
    &= 1 - F_{2n-2k}(2n(1+\zeta)), \numberthis{\label{eq: Fchi}}
\end{align*}
where $\zeta = \frac{\eps}{4\mu_1 - 2\eps} \in (0, 1)$.
By applying Lemma~\ref{lem: wallace}, we have if $n\zeta > -k$,
\begin{align*}
    \mathbb{P}[\ta_1 \geq \alpha_1 + \bar{\rho}, \alpha_{1,n} \leq \alpha_1 + \rho] \hspace{-4em}& \\
    &\leq 1-F_{2n-2k}\left( 2n (1+\zeta) \right) 
    \\ &< \frac{1}{2}\frac{\sqrt{2\pi} (n-k)^{n-k-1/2}e^{-(n-k)}}{\Gamma(n-k)} \mathrm{erfc}\left(\sqrt{n(\zeta+k)-(n-k)\log \frac{n(1+\zeta)}{n-k}} \right),
\end{align*}
where $\Gamma(\cdot)$ denotes the Gamma function.
For $n \geq 1/2$, it holds from Stirling's formula that
\begin{equation*}
    \sqrt{2\pi}n^{n-1/2}e^{-n} \leq \Gamma(n) \leq \sqrt{2\pi}e^{1/6}n^{n-1/2}e^{-n},
\end{equation*}
which results in
\begin{equation}\label{eq: erfcineq}
  \mathbb{P}[\ta_1\geq \alpha_1 + \bar{\rho}, \alpha_{1,n} \leq \alpha_1 + \rho]
    <  \frac{1}{2}\mathrm{erfc}\left(\sqrt{n(\zeta+k)-(n-k)\log \frac{n(1+\zeta)}{n-k}} \right).
\end{equation}
Notice that $(n-k)\log \frac{n(1+\zeta)}{n-k} >0$ always holds from the assumption of $n\zeta> -k$ where priors with $k \in \mathbb{Z}_{\geq 0}$ satisfies regardless of $n$.
Thus, if $\zeta + k \leq 0$, then the argument in the complementary error function in (\ref{eq: erfcineq}) becomes negative.
This makes the upper bound in (\ref{eq: erfcineq}) greater than or equal to $\frac{1}{2}$.
Therefore, for the priors with $k \in \mathbb{Z}_{<0}$, the right term in (\ref{eq: erfcineq}) is bounded by a constant since $\zeta \in (0,1)$.

Since the complementary error function is a decreasing function, for priors with $k\in \mathbb{Z}_{\geq 0}$, it holds from (\ref{eq: erfcineq}) that 
\begin{equation*}
    \mathbb{P}[\ta_1 \geq \alpha_1 + \bar{\rho}, \alpha_{1,n} \leq \alpha_1 + \rho]
  \leq \frac{1}{2}\mathrm{erfc}\left(\sqrt{n( \zeta - \log (1+\zeta)} \right),
\end{equation*}
where we substitute $k=0$.
By the change of variables, the complementary error function is bounded for any $x \geq 0$ as follows~\citep{simon1998erfc}:
\begin{equation*}
    \mathrm{erfc}(x) \leq e^{-x^2},
\end{equation*}
which implies
\begin{equation}\label{eq: case2}
  \mathbb{P}[\ta_1 \geq \alpha_1 + \bar{\rho}, \alpha_{1,n} \leq \alpha_1 + \rho]\leq \frac{1}{2} e^{-nc_{\mu_1}(\eps)},
\end{equation}
where $c_{\mu_1}(\eps) = \zeta - \log(1+\zeta) >0$ is a deterministic constants on $\mu_1$ and $\eps$.
By combining (\ref{eq: case2}) with (\ref{eq: case21}), we have
\begin{equation*}
     \I[\hk_1 < \mu_1 - \eps, \alpha_{1,n} \leq \alpha_1 + \rho]p_n(\eps|{\theta}_{1,n}) \leq e^{-n \frac{3\eps}{4\mu_1}} \left( 1- \frac{1}{2}e^{-nc_{\mu_1}(\eps)} \right) + \frac{1}{2}e^{-nc_{\mu_1}(\eps)} =: h(\mu_1, \eps, n).
\end{equation*}

From the power-series expansion of $\log(1+x)$, we have $\log(1+x) \geq x - \frac{x^2}{2} + \frac{x^3}{3}$ for $x \in (0,1)$ and 
\begin{align*}
    c_{\mu_1}(\eps) = \zeta - \log(1+\zeta) \leq \frac{\zeta^2}{2} - \frac{\zeta^3}{3} &= \frac{\zeta^2}{6}(3-2\zeta) \\
    &\leq \frac{\zeta^2}{2} = \mathcal{O}(\eps^{-2}).
\end{align*}

\paragraph{Case 3. On $\{ \hk_1 \in K(\eps), \alpha_{1,n} \geq \alpha_1 + \rho \}$}
By applying Lemma~\ref{lem: tmusmall} with $y=\mu_1 - \eps$ and $\xi = \alpha_1 + \rho$, we have
\begin{align*}
      \I[\hk_1 < \mu_1 - \eps] p_n(\eps | {\theta}_{1,n}) &\leq
      \bigg( \frac{\mu_1 - \eps}{\mu_1 - \eps/2} \bigg)^{n} \int_1^{\alpha_1 + \rho}  f_{n-k, \frac{n}{\alpha_{1,n}}}^{\mathrm{Er}}(x) \dx x + \int_{\alpha_1 + \rho}^\infty  f_{n-k, \frac{n}{\alpha_{1,n}}}^{\mathrm{Er}}(x)  \dx x  \\
      &=
        C_1(\mu_1,\eps, n) \mathbb{P}\left[\ta_1 \in [1, \alpha_1+ \rho] \mid \alpha_{1,n}\right] + \mathbb{P}\left[\ta_1 \geq  \alpha_1 + \rho \mid \alpha_{1,n} \right]\\
      &\leq  C_1(\mu_1,\eps, n) \mathbb{P}\left[\ta_1 \leq \alpha_1 + \rho \mid \alpha_{1,n}\right] + \mathbb{P}\left[\ta_1 \geq  \alpha_1  + \rho \mid \alpha_{1,n} \right] \\
      &=  C_1(\mu_1,\eps, n)G_k(1/\alpha_{1,n};n) + (1-G_k(1/\alpha_{1,n};n)),  \numberthis{\label{eq: pnKeps}}
\end{align*}
where $A_n = C_1(\mu_1,\eps, n) := \left( \frac{\mu_1 - \eps}{\mu_1 - \eps/2} \right)^{n} \leq e^{-n \frac{\eps}{2\mu_1 - \eps}} < 1$.
Since $\ta_1$ follows $\Er\left(n-k, \frac{n}{\alpha_{1,n}} \right)$, it holds that
\begin{align*}
    \mathbb{P}\left[\ta_1 \leq \alpha_1 + \rho \mid \alpha_{1,n}\right]= \frac{\gamma\left(n-k, \frac{n(\alpha_1+\rho)}{\alpha_{1,n}}\right) }{\Gamma(n-k)},
\end{align*}
where $\gamma(\cdot, \cdot)$ denotes the lower incomplete gamma function.
Therefore, letting 
\begin{equation*}
    G_k(x;n) := \frac{\gamma\left(n-k, n(\alpha_1+\rho)x \right) }{\Gamma(n-k)}
\end{equation*}
concludes the proof.
\end{proof}

\section{Proof of Theorem~\ref{thm: RegSubOopt}}~\label{sec: SubOpt}
As shown in proofs of Theorem~\ref{thm: RegOpt}, the integral term in (\ref{eq: integral}) diverges for $k \in \mathbb{Z}_{\leq 1}$ without the restriction on $\ha$.
In this section, we provide the partial proof of Theorem~\ref{thm: RegSubOopt} for $k\in \mathbb{Z}_{\leq 0}$, which shows the necessity of such requirement to achieve asymptotic optimality.
\begin{proof}[Proof of Theorem~\ref{thm: RegSubOopt}]
We consider a two-armed bandit problem with $\theta_1 = (\kappa, \alpha_1)$ and $\theta_2 = (\kappa, \alpha_2)$.
Let $\gamma = \max\{ \lceil\alpha_2  \rceil, \lfloor\alpha_2\rfloor +1\}$ and $m = \frac{\gamma}{\gamma-1}$, so that $\frac{\mu_2}{m} = \kappa \frac{\alpha_2}{\alpha_2 -1} \frac{\gamma-1}{\gamma} > \kappa$.
Assume $1 < \alpha_1 <\alpha_2$ and $\tmu_2(s) = \mu_2 = \kappa \frac{\alpha_2}{\alpha_2 -1}$ for all $s\in \mathbb{N}$.
Recall that $\TS$ starts from playing every arms twice for priors $k \leq 1$, i.e., $N_a(s) \geq 2$ holds for all $a \in \{1, 2\}$.
We have for $T \geq 5$
\begin{align*}
    \mathbb{E}[\reg(T)] &= \Delta_2 \mathbb{E} \left [\sum_{t=1}^T \I [j(t) = 2] \right] \\
    &\geq \Delta_2 \mathbb{E} \left[ \sum_{t=5}^T \I [j(t)=2, N_1(t) = 2] \right].
\end{align*}
From the definition of $N_1(\cdot)$, $\{ j(s) \ne 2, N_1(s) =2 \}$ implies $N_1(t) > 2$ for $t > s$.
Therefore, for any $t\geq 5$,
\begin{align*}
    \{j(t) = 2, N_1(t) = 2 \} &\Leftrightarrow \{ \forall s \in [1, t-4] : j(s+4) = 2 \} \\
     &\Leftrightarrow \{ \forall s \in [1, t-4] :\tmu_1(s+4) < \mu_2 \}.
\end{align*}
Let $T' = T - 4$, then we have
\begin{align*}
    \mathbb{E}\bigg[\sum_{t=5}^T \I  [j(t)=2, N_1(t) = 2] \bigg] &=
    \mathbb{E}\left[  \sum_{t=5}^{T} \I \left[\forall s \in [1, t-4] :\tmu_1(s+4) < \mu_2 \right] \right]  \\
    &= \mathbb{E}\left[ \sum_{s=1}^{T'} \left(\mathbb{P}[\tmu_1 \leq \mu_2 | \hk_1(2), \ha_1(2)]\right)^{s}  \right]. \numberthis{\label{eq: sublower}}
\end{align*}
Notice that $\hk_1(N_1(s))=\hk_1(2)$ and $\ha_1(N_1(s))=\ha_1(2)$ hold for all $s\geq 2$ since only $j(s)=2$ holds for all $s \geq 2$.
Here, we first provide the lower bound on $\mathbb{P}[\tmu_1 \leq \mu_2 | \hk_1, \ha_1]$.
Since $\frac{\mu_2}{m} = \frac{\kappa \alpha_2}{\alpha_2 - 1} \frac{\gamma-1}{\gamma} > \kappa$ holds, we can consider the case where $\hk_1(2) \leq  \frac{\mu_2}{m}$ occurs.

From (\ref{eq: tmucase1}), it holds for $y \geq \hk_1(n)$ that
\begin{align*}
    \mathbb{P}_t[\tmu_a \leq y] &= \int_{1}^{\frac{y}{y-\hk(n)}} f_{n-k, \frac{n}{\alpha_{n}}}^{\text{Er}}(x) \left( \frac{x-1}{\hk x} y \right)^{nx} \dx x + \int_{\frac{y}{y-\hk(n)}}^{\infty} f_{n-k, \frac{n}{\alpha_{n}}}^{\text{Er}}(x) \dx x \hspace{2em} \text{by (\ref{eq: tmucase1})}\\
    &\geq \int_{\frac{y}{y-\hk(n)}}^{\infty} f_{n-k, \frac{n}{\alpha_{n}}}^{\text{Er}}(x) \dx x.
\end{align*}
By letting $n=2$ and $y=\mu_2$, we have for $k \in \mathbb{Z}_{\leq 1}$
\begin{align*}
    \mathbb{P}[\tmu_1 \leq \mu_2 | \hk_1(2), \ha_1(2)] &\geq \I\left[\hk \leq  \frac{\mu_2}{m}\right]\int_{\frac{\mu_2}{\mu_2 - \hk}}^\infty f_{2-k, \frac{2}{\ha}}^{\text{Er}}(x)  \dx x \\
    &\geq \I\left[\hk \leq  \frac{\mu_2}{m}\right] \int_{\gamma}^\infty f_{2-k, \frac{2}{\ha}}^{\text{Er}}(x)  \dx x \numberthis{\label{eq: star_2}}\\
    &= \I\left[\hk \leq  \frac{\mu_2}{m}\right] \frac{\Gamma(2-k, \frac{2\gamma}{\ha})}{\Gamma(2-k)},  \numberthis{\label{eq: gz}}
\end{align*}
where (\ref{eq: star_2}) holds from $\frac{\mu_2}{\mu_2 - \hk_1(2)} \leq \frac{\mu_2}{\mu_2 - \mu_2/m} = \frac{m}{m-1} = \gamma = \lceil \alpha_2 \rceil$
and $\Gamma(\cdot, \cdot)$ is the upper incomplete Gamma function.
To simplify the notations, we drop the arguments on $n$ and $t$ of $\tmu_1(t)$, $\hk_1(n)$, and $\ha_1(n)$ in the following sections.

\subsection{Priors $k\in\mathbb{Z}_{\leq 0}$}
Note that $\Gamma(n, x)$ is an increasing function with respect to $n$ for fixed $x$.
Therefore, (\ref{eq: gz}) implies that if the lower bound of regret for the reference prior is larger than the lower bound, then every prior with $k\in\mathbb{Z}_{\leq 0}$ are suboptimal.
Therefore, let us consider the case $k=1$, where we can rewrite (\ref{eq: gz}) as
\begin{equation}\label{eq: gz0}
    \mathbb{P}[\tmu_1 \leq \mu_2 | \hk_1, \ha_1]  \geq \I\left[\hk_1 \leq  \frac{\mu_2}{m}\right] \frac{\Gamma(1, \frac{2\gamma}{\ha_1})}{\Gamma(1)} = e^{-\frac{2\gamma}{\ha_1}}.
\end{equation}
Since $\ha_1(2) \sim \InvG(1, 2\alpha_1)$ in (\ref{eq: MLEdist}), $z := \frac{2\gamma}{\ha}$ follows an exponential distribution with rate parameter $\alpha_1/\gamma$, i.e., $z \sim \Exp(\alpha_1/\gamma)$.
By combining (\ref{eq: gz0}) with (\ref{eq: sublower}), we have
\begin{align*}
    \mathbb{E}\Bigg[  \sum_{t=5}^T \I [j(t)=2, N_1(t) = 2] \Bigg]
    &\geq \mathbb{E}_{\hk, z}\left[\sum_{s=1}^{T'} \bigg( \I[\hk \leq \mu_2/m]e^{-z} \bigg)^{s}  \right] \\
    &= \mathbb{P}[\hk \leq \mu_2/m]\mathbb{E}_{z\sim \Exp(\alpha_1/\gamma)}\left[ \sum_{s=1}^{T'} e^{-zs}  \right], \numberthis{\label{eq: gz0_1}}
\end{align*}
where we used the stochastic independence of $\ha$ and $\hk$.
Here,
\begin{align*}
    \mathbb{E}_{z\sim \Exp(\alpha_1/\gamma)}\left[ \sum_{s=1}^{T'} e^{-zs}  \right] &=  \mathbb{E}_{z\sim \Exp(\alpha_1/\gamma)}\left[ (1-e^{-zT'}) \frac{e^{-z}}{1-e^{-z}}  \right] \\
    &=  \int_{0}^{\infty} (1-e^{-xT'})\frac{e^{-x}}{1-e^{-x}} e^{-\frac{\alpha_1}{\gamma}x} \dx x \\
    &\geq \int_{0}^{\infty} (1-e^{-xT'}) \frac{e^{-2x}}{1-e^{-x}} \dx x \hspace{2em} \text{by } \frac{\alpha_1}{\gamma} < 1 \\
    &\geq \left(1 - \frac{1}{e} \right) \int_{\frac{1}{T'}}^{\infty} \frac{e^{-2x}}{1-e^{-x}} \dx x  \\
    &= \left(1 - \frac{1}{e} \right) \left[ \log(e^{x}-1) - z + e^{-z} \right]_{x=\frac{1}{T'}}^{\infty} 
    \\
    &\geq \left(1 - \frac{1}{e} \right) \left ( \log T'+1-\frac{3}{2T'}\right), \numberthis{\label{eq: T3k0rslt}}
\end{align*}
where the last inequality holds from its power series expansion
\begin{align*}
     \log(e^{x}-1) - x + e^{-x} \geq \log(x) + 1 - \frac{3}{2}x
\end{align*}
and $\lim_{x\to \infty} \log(e^{x}-1) - x + e^{-x} =0$.
By combining (\ref{eq: T3k0rslt}) with (\ref{eq: gz0_1}) and (\ref{eq: sublower}) and elementary calculation with $\hk_1(2) \sim \Par(\kappa_1, 2\alpha_1)$, we have
\begin{align*}
    \mathbb{E}[\reg(T)] &\geq \Delta_2 \left( 1- \left(\frac{m\kappa}{\mu_2}\right)^{2\alpha_1} \right) \left( 1 - \frac{1}{e} \right)\left( \log T' + 1 -\frac{3}{2T'} \right) \\
    &= \Delta_2 \left( 1- \left(\frac{m\kappa}{\mu_2}\right)^{2\alpha_1} \right) \left( 1 - \frac{1}{e} \right) \left( \log (T+4) + 1 -\frac{3}{2(T+4)} \right).
\end{align*}
Therefore, under $\TS$ with $k \in \mathbb{Z}_{\leq 1}$, there exists a constant $C(\alpha_1, \alpha_2)$ such that
\begin{equation*}
    \liminf_{T \to \infty} \frac{\mathbb{E}[\reg(T)]}{\log T} \geq C(\alpha_1, \alpha_2).
\end{equation*}

\subsection{Priors $k \in \mathbb{Z}_{\leq 0}$}
Similarly to the last section, it is sufficient to consider the case $k=0$, where we can rewrite (\ref{eq: gz}) as
\begin{equation}\label{eq: gz1}
    \mathbb{P}[\tmu_1 \leq \mu_2 | \hk_1, \ha_1]  \geq \I\left[\hk_1 \leq  \frac{\mu_2}{m}\right] \frac{\Gamma(2, \frac{2\gamma}{\ha_1})}{\Gamma(2)}.
\end{equation}
From the definition of the upper incomplete Gamma function, we have
\begin{equation*}
    g(z) := \Gamma(2, z) = \int_z^\infty x^1 e^{-x} \dx x = e^{-z} (z + 1),
\end{equation*}
as a counterpart of $e^{-z}$ in (\ref{eq: gz0_1}) with the same notations $z = \frac{2\gamma}{\ha_1} \sim \Exp\left(\frac{\alpha_1}{\gamma}\right)$.

Therefore, by replacing $e^{-zs}$ in (\ref{eq: gz0_1}) with $g(z)^s$, we have
\begin{align*}
    \mathbb{E}_{z}\left[ \sum_{s=1}^{T'} (g(z))^s  \right] 
    &\geq \mathbb{E}_{z}\left[ \I[z\in(0,1]] \sum_{s=1}^{T'} (g(z))^s  \right] \\
    &\geq \mathbb{E}_{z}\left[ \I[z\in(0,1]] \sum_{s=1}^{T'} (1-z^2)^s  \right] \\
    &= \mathbb{E}_{z} \left[ \I[z\in(0,1]] (1-(1-z^2)^{T'}) \frac{1-z^2}{z^2} \right],
\end{align*}
where we used the fact $z \in [0,1]$, $g(z) \geq 1-z^2$ in the second inequality.
Since $z \in \left(0, \frac{1}{\sqrt{T'}} \right]$, $(1-z^2)^{T'} \leq \frac{1}{1+T'z^2}$ holds, we have $1-(1-z^2)^{T'} \geq \frac{T'z^2}{1+T'z^2}$.

By applying this fact, we have for $T' > 1$,
\begin{align*}
  \mathbb{E}_{z}\Bigg[ \sum_{s=1}^{T'}  (g(z))^s  \Bigg] &\geq \mathbb{E}_{z} \left[ \frac{T'(1-z^2)}{1+T'z^2}  \I\left[z\in\left(0,\frac{1}{\sqrt{T'}}\right]\right] \right] \\
    &\geq\mathbb{E}_{z\sim \Exp(\alpha_1/\gamma)} \left[  \left( \frac{T'}{2} - \frac{1}{2} \right) \I\left[z\in\left(0,\frac{1}{\sqrt{T'}}\right]\right] \right] \hspace{3em} \\
    &= \left( \frac{T'}{2} - \frac{1}{2} \right) \int_0^{\frac{1}{\sqrt{T'}}}  \frac{\alpha_1}{\gamma}  e^{-\frac{\alpha_1}{\gamma}z} \dx z \numberthis{\label{eq: star_4}} \\
    &= \left( \frac{T'}{2} - \frac{1}{2} \right) \left(1-e^{-\frac{\alpha_1}{\gamma\sqrt{T'}}}\right).
\end{align*}
Notice that $e^{-x} \leq 1 - \frac{x}{2}$ holds for $x < 1$, which gives
\begin{align*}
      \mathbb{E}_{z}\bigg[  \sum_{s=1}^{T'}  (g(z))^s  \bigg] &\geq \left( \frac{T'}{2} - \frac{1}{2} \right) \left(1-e^{-\frac{\alpha_1}{\gamma\sqrt{T'}}}\right) \\
      &\geq \left( \frac{T'}{2} - \frac{1}{2} \right) \frac{\alpha_1}{2\gamma\sqrt{T'}} = \frac{\alpha_1}{4\gamma}\left( \sqrt{T'} - \frac{1}{\sqrt{T'}} \right). \numberthis{\label{eq: suboptimalend}}
\end{align*}

By applying (\ref{eq: suboptimalend}) to (\ref{eq: sublower}), we obtain for $k \in \mathbb{Z}_{\leq 0}$ and $T' = T-4 > 1$, 
\begin{align*}
    \mathbb{E}[\reg(T)] &\geq \Delta_2 \frac{\alpha_1}{4\gamma} \left( 1- \left(\frac{m\kappa}{\mu_2}\right)^{2\alpha_1} \right) \left( \sqrt{T'} - \frac{1}{\sqrt{T'}} \right) \\
    &= \mathcal{O}(\sqrt{T}).
\end{align*}
Notice that from the definition of $m = \frac{\gamma}{\gamma-1} = \frac{\ceil{\alpha_2}}{\ceil{\alpha_2}-1}$, $m \frac{\kappa}{\mu_2} = m\left( 1 - \frac{1}{\alpha_2}\right) < 1$ holds.
Therefore, under $\TS$ with priors $k \in \mathbb{Z}_{\leq 0}$, there exists a constant $C'(\alpha_1, \alpha_2) > 0$ such that
\begin{equation*}
    \liminf_{T \to \infty} \frac{\mathbb{E}[\reg(T)]}{\sqrt{T}} \geq C'(\alpha_1, \alpha_2).
\end{equation*}
\end{proof}

\section{Priors and posteriors}\label{sec: bayesian}
In this section, we provide details on the problem of Jeffreys prior and the probability matching prior under the multiparameter models.
One can find more details from references in this section.
\subsection{Problems of the Jeffreys prior in the presence of nuisance parameters}
The Jeffreys prior was defined to be proportional to the square root of the determinant of the FI matrix so that it remains invariant under all one-to-one reparameterizations of parameters~\citep{jeffreys1998theory}.
However, the Jeffreys prior is known to suffer from many problems when the model contains nuisance parameters~\citep{datta1995some,ghosh2011objective}.
Therefore, Jeffreys himself recommended using other priors in the case of multiparameter models~\citep{berger1992development}.
For example, for the location-scale family, Jeffreys recommended using alternate priors, which coincide with the exact matching prior~\citep{diciccio2017simple}.

As mentioned in the main text, it is known that the Jeffreys prior leads to inconsistent estimators for the variance in the Neyman-Scott problem~\citep[see][Example 3.]{berger1992development}.
Another example is Stein's example~\citep{stein1959example}, where the model of the Gaussian distribution with a common variance is considered.
In this example, the Jeffreys prior lead to an unsatisfactory posterior distribution since the generalized Bayesian estimator under the Jeffreys prior is dominated by other estimators for the quadratic loss~\citep[see][Example 3.5.9.]{robert2007bayesian}.
Note that \citet{bernardo1979reference} showed that the reference prior does not suffer from such problems, which would explain why the reference prior shows better performance than the Jeffreys prior in the multiparameter bandit problems.

\subsection{Probability matching prior}
The probability matching prior is a type of noninformative prior that is designed to achieve the synthesis between the coverage probability of the Bayesian interval estimates and that of the frequentist interval estimates~\citep{welch1963formulae, tibshirani1989noninformative}. 
Therefore, the posterior probability of certain intervals matches exactly or asymptotically the frequentist's coverage probability under the probability matching prior.
If the posterior probability of certain intervals exactly matches the confidence interval, such a prior is called an exact matching prior.
In cases where the Bayesian interval estimate does not exactly match the frequentist's coverage probability, but the difference is small, it is called a $k$-th order matching prior. The difference between the two probabilities is measured by a remainder term, usually denoted as $\mathcal{O}(n^{-\frac{k}{2}})$, where $n$ is the sample size and $k$ is the order of the matching\footnote{Some papers call a prior $k$-th order matching prior when a remainder is $\mathcal{O}\left(n^{-\frac{k+1}{2}}\right)$~\citep{diciccio2017simple}. In this paper, we follow notations used in \citet{mukerjee1997second} and \citet{datta2005probability}.}.

For example, let $\theta \in \mathbb{R}_{+}$ be a parameter of interest.
For some priors $\pi(\theta)$, let $\psi(\theta|X_n)$ be a posterior distribution after observing $n$ samples $X_n$.
Then, for any $\alpha \in (0,1)$, let us define $\theta_{\alpha} > 0$ such that
\begin{equation*}
    \int_{0}^{\theta_{\alpha}} \psi(\theta|X_n) \dx \theta = \alpha.
\end{equation*}
When $\pi(\theta)$ is the second order probability matching prior, it holds that
\begin{equation*}
    \mathbb{P}[\theta \leq \theta_\alpha | X_n] = \alpha + \mathcal{O}(n^{-1}). 
\end{equation*}
When $\pi(\theta)$ is the exact probability matching prior, we have
\begin{equation*}
    \mathbb{P}[\theta \leq \theta_\alpha | X_n] = \alpha.
\end{equation*}
For more details, we refer readers to \citet{datta2005probability} and \citet{ghosh2011objective} and the references therein.

\subsection{Details on the derivation of posteriors}\label{sec: derivation_posterior}
In this section, we provide the detailed derivation of posteriors.

Let the observation $\br = (r_{1}, \ldots, r_{n})$ of an arm and let $q(n) = \sum_{s=1}^n \log r_s$.
Then, Bayes' theorem gives the posterior probability density as
\begin{align*}
    p(\ka \mid\br) 
    =\frac{p(\br| \ka) p(\ka)}{\int_{0}^\infty \int_{0}^\infty p(\br \mid \ka)p(\ka) \dx \kappa \dx \alpha},
\end{align*}
where 
\begin{align*}
    p(\br\mid\ka) &= \alpha^n \kappa^{n\alpha}\left( \prod_{s=1}^n r_{a,s} \right)^{-\alpha-1} \I[\kappa \leq \hk(n)] \\
    &= \alpha^n \kappa^{n\alpha} \exp(-q(n)(\alpha+1)) \I[\kappa \leq \hk(n)].
\end{align*}
By direct computation with given prior with $k \in \mathbb{Z}$, we have
\begin{align*}
    \int_{0}^\infty \int_{0}^\infty p(\br \mid \ka)p(\ka) \dx \kappa \dx \alpha &= \int_{0}^\infty \int_{0}^\infty p(\br \mid \ka) \frac{\alpha^{-k}}{\kappa}\dx \kappa \dx \alpha \\
    &= \int_{0}^\infty  \alpha^{n-k} \exp(-q(n) (\alpha+1)) \int_{0}^{\hk} \kappa^{n\alpha -1}  \dx \kappa \dx \alpha \\
    &=  \int_{0}^\infty  \frac{\alpha^{n-k-1}}{n} e^{-q(n)} \exp(-\alpha(q(n) - n \log \hk)) \dx \alpha \\
    &= \frac{\Gamma(n-k)}{n}\frac{e^{-q(n)}}{(q(n) - n \log \hk)^{n-k}}.
\end{align*}
Therefore, the joint posterior probability density is given as follows:
\begin{equation*}
    p(\ka \mid \br) = \frac{n[q(n) - n\log \hk(n)]^{n-k}}{\Gamma(n-k)}\alpha^{n-k} \kappa^{n\alpha-1} e^{-q(n)\alpha} \I[0<\kappa \leq \hk(n)],
\end{equation*}
which gives the marginal posterior of $\alpha$ as
\begin{equation} \label{eq: postak1}
    p(\alpha \mid \br) = \frac{\alpha^{n-k-1}[q(n) - n\log \hk(n)]^{n-k}}{\Gamma(n-k)} e^{-\alpha(q(n)-n\log \hk(n))}.
\end{equation}
Thus, sample $\ta$ generated from the marginal posterior actually follows the Gamma distribution with shape $n-k$ and rate $q(n)-n\log \hk(n)= \frac{n}{\ha}$, i.e., $\ta \sim \Er\left(n-k, \frac{n}{\ha} \right)$ as $n\in \mathbb{N}$ and $k \in \mathbb{Z}$ if $n > k$. 
When $\ta$ is given, the conditional posterior of $\kappa$ is given as
\begin{align*}
  p(\kappa \mid \br, \alpha) 
    &= \frac{p(\ka \mid \br)}{p(\alpha \mid \br)}
  \\&= \frac{n\alpha}{\hk^{n\alpha}} \kappa^{n\alpha -1} \I[0<\kappa \leq \hk(n)]. \numberthis{\label{eq: pdfkappak1}}
\end{align*}
Hence, the cumulative distribution function (CDF) of $\kappa$ given $\alpha$ is given as
\begin{equation}\label{eq: cdfkappak1}
  \mathbb{P}(\kappa \leq x) = F(x \mid \br, \alpha=\ta) = \left( \frac{x}{\hk(n)}\right)^{n\ta}, \, 0 < x \leq \hk(n). 
\end{equation}
Note that MLEs of $\ka$ are equivalent to the maximum a posteriori (MAP) estimators when one uses the Jeffreys prior~\citep{Pareto_inference, Pareto_prior}.

In sum, under aforementioned priors, we consider the marginalized posterior distribution on $\alpha$
\begin{equation*}
    p(\alpha \mid \br) = \Er\left(n-k, \frac{n}{\ha}\right)
\end{equation*}
and the cumulative distribution function (CDF) of the conditional posterior of $\kappa$
\begin{equation*}
  F(x \mid \br, \alpha=\ta) = \left( \frac{x}{\ha(n)}\right)^{n\ta}, \, 0 < x \leq \hk(n).
\end{equation*}
Note that we require $\max\left\{2, k+1 \right\}$ initial plays to avoid improper posteriors and improper MLEs.

\section{Technical lemma}
In this section, we present some technical lemmas used in the proof of main lemmas.

\begin{restatable}{lemma}{lemCramer}~\label{lem: chicramer}
Let $Z$ be a random variable following the chi-squared distribution with the degree of freedom $2n$.
Then, for any $ x \in (0,1)$
\begin{equation*}
    \mathbb{P}[Z \leq 2nx] \leq e^{-n h(x)},
\end{equation*}
where $h(x) = (x-1 -\log x) \geq 0$.
\end{restatable}
\begin{proof}
Let $X_i$ be random variables following the standard normal distribution, so that $Z = \sum_{i=1}^{2n} X_i^2$ holds.
From Lemma~\ref{lem: cramer}, one can derive
\begin{equation*}
    \mathbb{P}[Z \leq 2nx] =  \mathbb{P}\left[ \frac{1}{2n}\sum_{i=1}^{2n} X_i^2 \leq x \right] \leq \exp{\left( - 2n\inf_{z \leq x} \Lambda^* (z) \right)}.
\end{equation*}
From the definition of the moment generating function, one can see that
\begin{equation*}
    \Lambda^*(z) = \sup_{\lambda \in \mathbb{R}} \lambda z - \log \mathbb{E}\left[e^{\lambda X_1^2} \right] = \sup_{\lambda \in \mathbb{R}} \lambda z + \frac{1}{2}\log (1-2\lambda) = \frac{1}{2}(z -1 -\log z),
\end{equation*}
which concludes the proof.
\end{proof}

\section{Known results}
In this section, we present some proved lemmas that we use without proofs.
\begin{lemma}[Bernstein's inequality]\label{lem: subexp}
Let $X$ be a $(\sigma^2, b)$-subexponential random variable with $\mathbb{E}[X]=\mu$ and $Var[X]=\sigma^2$, which satisfies
\begin{equation*}
    \mathbb{E}[e^{\lambda X}] \leq \exp{\frac{\lambda^2 \sigma^2}{2}} \quad \text{ for } |\lambda| \leq \frac{1}{b}.
\end{equation*}
Let $X_i$ be independent $(\sigma^2, b)$-subexponential.
Then, it holds that
\begin{equation*}
    \mathbb{P}\left( \abs{\frac{1}{n} \sum_{s=1}^n X_i - \mu} \geq t \right) \leq 2\exp \left( - \frac{n}{2}\min \left\{ \frac{t^2}{\sigma^2} , \frac{t}{b} \right\}\right). 
\end{equation*}
\end{lemma}
For more details, we refer the reader to \citet{vershynin2018high}.

\begin{lemma}[Theorem~4.1. in \citet{wallace1959}]\label{lem: wallace}
Let $F_n$ be the distribution function of the chi-squared distribution on $n$ degrees of freedom.
For all $t > n$, all $n>0$, and with $w(t) = \sqrt{t - n -n \log(t/n)}$,
\begin{equation*}
    1 - F_n(t) < \frac{d_n}{2} \mathrm{erfc}\left(\frac{w(t)}{\sqrt{2}}\right),
\end{equation*}
where $d_n = \frac{\left( \frac{n}{2} \right)^{\frac{n-1}{2}} e^{-\frac{n}{2}} \sqrt{2\pi} }{\Gamma(n/2)}$ and $\mathrm{erfc}(\cdot)$ is the complementary error function.
\end{lemma}

\begin{lemma}[Cram\'{e}r's theorem]\label{lem: cramer}
Let $X_1, \ldots, X_n$ be i.i.d. random variables on $\mathbb{R}$. Then, for any convex set $C \in \mathbb{R}$,
\begin{equation*}
    \mathbb{P}\left[\frac{1}{n}\sum_{i=1}^n X_i \in C\right] \leq \exp{\left( -n \inf_{z \in C} \Lambda^* (z) \right)},
\end{equation*}
where $\Lambda^*(z) = \sup_{\lambda \in \mathbb{R}} \lambda z - \log \mathbb{E}[e^{\lambda X_1}]$.
\end{lemma}

\begin{lemma}[Result of term (A) in \citet{KordaTS}]\label{lem: kordaresult}
When one uses the Jeffreys prior as a prior distribution under the Pareto distribution with known scale parameter, TS satisfies that for sufficiently small $\eps >0$,
\begin{equation*}
    \sum_{t=1}^T \mathbb{E}\left[\I[j(t) \ne 1, \eM_\eps^c(t)] \right] \leq \mathcal{O}\left(\eps^{-1}\right).
\end{equation*}
\end{lemma}

\section{Additional experimental results}\label{app: Exp}
\begin{figure*}[!t]
\centering
    \subfigure[The Jeffreys prior $k=0$]{\label{fig: ka0}\includegraphics[width=0.35\textwidth]{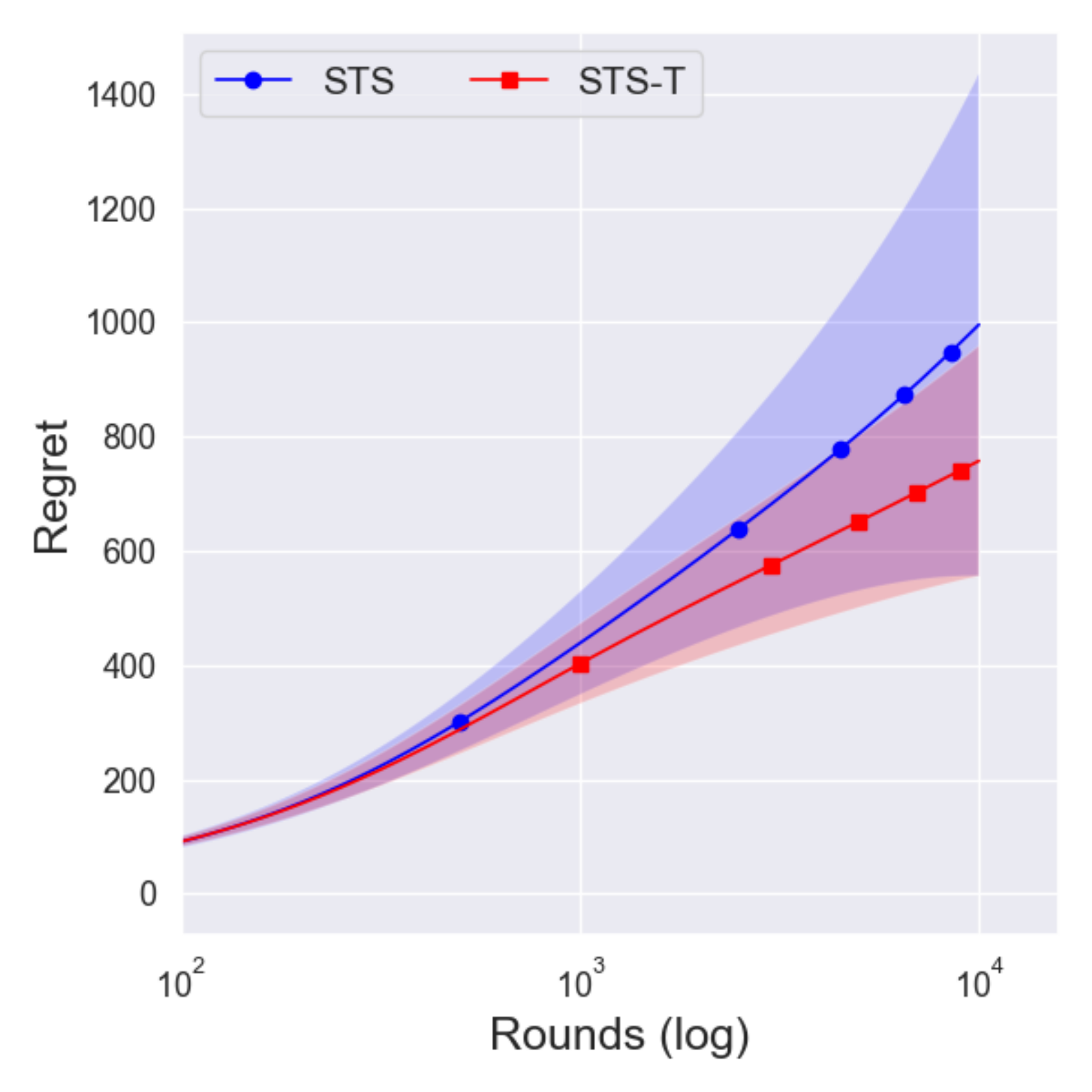}}
    \quad
    \subfigure[The reference prior $k=1$]{\label{fig: ka1}\includegraphics[width=0.35\textwidth]{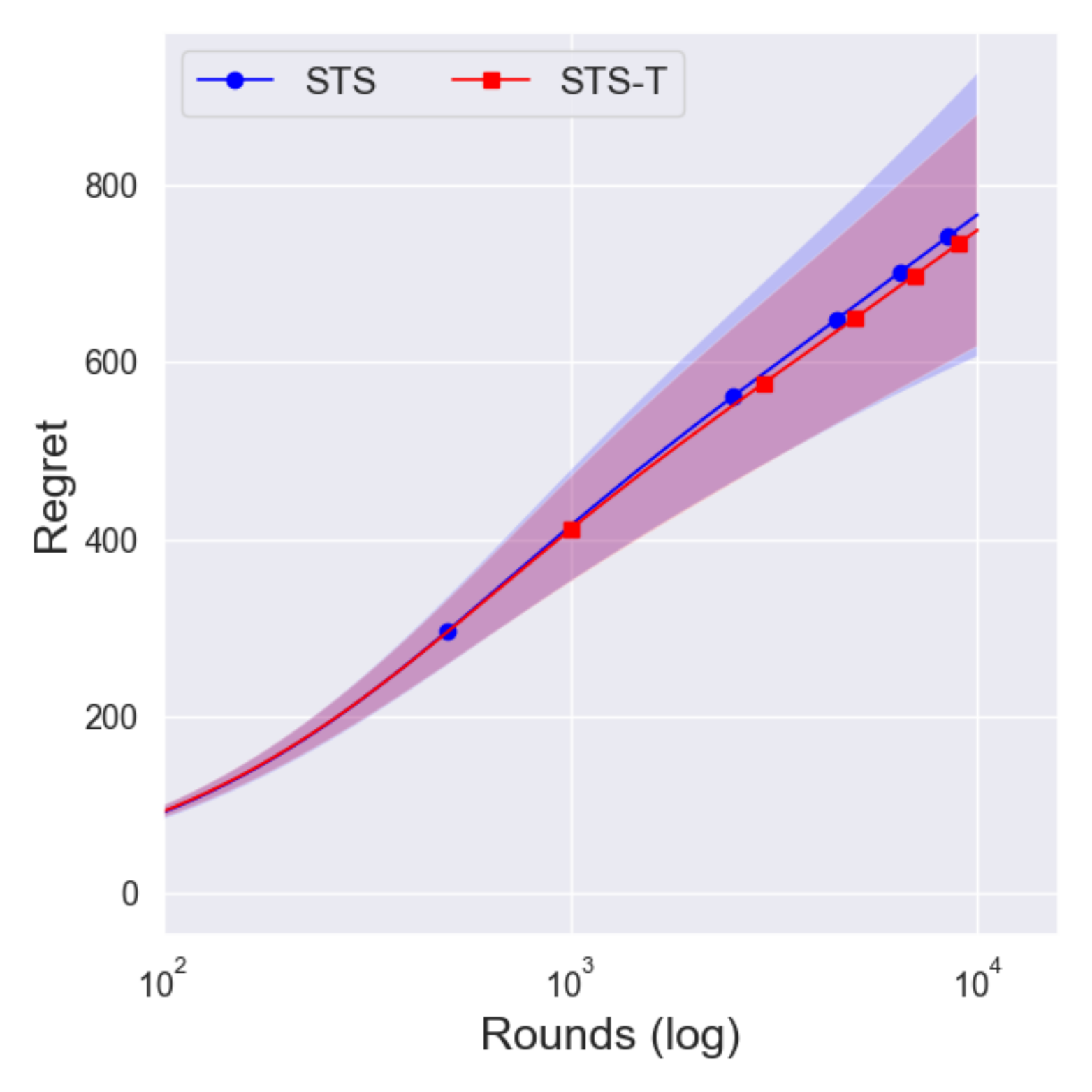}}
    \medskip
    \subfigure[Prior with $k=3$]{\label{fig: ka3}\includegraphics[width=0.3\textwidth]{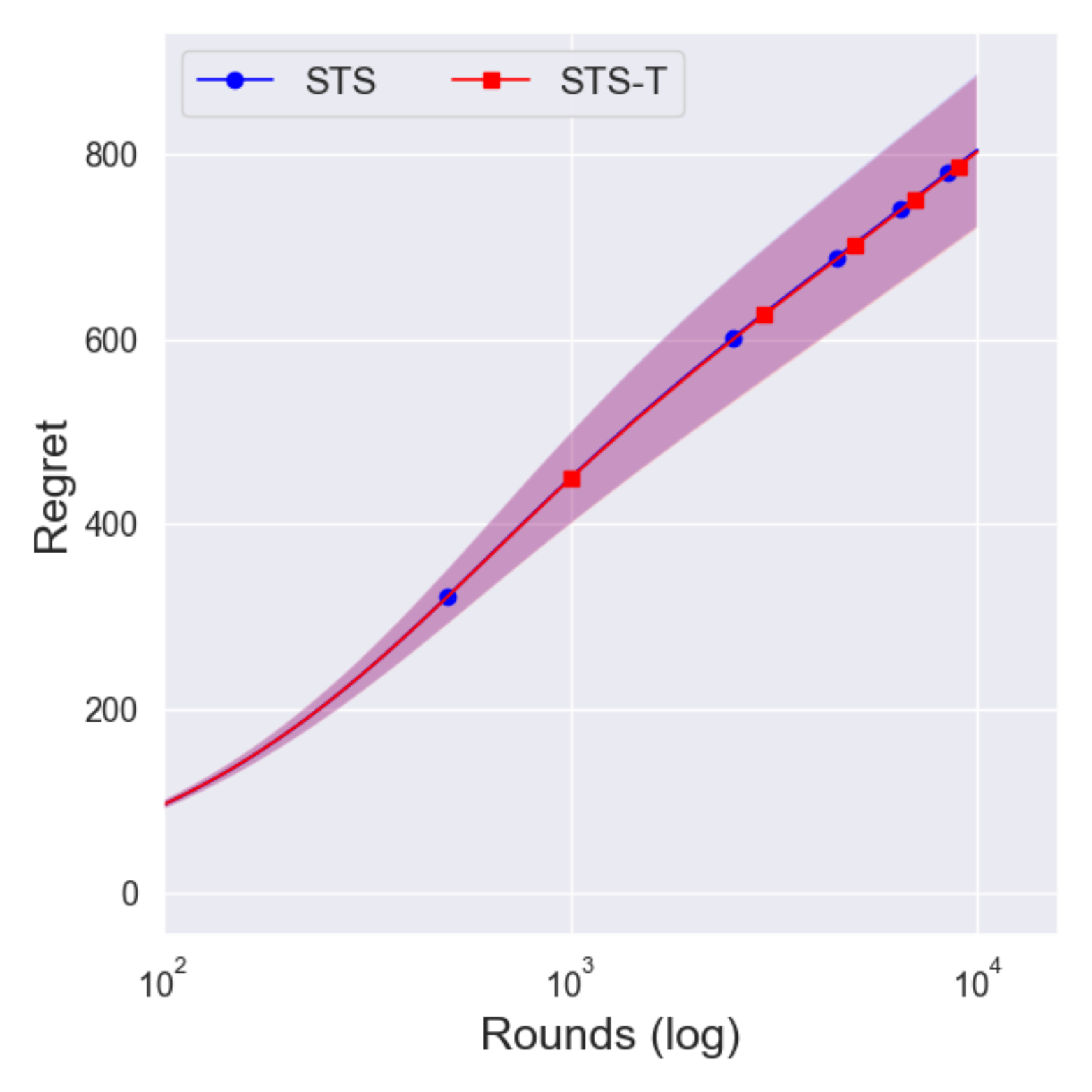}}
    \hfill
    \subfigure[Prior with $k=-1$]{\label{fig: ka-1}\includegraphics[width=0.3\textwidth]{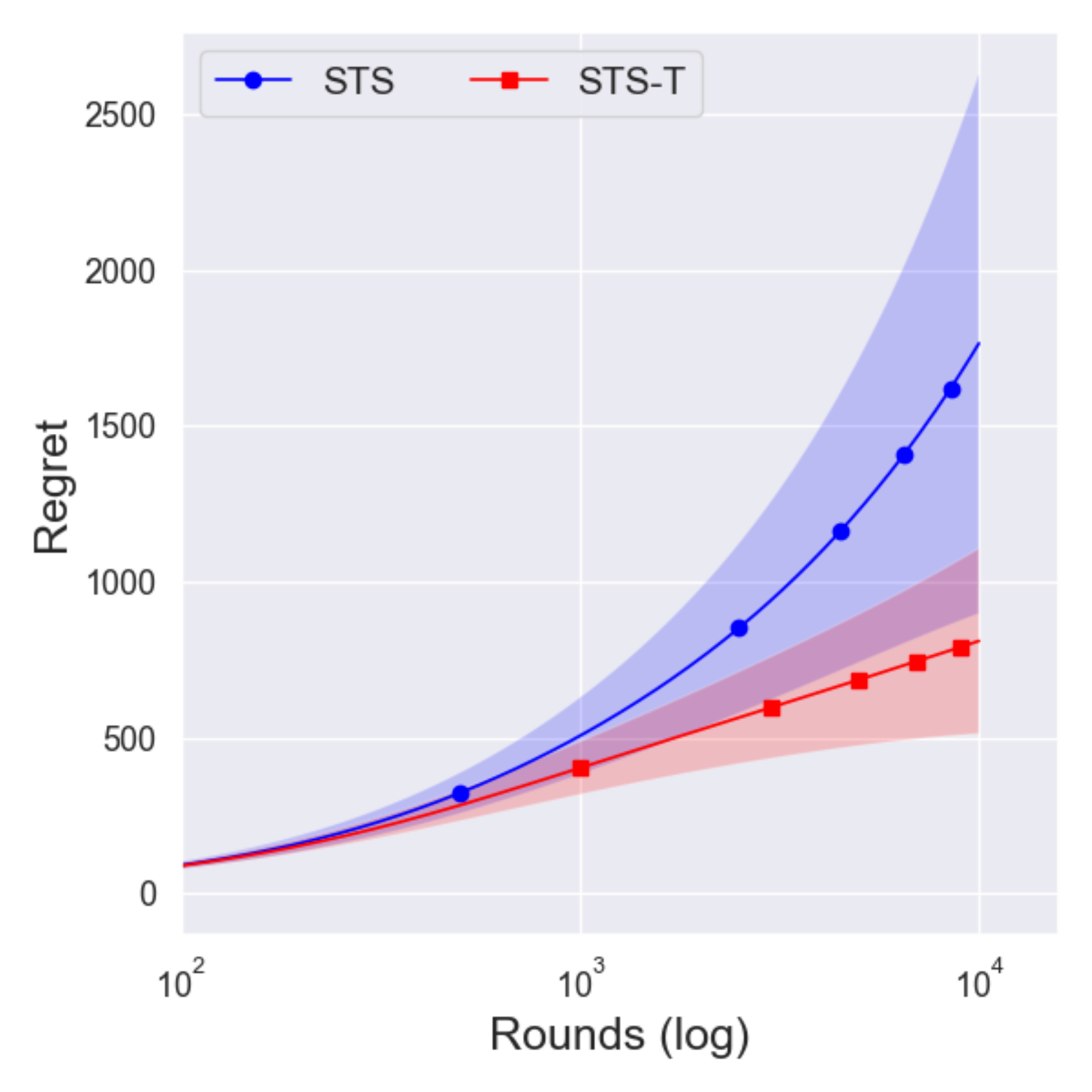}}
    \hfill
    \subfigure[Prior with $k=-3$]{\label{fig: ka-3}\includegraphics[width=0.3\textwidth]{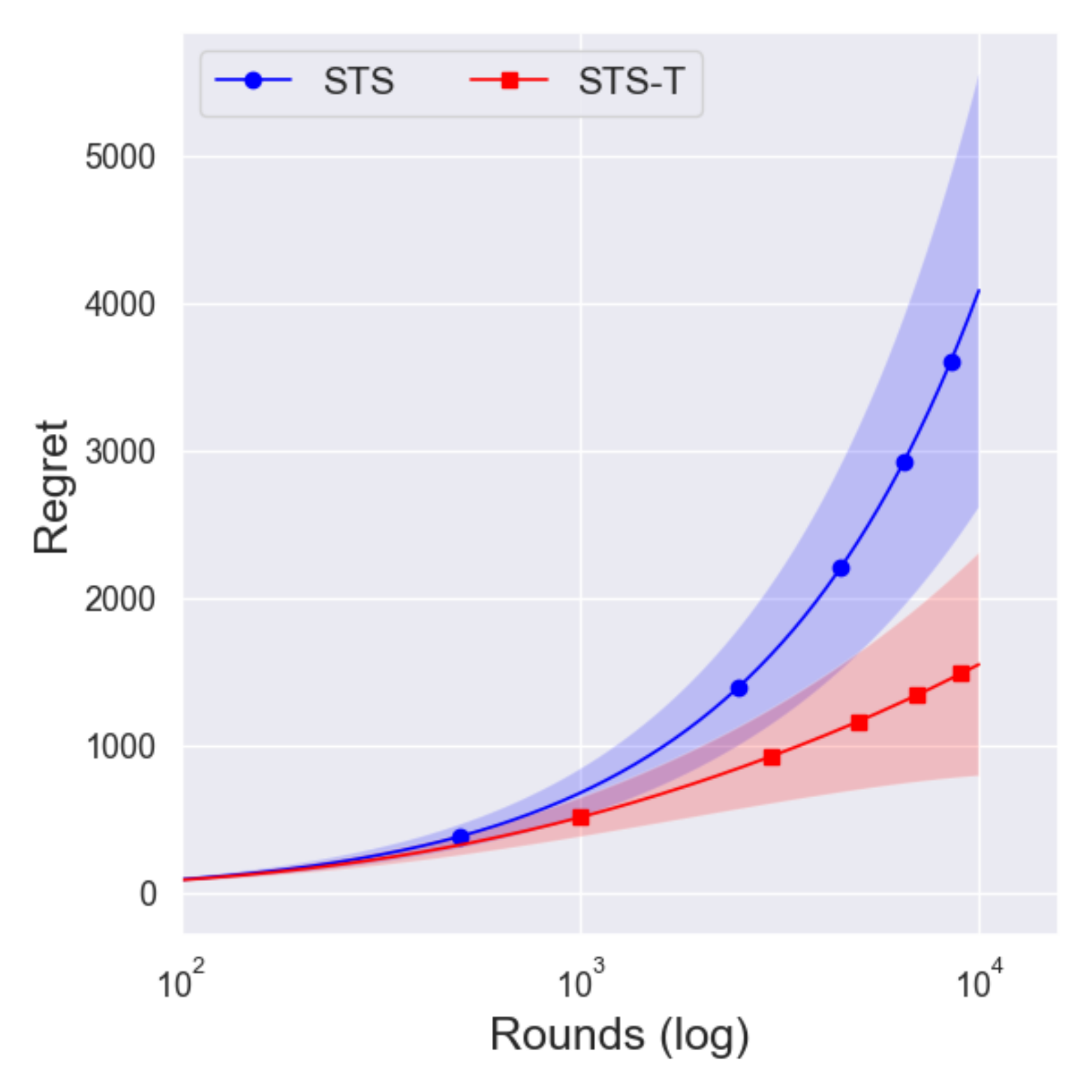}}
\vspace{-1em}
\caption{The solid lines denotes an averaged regret over independent 100,000 runs. The shaded regions show a quarter standard deviation.}
\label{fig: all_k_sqrt}
\end{figure*}
From Figure~\ref{fig: all_k_sqrt}, one can observe that the performance difference between $\TS$ and $\TST$ is large as $k$ decreases.
Since a truncation procedure aims to prevent an extreme case that can occur under $\TS$ with priors $k \in \mathbb{Z}_{\leq 1}$, it is quite natural to see that there is no difference between $\TS$ and $\TST$ with prior $k=3$.
We can further see the improvement of $\TST$ is dramatic as $k$ decreases, where an extreme case can easily occur.

\begin{figure*}[!t]
\centering
    \subfigure[Cumulative regret of $\TS$ with various $k$ under $\bth_4'$]{\label{fig: STS2}\includegraphics[width=0.42\textwidth]{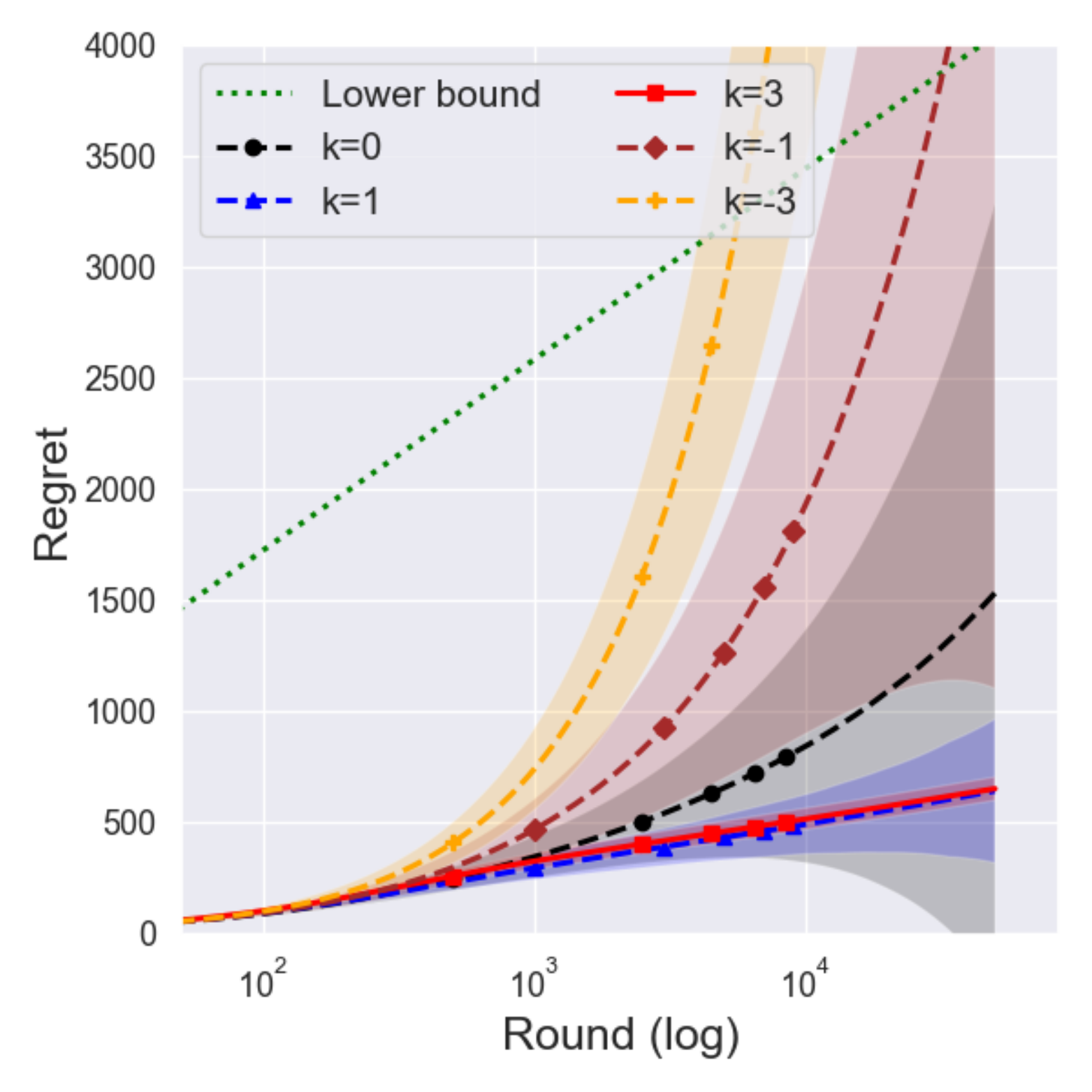}}
    \quad
    \subfigure[Cumulative regret of $\TST$ with various $k$ under $\bth_4'$]{\label{fig: STST2}\includegraphics[width=0.42\textwidth]{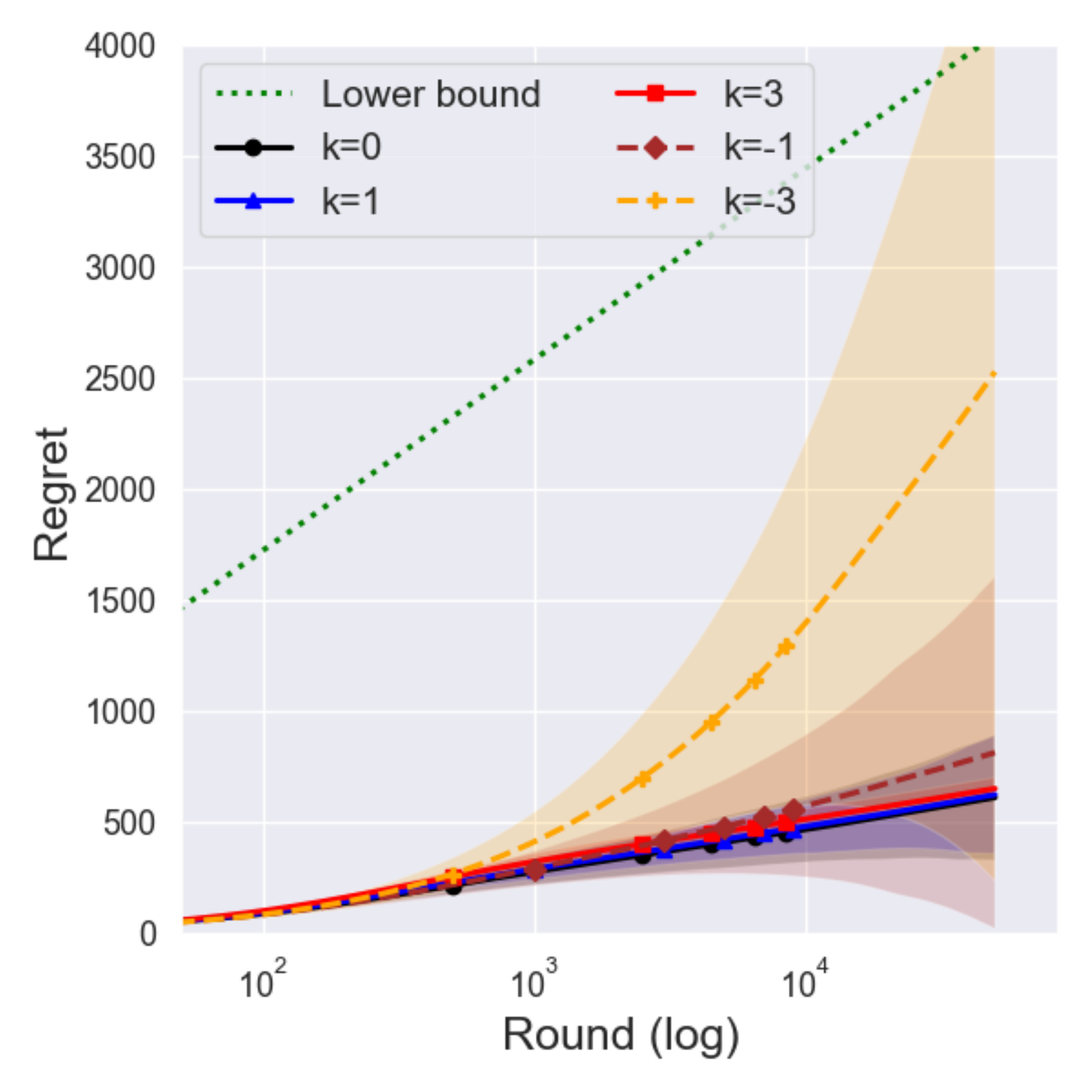}}
\caption{The solid lines denote the averaged cumulative regret over 100,000 independent runs of priors that can achieve the optimal lower bound in (\ref{eq: lowerbound}).
The dashed lines denote that of priors that cannot achieve the optimal lower bound in (\ref{eq: lowerbound}).
The green dotted line denotes the problem-dependent lower bound based on Lemma~\ref{lem: KLinf}.}
\label{fig: overall2}
\end{figure*}

\begin{figure*}[!ht]
\centering
    \subfigure[The Jeffreys prior $k=0$]{\label{fig: 2ka0}\includegraphics[width=0.35\textwidth]{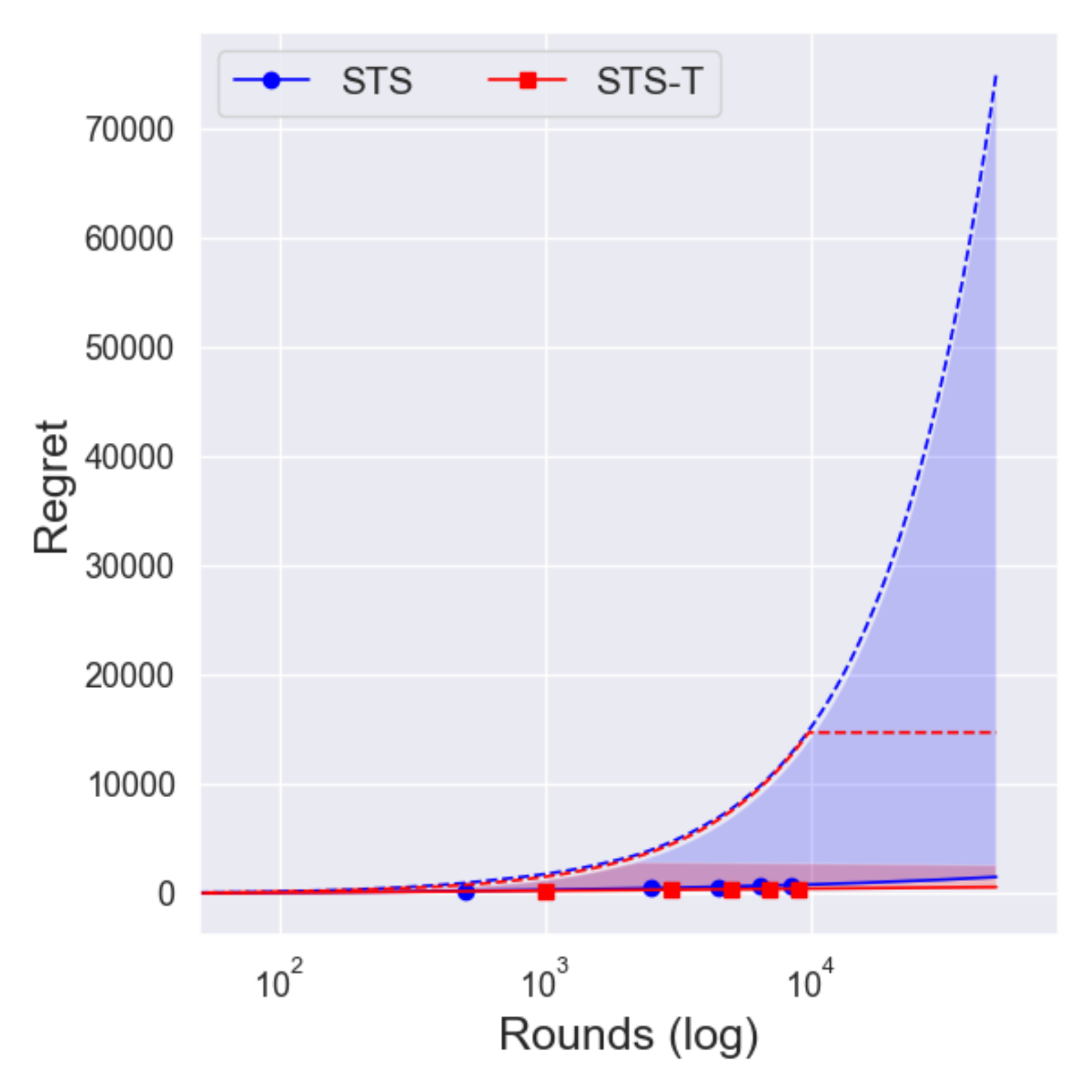}}
    \quad
    \subfigure[The reference prior $k=1$]{\label{fig: 2ka1}\includegraphics[width=0.35\textwidth]{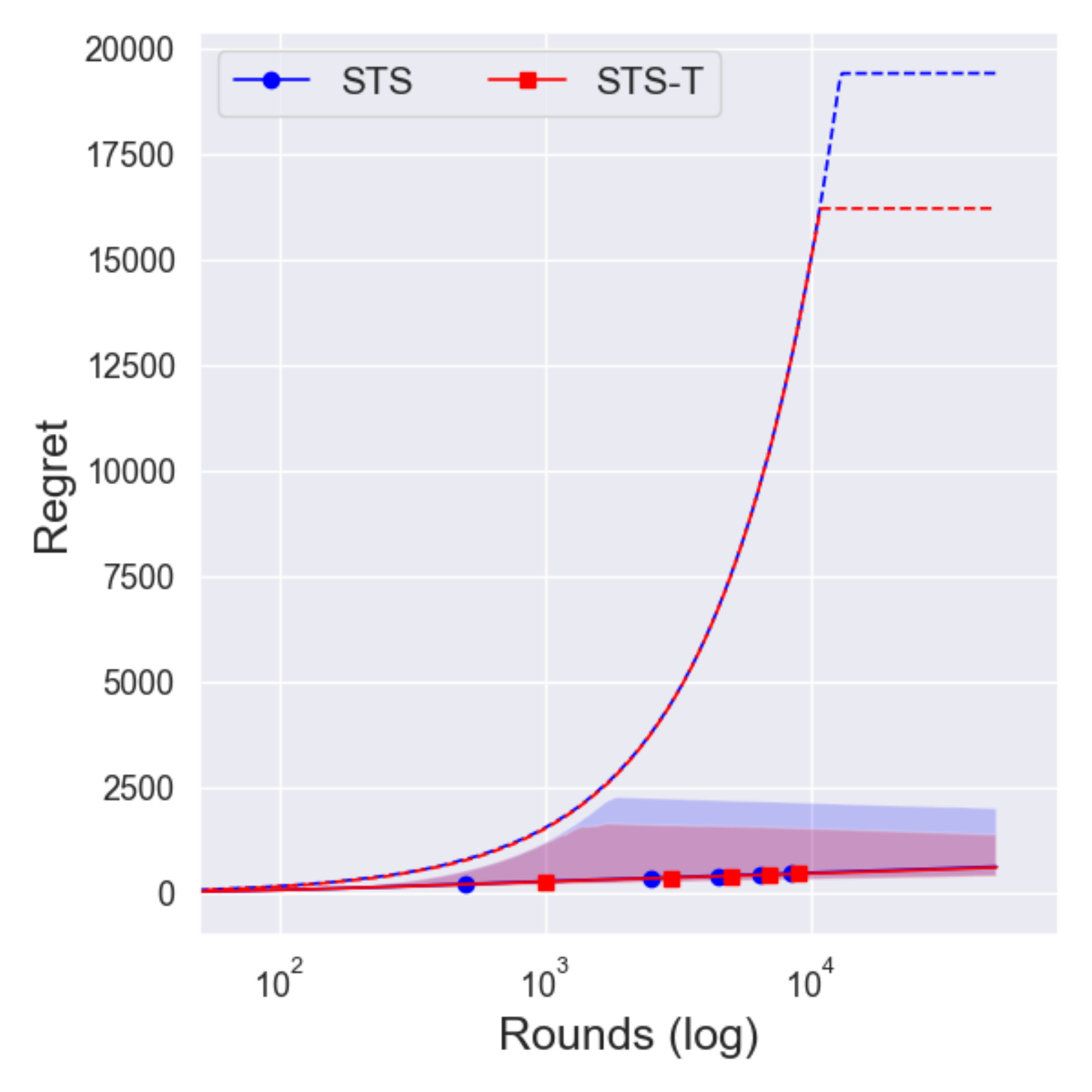}}
    \hfill
    \medskip
    \subfigure[Prior with $k=3$]{\label{fig: 2ka3}\includegraphics[width=0.32\textwidth]{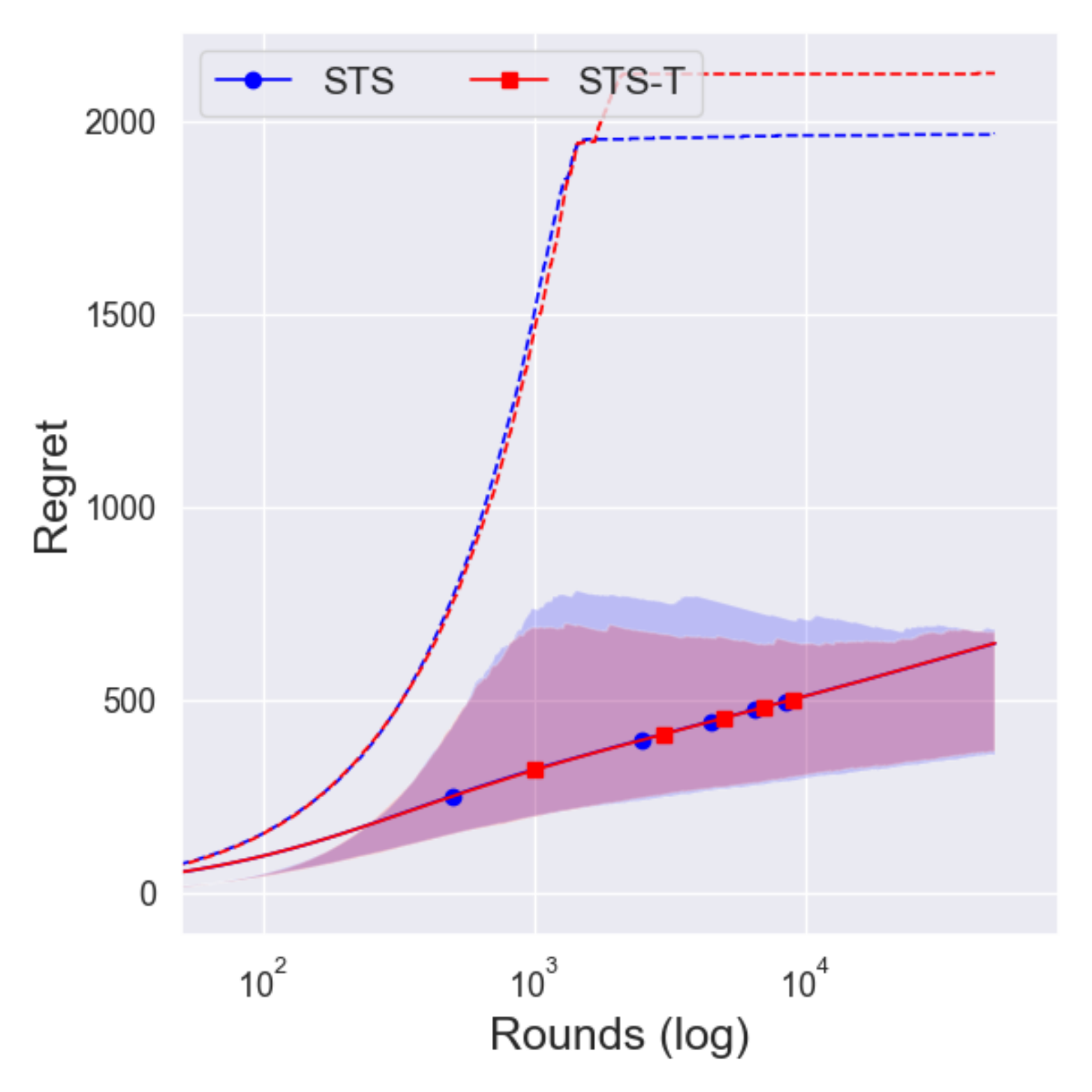}}
    \hfill
    \subfigure[Prior with $k=-1$]{\label{fig: 2ka-1}\includegraphics[width=0.32\textwidth]{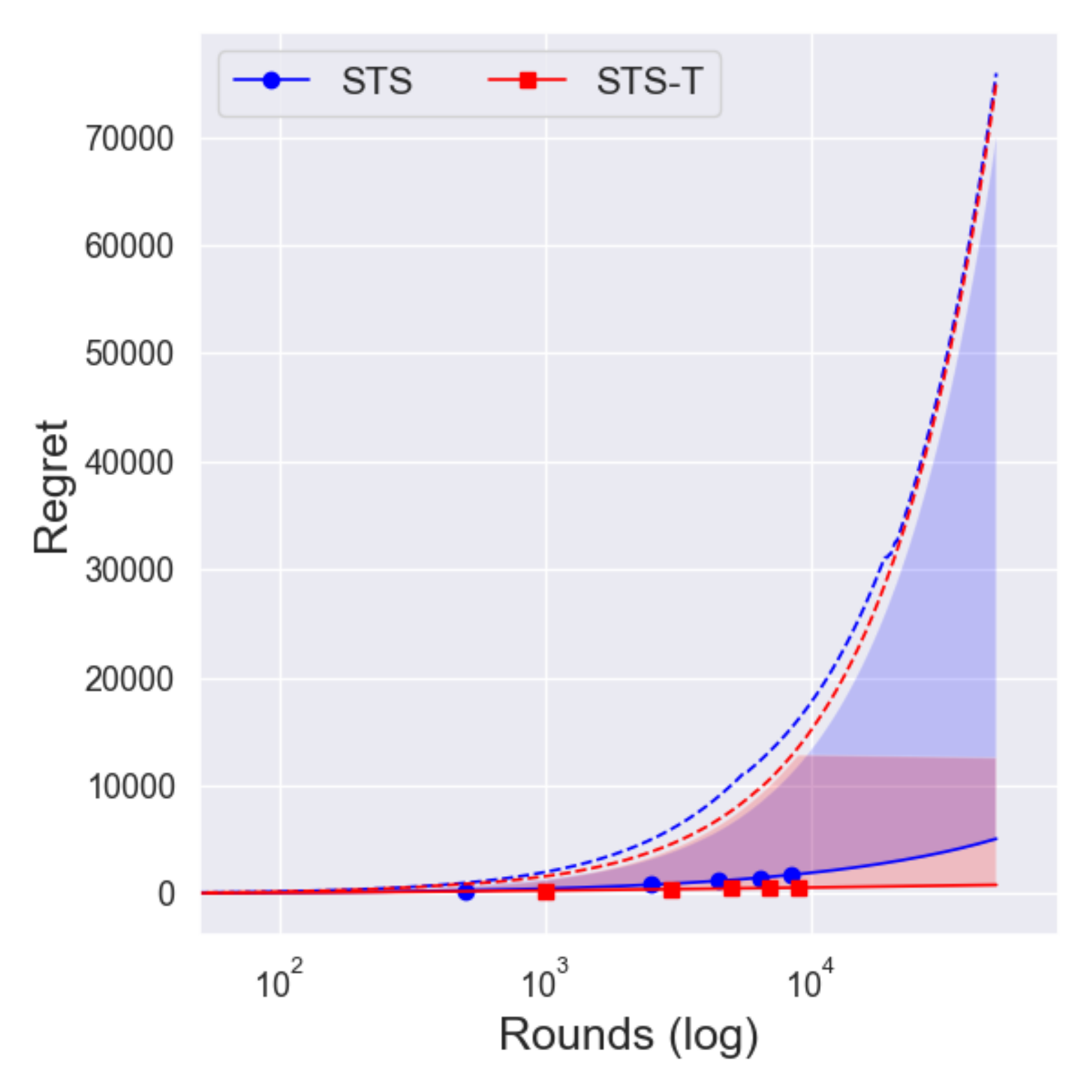}}
    \hfill
    \subfigure[Prior with $k=-3$]{\label{fig: 2ka-3}\includegraphics[width=0.32\textwidth]{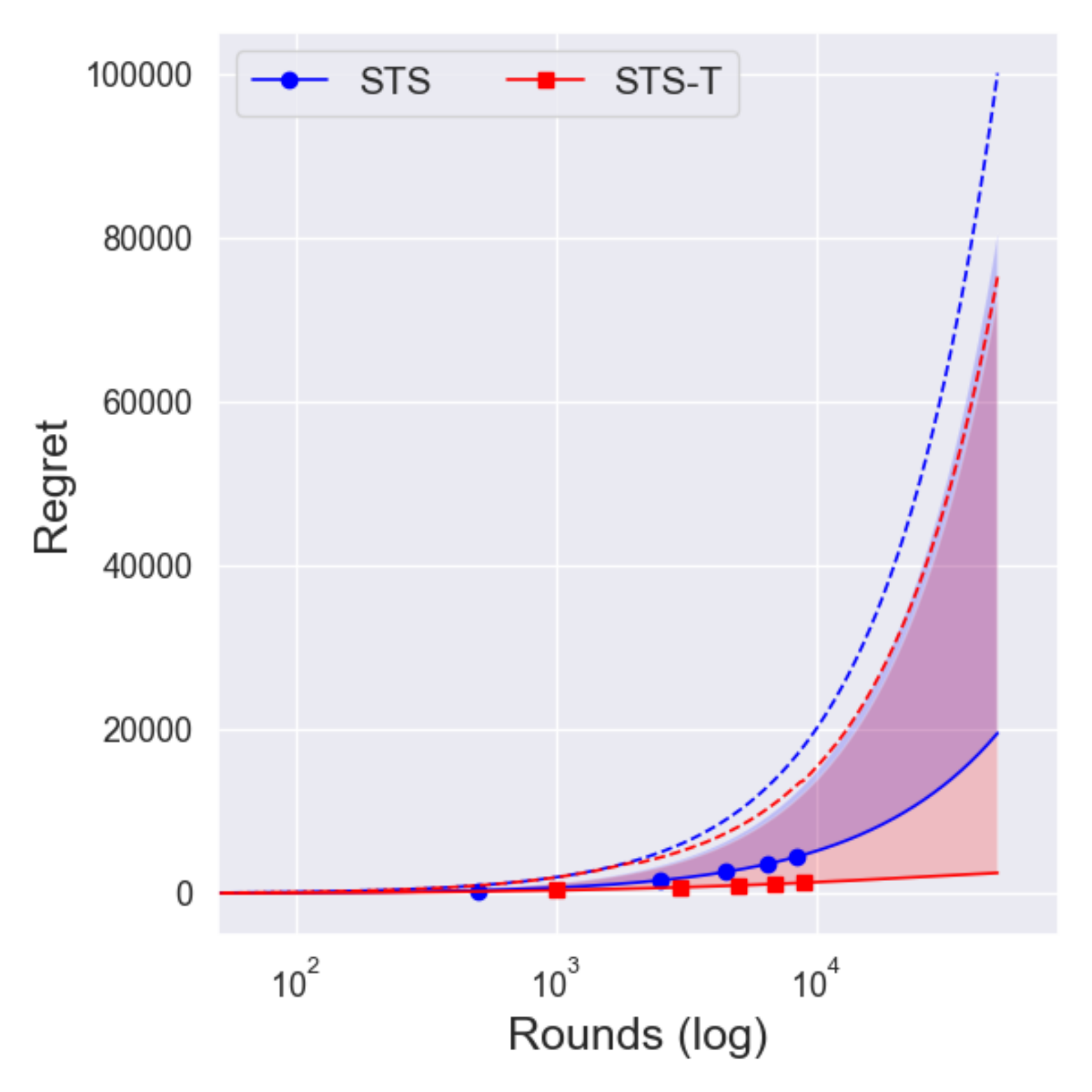}}
\caption{The solid lines denotes an averaged regret over independent 10,000 runs. The shaded regions and dashed lines show the central $99\%$ interval and the upper $0.05\%$ of regret, respectively.}
\label{fig: all_k_995}
\end{figure*}

We further consider another $4$-armed bandit problem $\bth_4'$ where $\bm{\kappa}=(1.0, 1.5, 2.0, 2.0)$ and $\bm{\alpha} = (1.2, 1.5, 1.8, 2.0)$ where $\bmu = (5.0, 4.5, 4.5, 4.0)$.
$\bth_4'$ would be a more challenging problem than $\bth_4$ in the sense that the $\kappa$ determines the left boundary of the support, where larger $\kappa$ implies larger minimum value of the arm.
Therefore, if $\kappa$ of the suboptimal arm is larger than that of the optimal arm, it would make a problem difficult in the first few trials.
Figures~\ref{fig: overall2} and~\ref{fig: all_k_995} show the numerical results with time horizon $T=$ 50,000 and independent 10,000 runs.
Although $\TS$ with the reference prior shows similar performance to the conservative prior $k=3$, its performance varies a lot.

Figures~\ref{fig: 2ka0} and~\ref{fig: 2ka1} show the effectiveness of the truncation procedure where $\TST$ has a much smaller upper $0.05\%$ regret than that of $\TS$. 
Although $k=-1$ also shows huge improvements in the central $99\%$ interval of regret as shown in Figure~\ref{fig: 2ka-1}, $\TST$ with $k=-1$ shows worse performance compared with priors with $k\in \mathbb{Z}_{\geq 0}$ in Figure~\ref{fig: STST2}.

\end{document}